%% file: main.tex
\documentclass{article}

\usepackage[preprint]{neurips_2026}
\usepackage{microtype}
\usepackage{graphicx}
\usepackage{booktabs}
\usepackage{amsmath,amssymb,amsfonts}
\usepackage{amsthm}
\usepackage{mathtools}
\usepackage{algorithm}
\usepackage{algorithmic}
\usepackage{hyperref}
\usepackage{url}
\usepackage{float}
\usepackage{enumitem}
\usepackage{tikz}
\usetikzlibrary{positioning,fit,arrows.meta,calc}
\usepackage{xcolor}

\newcommand{\E}{\mathbb{E}}
\newcommand{\Prob}{\mathbb{P}}
\newcommand{\Var}{\mathrm{Var}}
\newcommand{\1}{\mathbf{1}}
\newcommand{\norm}[1]{\left\lVert #1\right\rVert}

\usepackage{titlesec}

\newenvironment{localparagraphgap}[1]{%
  \par
  \begingroup
  \titlespacing*{\paragraph}{0pt}{#1}{1em}%
}{%
  \par
  \endgroup
}

\theoremstyle{plain}
\newtheorem{theorem}{Theorem}
\newtheorem{corollary}{Corollary}
\newtheorem{proposition}{Proposition}
\newtheorem{lemma}{Lemma}
\theoremstyle{definition}

\theoremstyle{remark}
\newtheorem{remark}{Remark}

\title{Marginal-Contribution Policy Gradients under Filtered Feedback for Multi-Agent LLMs}

\author{
Elai Ben-Gal\\
Department of Mathematics \\
Stanford University \\
\texttt{bengal@stanford.edu}
\And
Yunjin Tong\\
Graduate School of Business \\
Stanford University \\
\texttt{yjtong@stanford.edu}
}

\input{tables/numbers2}
\input{tables/numbers_e4}

\input{tables/numbers_compute}
\input{tables/numbers_e5}
\input{tables/numbers_e8}
\begin{document}

\maketitle

\begin{abstract}
We develop a unified treatment of credit assignment for RL training in multi-agent LLM systems. In these systems, several agents produce outputs, a mechanism such as a router deploys one, and only that output is rewarded. How should each agent learn from this one reward? Giving every agent the shared reward follows the right gradient but is noisy, because each agent's update carries every other agent's luck. We prove that rewarding only the selected agent (winner-take-all) is biased: it maximizes the agent's own expected reward from being selected, and the bias equals the reward it takes from its competitors. Between the two lie contribution signals, which subtract from the reward a baseline that ignores the agent's own output. We prove that once agents interact, an agent's contribution must depend on which other agents are present; that contribution signals keep the system's gradient, so they can help only by lowering noise; and that their noise splits into three terms, which determine when one beats shared reward. We show how to learn contributions from ordinary training data and, under routing, how to estimate them without extra generation. On GSM8K with Llama-3.2-3B agents, measured gradient noise matches these predictions and shows how to fix a noisy estimator. In pre-registered training with nine agents, a contribution signal learns faster than shared reward ($p=\eEightHOneP$) and tolerates twice the learning rate.
\end{abstract}

\section{Introduction}

Many LLM systems are multi-agent: several agents generate outputs, and a mechanism selects, combines, or filters them into one deployed answer. In \emph{routed} systems, agents compete: each proposes a candidate and a router deploys one. In \emph{collaborative} systems, agents hold roles, such as planner, solver, and critic, and build the output together. Section~\ref{sec:definitions} models both with one object: a planner--solver--critic pipeline is a chain of layers with one agent each, and a router pool is one layer of parallel agents. Such systems are used because diverse or specialized models together can answer better or more cheaply than one model \citep{du2024improving,wang2025mixture,ong2025routellm}; once a system is deployed, each agent should be trained for its part in it, from the reward the system receives. Only the deployed output is rewarded, so under reinforcement learning each agent's policy gradient sees \emph{filtered feedback}: a reward on the system's output, not on the agent's own action.

Two simple training rules are common. \emph{Shared reward} gives every agent the deployed reward $G$. It follows the system gradient, but it is noisy: agent $i$'s update also carries the variation caused by every other agent and by the mechanism. \emph{Winner-take-all} gives $G$ only to the agent whose candidate was deployed. It is biased: it performs exact gradient ascent on the agent's own expected reward from being selected, not on system performance, and the bias equals the return the agent displaces from its competitors (Theorem~\ref{thm:wta}).

Between these two rules lies a family of \emph{contribution signals}: subtract from $G$ a baseline that does not depend on the agent's own output. Every such signal has the same expected gradient as shared reward (Lemma~\ref{lem:baseline}), so its only purpose is lower variance. Using one well requires three answers: what it means for an agent to contribute (\emph{attribution}), which contributions filtered feedback can identify (\emph{identification}), and when a contribution signal beats shared reward (\emph{variance}). Our contributions are:
\begin{enumerate}[leftmargin=*,itemsep=0pt,topsep=1pt]
\item \textbf{Two failures.} The reward's distribution alone cannot identify contribution (Proposition~\ref{prop:shared-credit}). Winner-take-all ascends private utility, with bias equal to the displaced competitor return (Theorem~\ref{thm:wta}).
\item \textbf{Attribution.} Once agents interact, no single number per agent can split every group's value (Proposition~\ref{prop:modularity}). Requiring contributions to add up along every order of adding agents makes them unique for a given valuation (Proposition~\ref{prop:field-unique}); the valuation, retention or removal, is a modeling choice.
\item \textbf{Variance.} Every retention contribution, and the leave-one-out removal contribution, is a valid training signal (Corollaries~\ref{cor:retention-pg} and~\ref{cor:routing-removal-gradient}). For every signal of the form reward minus a baseline that does not read the agent's own output, the gradient variance is an irreducible term plus a baseline-mismatch term plus a noise term (Proposition~\ref{prop:variance-monotone}).
\item \textbf{Identification.} At the population level, masked regression identifies retention contributions from training data (Lemma~\ref{lem:mask-target}). Under routing, propensity weighting identifies removal contributions without extra generation, and the estimate stays gradient-correct (Corollaries~\ref{cor:routing-dr} and~\ref{cor:dr-gradient}). Theorem~\ref{thm:selective} (Appendix~\ref{sec:selective}) allocates exact counterfactual evaluations when error levels are known.
\item \textbf{Experiments.} On routed GSM8K, measured gradient noise follows the decomposition, which locates why the removal estimate is noisy and gives a quieter signal (Corollary~\ref{cor:hybrid}). With nine agents, a removal signal learns faster than shared reward and tolerates larger learning rates; with three, training separates no signal (Section~\ref{sec:experiment}).
\end{enumerate}
\leavevmode\textbf{Roadmap.} Section~\ref{sec:definitions} defines the system and its two failures; Section~\ref{sec:axioms} takes the value of each group of agents and returns each agent's contribution; Section~\ref{sec:signals} says which contributions are valid training signals and how noisy each is; Section~\ref{sec:critic} learns contributions from logged training data; Section~\ref{sec:experiment} tests the predictions.

\section{Related Work}
\label{sec:related}

\paragraph{Credit under shared rewards.}
That a system-level reward need not reveal individual contribution, and that allocating it to agents can misalign private incentives, are classical observations \citep{alchian1972production,holmstrom1982moral,wolpert2001optimal}. Multi-agent RL uses counterfactual signals, such as COMA, difference rewards, and action-dependent baselines, that remove teammate variation while preserving the team policy gradient \citep{foerster2018counterfactual,wu2018variance,castellini2021difference}; Lemma~\ref{lem:baseline} is the corresponding condition here. Cooperative game theory attributes value through marginal contributions and scalar summaries such as the Shapley value \citep{shapley1953value,wang2020shapley,sharp_2026}. We show that once agents interact, contribution must depend on which other agents are present; the Shapley value is one average of these contributions.

\paragraph{Baselines and variance.}
For a single policy, the variance-optimal baseline weights the reward by the squared score norm \citep{weaver2001optimal,greensmith2004variance}; Proposition~\ref{prop:variance-monotone}(ii) states it for baselines that may read other agents' outputs, as do factorized and optimal multi-agent baselines \citep{wu2018variance,kuba2021settling}; whether such a baseline lowers variance in practice is empirical \citep{tucker2018mirage}. The hybrid signal \eqref{eq:hybrid} adds a Q-Prop-type control variate \citep{gu2017qprop,ciosek2018expected} to a model-based difference reward \citep{wang2017optimal,su2020doubly}. RLOO and GRPO use the other samples of the same prompt as the baseline \citep{kool2019buy,ahmadian2024back,shao2024deepseekmath}; without GRPO's standardization \citep{liu2025understanding} this is our shared-reward signal.

\paragraph{Multi-agent LLM systems.}
Debate, mixture-of-agents, and learned routers combine LLMs at inference time \citep{du2024improving,wang2025mixture,ong2025routellm}. Recent training methods obtain counterfactual credit through fixed-context replay, agent-removal rollouts, shared-prefix comparisons, or router--expert signals \citep{c3_2026,ccpo_2026,treecredit_2026,morse_2026}. We address what precedes any estimator: which counterfactual is valid for training, which one the feedback identifies \citep{dudik2011doubly}, and when a valid signal lowers variance.

\section{Mechanisms, Filtered Feedback, and Classical Failures} \label{sec:definitions}
The theory applies to any system of differentiable stochastic policies composed by a mechanism and trained by policy gradients. LLMs make it necessary: their action spaces are too large to marginalize, and their counterfactual rollouts are expensive.

\paragraph{Notation.}
Random variables and population-level scalars are uppercase; realizations and
deterministic functions are lowercase; estimators carry a hat; mechanism randomness $\xi$ is the exception.

A \emph{system} consists of agents arranged in \textit{layers} $\ell \in\{1,\ldots,L\}$, a \textit{mechanism} $M$ that maps their actions to one deployed output, and a scalar reward on that output. The mechanism thus produces \emph{filtered feedback}: the reward attaches to the deployed output, and what it reveals about any agent's action is determined by $M$. At layer $\ell$, the agents $N_\ell=\{1,\dots,K_\ell\}$
observe the history $H_\ell$, which contains the prompt and all outputs of
earlier layers, and act conditionally independently: agent $i$ samples
$A_\ell^{(i)}\sim\pi_{\theta_i}(\cdot\mid H_\ell)$, so that
\begin{equation}
\label{eq:layer-factorization}
\Prob_\theta\left(A_\ell^{(N_\ell)}=a^{(N_\ell)}\mid H_\ell=h\right)
=
\prod_{i\in N_\ell}\pi_{\theta_i}\left(a^{(i)}\mid h\right).
\end{equation}
The factorization holds within a layer only. Across layers, agents interact through the history: a planner--solver--critic pipeline is a chain of singleton layers, and a routing or voting pool is one layer of parallel agents that do not see one another. Parallel solvers followed by an editor are two layers.

\begin{figure}[H]
\centering
\begin{tikzpicture}[
  font=\footnotesize,
  agent/.style={draw, rounded corners=2pt, minimum width=1.0cm,
                minimum height=0.5cm, fill=gray!8},
  hist/.style={draw, rounded corners=2pt, minimum width=0.85cm,
               minimum height=0.5cm, fill=gray!25},
  mech/.style={draw, thick, rounded corners=2pt, minimum width=0.95cm,
               minimum height=1.7cm, fill=gray!40},
  lab/.style={font=\scriptsize, gray!60!black},
  arr/.style={->, >=stealth},
  alloc/.style={->, >=stealth, dashed, gray!60!black}
]

\node[hist] (H1) {$H_1$};

\node[agent, right=0.8cm of H1]   (A1v) {$\pi_{\theta_2}$};
\node[agent, above=0.22cm of A1v] (A1a) {$\pi_{\theta_1}$};
\node[below=0.02cm of A1v]        (d1)  {$\vdots$};
\node[agent, below=0.32cm of d1]  (A1K) {$\pi_{\theta_{K_1}}$};

\draw[arr] (H1) -- (A1a);
\draw[arr] (H1) -- (A1v);
\draw[arr] (H1) -- (A1K);

\node[hist, right=0.9cm of A1v] (H2) {$H_2$};
\draw[arr] (A1a) -- node[above, lab, pos=0.6] {$A_1^{(1)}$} (H2);
\draw[arr] (A1v) -- (H2);
\draw[arr] (A1K) -- node[below, lab, pos=0.7, inner sep=4pt] {$A_1^{(K_1)}$} (H2);

\node[agent] (B1) at ($(A1a -| H2.east)+(1.15,0)$) {$\pi_{\theta_1}$};
\node[agent, below=0.22cm of B1] (B2) {$\pi_{\theta_2}$};
\node[below=0.02cm of B2]        (d2) {$\vdots$};
\node[agent, below=0.32cm of d2] (BK) {$\pi_{\theta_{K_2}}$};
\draw[arr] (H2) -- (B1);
\draw[arr] (H2) -- (B2);
\draw[arr] (H2) -- (BK);

\node[hist, right=0.9cm of B2] (H3) {$H_3$};
\draw[arr] (B1) -- (H3);
\draw[arr] (B2) -- (H3);
\draw[arr] (BK) -- (H3);
\node[right=0.45cm of H3] (dots) {$\cdots$};
\draw[arr] (H3) -- (dots);

\node[mech, right=0.55cm of dots] (M) {$M$};
\draw[arr] (dots) -- (M.west);
\draw[arr] ($(M.north)+(0,0.3)$) -- node[right, lab, pos=0.0] {$\xi$} (M.north);
\node[right=0.55cm of M] (Y) {$Y$};
\draw[arr] (M) -- (Y);
\node[right=0.45cm of Y] (G) {$G$};
\draw[arr] (Y) -- (G);

\coordinate (fork) at ($(B1.north east)+(0.25,0.65)$);
\draw[alloc, -] (G.north) to[out=120, in=5, looseness=0.6]
  node[above, lab, pos=0.5] {$X_i$} (fork);
\draw[alloc] (fork) to[out=185, in=25] (A1a.north east);
\draw[alloc] (fork) to[out=215, in=45] (B1.north east);

\node[draw=gray!50, dashed, rounded corners=3pt, fit=(A1a)(A1K),
      inner sep=3pt, label={[lab]above:layer $1$}] {};
\node[draw=gray!50, dashed, rounded corners=3pt, fit=(B1)(BK),
      inner sep=3pt, label={[lab]above:layer $2$}] {};

\end{tikzpicture}
\caption{Agents act in layers, conditionally independently given the
history $H_\ell$; the mechanism $M$ deploys one output $Y$, which alone
is rewarded. Each agent receives a training signal $X_i$ (dashed): $X_i=G$ under shared reward, $X_i=\alpha_iG$ under an allocation rule.}
\label{fig:system}
\end{figure}

Let $A=(A_\ell^{(N_\ell)})_{\ell=1}^{L}$ collect all actions. The mechanism deploys
$Y=M(H_1,A,\xi)$,
where $H_1$ is the prompt and $\xi$ is mechanism randomness, and the system receives the bounded reward $G=r(H_1,Y)$, with objective
$J_{\mathrm{sys}}(\theta)=\E_\theta[G]$.
For $i\in N_\ell$, the score is
$\Psi_i
=
\nabla_{\theta_i}
\log\pi_{\theta_i}(A_\ell^{(i)}\mid H_\ell)$. Throughout, each policy is differentiable, $\E\norm{\Psi_i}^2<\infty$, and gradients may be interchanged with expectations. All gradient identities are on-policy, and all proofs are in the appendix.

\paragraph{Training signals and allocation rules.}
A training rule returns to each agent $i$ a scalar \emph{training signal} $X_i$, and agent $i$ updates along $\Psi_iX_i$. We call $X_i$ \emph{gradient-correct} if $\E[\Psi_iX_i]=\nabla_{\theta_i}J_{\mathrm{sys}}$. \emph{Shared reward} sets $X_i=G$ for every agent; it is gradient-correct by the score-function identity. An \emph{allocation rule} instead splits the reward into shares $X_i=\alpha_iG$, with $\alpha_i(H_1,A,\xi)\in[0,1]$ and $\sum_i\alpha_i=1$, and defines agent $i$'s \emph{private utility} $U_i(\theta)=\E_\theta[\alpha_iG]$, so that $J_{\mathrm{sys}}=\sum_iU_i$. We use allocation rules only to analyze winner-take-all (Section~\ref{sec:routed}); training signals need not sum to $G$ across agents.

We first ask what the reward alone reveals. The quantity \eqref{eq:action-conditioned-relevance} below measures how much agent~1's realized action moves the expected reward; any notion of contribution must separate an agent that determines the reward from one that never affects it.

\begin{proposition}[The reward's law does not identify contribution]
\label{prop:shared-credit}
There exist two systems with identical agent policies and identical deployed-reward distributions such that
\begin{equation}
\label{eq:action-conditioned-relevance}
\E[G\mid H,A^{(1)}]-\E[G\mid H]
\end{equation}
is nonconstant in one and vanishes almost surely in the other. Thus the reward distribution alone does not identify this contribution quantity. The construction embeds in any system with at least two agents, in the same or different layers.
\end{proposition}
The proposition concerns the law of $G$ alone: data pairing each reward with the agents' actions identify \eqref{eq:action-conditioned-relevance} (Lemma~\ref{lem:mask-target}).

\subsection{Routed Systems}
\label{sec:routed}
A \textit{routed system} has a single layer of parallel agents $N=\{1,\dots,K\}$. Given their actions $A^{(1:K)}$, the mechanism draws $I\sim p(\cdot\mid H,A^{(1:K)})$ and deploys $Y=A^{(I)}$. We consider softmax routing with uniform exploration,
\begin{equation}
\label{eq:softmax-router}
p_j
=
(1-\epsilon)\,
\frac{e^{s(H,A^{(j)})/\tau}}
{\sum_{k=1}^{K}e^{s(H,A^{(k)})/\tau}}
+\frac{\epsilon}{K},
\end{equation}
where $s$ scores each candidate, $\tau>0$ is a temperature, and
$\epsilon\in[0,1]$ mixes in uniform exploration. With
$R^{(j)}:=r(H,A^{(j)})$, the deployed reward is
$G=R^{(I)}$.
\textit{Winner-take-all} is the allocation rule
$\alpha_i=\1_{\{I=i\}}$: it trains only the selected agent, and $U_i=\E[\1_{\{I=i\}}G]$. Define the competitors' expected return and its conditional means
\[
D_i:=\sum_{j\ne i}p_jR^{(j)},
\qquad
d_i(H,A^{(i)}):=\E[D_i\mid H,A^{(i)}],
\qquad
\bar d_i(H):=\E\left[D_i\mid H\right],
\]
so that $d_i-\bar d_i$ measures how much agent $i$'s output moves its competitors' return. Raising its own selection probability necessarily lowers theirs.

\paragraph{Example: two routed agents.}
Fix a prompt and two answers with rewards $R^{(1)},R^{(2)}\in\{0,1\}$; the router deploys agent~1's answer with probability $0.8$ and agent~2's with $0.2$. Shared reward pays each agent $G$, with mean $0.8R^{(1)}+0.2R^{(2)}$; winner-take-all pays agent~1 only when it is deployed, $0.8R^{(1)}$ on average. Agent~1's \emph{removal contribution} is the drop in expected reward when agent~1 is removed and agent~2 is deployed for sure; the return it \emph{displaces} is what agent~2 loses, $0.8R^{(2)}$:
\begin{center}
\small
\setlength{\tabcolsep}{3pt}
\begin{tabular}{@{}ccccc@{}}
\toprule
$(R^{(1)},R^{(2)})$ & shared reward (mean) & winner-take-all pays agent~1 & agent~1's contribution & displaced return\\
\midrule
$(1,1)$ & 1.0 & 0.8 & 0 & 0.8\\
$(1,0)$ & 0.8 & 0.8 & 0.8 & 0\\
$(0,1)$ & 0.2 & 0 & $-0.8$ & 0.8\\
\bottomrule
\end{tabular}
\end{center}
In every row winner-take-all pays agent~1 its contribution plus the return it displaces. When both answers are correct, agent~1 adds nothing, yet winner-take-all pays it $0.8$: it pays for being selected, not for contributing. Theorem~\ref{thm:wta} is the gradient form of this identity.

\begin{theorem}[Winner-take-all bias is displaced return]
\label{thm:wta}
Fix $i$ and suppose the router depends on $\theta_i$ only through the
candidates. Then
\begin{equation}
\label{eq:wta-gradient}
\E\left[\1_{\{I=i\}}\Psi_iG\right]=\nabla_{\theta_i}U_i,
\end{equation}
so the winner-take-all update performs exact gradient ascent on private
utility, and its bias for the system gradient,
$b_i:= \E[\1_{\{I=i\}} \Psi_i G] - \nabla_{\theta_i}J_{\mathrm{sys}}$, satisfies
\begin{equation}
\label{eq:displacement-bias-identity}
b_i
=
-\nabla_{\theta_i}\sum_{j\ne i}U_j
=
-\E\left[\Psi_i\left(d_i(H,A^{(i)})-\bar d_i(H)\right)\right].
\end{equation}
In particular, $b_i=0$ if and only if $d_i(H,A^{(i)})-\bar d_i(H)$ is $L^2$-orthogonal to every coordinate of $\Psi_i$. Two score-free conditions also suffice. The bias vanishes if agent $i$'s output carries no information about the displaced return, $\E[D_i\mid H,A^{(i)}]=\E[D_i\mid H]$. Under \eqref{eq:layer-factorization}, it also vanishes if the router is candidate-blind for agent $i$: no $p_j$, $j\ne i$, changes as $A^{(i)}$ varies with $(H,A^{(-i)})$ fixed.
\end{theorem}

A router that ignores the candidates has $b_i=0$; one that rates them displaces return.

\section{Marginal Contribution} \label{sec:axioms}

We now ask what a contribution is. Rather than propose a formula, we state the properties a contribution must have and let them determine it. From here on we fix one layer and drop its index: $N=\{1,\dots,K\}$ are its agents, $H$ is its history, $A^{(S)}$ collects the outputs of the agents in $S\subseteq N$, and $G$ is still the terminal reward, after all later layers. Consider a set function $v:2^N\to\mathbb R$ that assigns a \emph{value} to each group; $v(S)$ includes value that exists only through interactions inside $S$. An evaluation rule, defined below, makes $v$ concrete.

Suppose first that each agent $i$ has one number $\Delta_i$, whatever other agents are present. We require \emph{totality}: $\sum_{i\in S}\Delta_i=v(S)$ for every $S\subseteq N$, so every group's value splits into its members' contributions. Totality forces $v(\varnothing)=0$ and $\Delta_i=v(\{i\})$, and it implies \emph{nullity}: an agent with $v(S\cup\{i\})=v(S)$ for every $S$ gets $\Delta_i=0$.

\begin{proposition}[Scalar credit exists only without interaction]
\label{prop:modularity}
For any $v:2^N\to\mathbb R$ with $v(\varnothing)=0$, the following are equivalent:
(i) scalars $\{\Delta_i\}_{i\in N}$ satisfying totality exist;
(ii) agents never interact: for all distinct $i,j\in N$ and $S\subseteq N\setminus\{i,j\}$,
\begin{equation}
\label{eq:no-pairwise-interaction}
v(S\cup\{i,j\})-v(S\cup\{i\})
=
v(S\cup\{j\})-v(S);
\end{equation}
(iii) $v$ is additive: $v(S)=\sum_{i\in S}v(\{i\})$ for every $S\subseteq N$.
Then the scalars are unique: $\Delta_i=v(\{i\})$.
\end{proposition}

One number per agent therefore cannot account for interaction (Remark~\ref{rem:pair} gives a two-agent example). We let an agent's contribution depend on its \emph{background}: $\Delta_{i\mid S}$ is agent $i$'s contribution when the agents in $S\subseteq N\setminus\{i\}$ are already present. A group can be built by adding its members in any order. \emph{Ordered totality} requires that along every order the contributions add up to the group's value: for every $S=\{i_1,\dots,i_m\}$ and every ordering of its members,
$\sum_{t=1}^m\Delta_{i_t\mid\{i_1,\dots,i_{t-1}\}}=v(S)$.

\begin{proposition}[Set-conditioned contribution is unique]
\label{prop:field-unique}
For any $v:2^{N}\to\mathbb R$ with $v(\varnothing)=0$, ordered totality
determines the contributions uniquely:
\[
\Delta_{i\mid S}=v(S\cup\{i\})-v(S)
\qquad
\text{for every }i\in N,\ S\subseteq N\setminus\{i\},
\]
and conversely these value differences satisfy ordered totality for every
such $v$.
\end{proposition}

We call $\{\Delta_{i\mid S}:i\in N,\ S\subseteq N\setminus\{i\}\}$ the \emph{contribution field}. Its variation across backgrounds is how interaction appears, and it collapses to one number per agent exactly when $v$ is additive. Uniqueness holds for a given $v$; which valuation to use is a modeling choice, made below. Proposition~\ref{prop:modularity} limits representing every group's value by fixed numbers, not training signals: the Shapley value averages $\Delta_{i\mid S}$ over uniformly random orders \citep{shapley1953value}, its weights sum to one, and it is therefore gradient-correct whenever every entry it averages is (Corollary~\ref{cor:retention-pg}).

\paragraph{Evaluation rules.}
An \emph{evaluation rule} $\rho$ specifies a counterfactual experiment that, given $H$ and the realized outputs $A^{(S)}$ of a background $S$, produces a terminal reward $G^\rho_S$, possibly using fresh randomness. Let $q^\rho(S):=\E[G^\rho_S\mid H,A^{(S)}]$ and
\begin{equation}
\label{eq:evaluation-rule-field}
v^\rho(S):=q^\rho(S)-q^\rho(\varnothing),
\qquad
\Delta^\rho_{i\mid S}:=v^\rho(S\cup\{i\})-v^\rho(S)=q^\rho(S\cup\{i\})-q^\rho(S);
\end{equation}
the centering gives $v^\rho(\varnothing)=0$ and cancels from every contribution. Under the \emph{retention} rule, omitted agents are averaged under their current policies:
\begin{equation}
\label{eq:retention-conditional-mean}
q^{\mathrm{ret}}(S)
:=
\E\left[G\mid H,A^{(S)}\right],
\end{equation}
so agent $i$'s contribution at the empty background is \eqref{eq:action-conditioned-relevance}. Under the \emph{removal} rule, omitted agents are deleted and the reduced system is run with the outputs in $S$ held fixed, giving $q^{\mathrm{rem}}(S)$ whenever the reduced system is well defined. Retention is defined for every system, including pipelines that stop working when an intermediate agent is deleted. Removal is defined when the reduced system still runs, as for a router over the remaining candidates. Every retention entry is a valid training signal, while under removal only the leave-one-out entry $\Delta^{\mathrm{rem}}_{i\mid N\setminus\{i\}}$ is guaranteed to be (Section~\ref{sec:signals}).

In the example of Section~\ref{sec:routed}, removal gives the contributions in its table, $\Delta^{\mathrm{rem}}_1=0.8(R^{(1)}-R^{(2)})$; retention keeps agent~2's answer and redraws agent~1's, so $\Delta^{\mathrm{ret}}_{1\mid\{2\}}$ compares agent~1's answer with its average.

\section{Valid Training Signals}
\label{sec:signals}

Which entries of the contribution field are gradient-correct training signals? The answer rests on one lemma: which changes to the reward leave an agent's expected policy gradient unchanged.

\begin{lemma}[Baseline invariance]
\label{lem:baseline}
Fix $i\in N$ and let a baseline $B_i$ be any square-integrable random variable
satisfying
\begin{equation}
\label{eq:baseline-cmi}
\E[B_i\mid H,A^{(i)}]
=
\E\left[B_i\mid H\right].
\end{equation}
Then
\begin{equation}
\label{eq:baseline-gradient}
\E\left[\Psi_i\left(G-B_i\right)\right]
=
\nabla_{\theta_i}J_{\mathrm{sys}} .
\end{equation}
The baseline may also depend on a randomly sampled variable $S$, provided $S\perp A^{(i)}\mid H$ and
$
\E[B_i(S)\mid H,A^{(i)},S]
=
\E\left[B_i(S)\mid H,S\right];
$
for example, $S$ may be a randomly sampled background.
\end{lemma}

Not every entry of a contribution field is gradient-correct. Under removal, both endpoints of an entry evaluate systems with agents deleted, and their difference need not recover the factual gradient; the exception is the leave-one-out entry, whose upper endpoint is the factual system (Corollary~\ref{cor:routing-removal-gradient}).

\begin{corollary}[Every entry of the retention field is a valid signal]
\label{cor:retention-pg}
Under the retention rule, for every
$S\subseteq N\setminus\{i\}$, fixed or drawn conditionally
independently of $A^{(i)}$ given $H$,
\begin{equation}
\label{eq:retention-pg}
\E\left[\Psi_i\,\Delta^{\mathrm{ret}}_{i\mid S}\right]
=
\E\left[\Psi_i\left(G-q^{\mathrm{ret}}(S)\right)\right]
=
\nabla_{\theta_i}J_{\mathrm{sys}} .
\end{equation}
\end{corollary}
Moreover, $\E[G-q^{\mathrm{ret}}(S)\mid H,A^{(S\cup\{i\})}]=\Delta^{\mathrm{ret}}_{i\mid S}$, so the residual $G-q^{\mathrm{ret}}(S)$ is unbiased for the contribution.

\paragraph{Variance.} Since every valid signal has the same expected gradient, a contribution signal can help only by lowering variance. In words, part~(iii) below says that the gradient noise of any signal of the form reward minus a baseline plus zero-mean noise has three parts: a floor that no such signal removes, the distance of its baseline from the best one, and the noise its estimate adds. Part~(i) concerns the reward's variance; parts~(ii)--(iii) concern the gradient's, which training uses. Write $\omega_i:=\norm{\Psi_i}^2$ and $\bar\omega_i:=\E[\omega_i\mid H]$, and for $T\subseteq N\setminus\{i\}$ let
$B^\star_{T}:=\E[\omega_iG\mid H,A^{(T)}]/\bar\omega_i$
on $\{\bar\omega_i>0\}$ (and $0$ elsewhere): the expected reward given $A^{(T)}$, weighted by the squared score norm \citep{weaver2001optimal}.

\begin{proposition}[Reward variance and gradient variance]
\label{prop:variance-monotone}
Fix $i\in N$.
\begin{enumerate}[label=(\roman*),leftmargin=*,itemsep=1pt,topsep=1pt]
\item For nested backgrounds $S\subseteq T\subseteq N$,
$\E[\Var(G\mid H,A^{(T)})]
\le
\E[\Var(G\mid H,A^{(S)})]$.
\item \emph{Best baseline.} Let $T\subseteq N\setminus\{i\}$. For every bounded baseline $B$ that is a function of $(H,A^{(T)})$,
\begin{equation}
\label{eq:baseline-variance-identity}
\operatorname{tr}\operatorname{Cov}\left(\Psi_i(G-B)\right)
=
\operatorname{tr}\operatorname{Cov}\left(\Psi_i(G-B^\star_T)\right)
+
\E\left[\bar\omega_i\left(B-B^\star_T\right)^2\right].
\end{equation}
Hence $B^\star_T$ is the best such baseline, the best attainable variance does not increase as $T$ grows, and the retention baseline $q^{\mathrm{ret}}(T)$ is best if and only if $\operatorname{Cov}(\omega_i,G\mid H,A^{(T)})=0$ almost surely on $\{\bar\omega_i>0\}$.
\item \emph{Three-term decomposition.} Let $\bar g:=\E[G\mid H,A^{(N)}]$, and let $X$ be a signal with $\E[X^2]<\infty$, $\E[\omega_iX^2]<\infty$, and $\E[X\mid H,A^{(N)}]=\bar g-\beta$, where $\beta$ is a function of $(H,A^{(N\setminus\{i\})})$. Write $\zeta:=X-(\bar g-\beta)$ and $B^\star:=B^\star_{N\setminus\{i\}}$. Then $X$ is gradient-correct and
\begin{equation}
\label{eq:three-term}
\operatorname{tr}\operatorname{Cov}\left(\Psi_iX\right)
=
\underbrace{\operatorname{tr}\operatorname{Cov}\left(\Psi_i(\bar g-B^\star)\right)}_{\text{irreducible}}
+
\underbrace{\E\left[\bar\omega_i(\beta-B^\star)^2\right]}_{\text{baseline mismatch}}
+
\underbrace{\E\left[\omega_i\zeta^2\right]}_{\text{noise}}.
\end{equation}
\end{enumerate}
\end{proposition}

Part~(i) alone does not imply lower gradient variance: the score norm $\omega_i$ can correlate with the reward, and then the retention baseline is not the best one (part~(ii)). Signals in the class of part~(iii) share the first term and differ only in $\beta$ and $\zeta$; signals outside it, such as an exact contribution $\Delta^{\mathrm{ret}}_{i\mid S}$, can have lower variance but require exact conditional means. Shared reward has $\beta=0$ and noise $\nu:=G-\bar g$, the variation added by the mechanism and later layers once this layer's outputs are fixed. In the example of Section~\ref{sec:routed}, exact removal has $\beta=R^{(2)}$, what agent~2 earns alone, and shared reward's noise $\nu$: the two differ only in how far $\beta$ is from $B^\star$. For the single-rollout score-function update, a signal of this form has lower gradient variance than shared reward exactly when its mismatch plus noise is below that of shared reward, $\E[\bar\omega_i(B^\star)^2]+\E[\omega_i\nu^2]$. Lower variance matters through the step size: if $J_{\mathrm{sys}}$ is $L$-smooth, an SGD step $\theta_i+\eta u$ along an unbiased update $u$ raises it in expectation whenever $\eta<2\norm{\nabla_{\theta_i}J_{\mathrm{sys}}}^2/(L\,\E\norm{u}^2)$ \citep{bottou2018optimization}, so a signal with a smaller second moment tolerates a larger learning rate; AdamW divides each coordinate by its root mean square, so its expected step is also larger. Table~\ref{tab:signals} lists $\beta$ and $\zeta$ for each signal we use.

\section{Learning Marginal Contributions}
\label{sec:critic}

Contributions need $q(S)$ for many backgrounds, and evaluating all $2^{K}$ of them is exponentially expensive, so a critic learns them from training data: given $(H,A^{(N)},G)$, draw a background $S$ independently of $(A^{(N)},G)$ given $H$, with positive probability on every background to be identified, and fit a critic $\hat q_{\phi}(H,m(S),A^{(S)})$, which sees the mask $m(S)\in\{0,1\}^{K}$ of $S$ and the outputs in $S$, by minimizing
\begin{equation}
\label{eq:mask-loss}
\mathcal L(\phi)
=
\E\left[
\left(
G-\hat q_{\phi}
\left(H,m(S),A^{(S)}\right)
\right)^2
\right].
\end{equation}

\begin{lemma}
\label{lem:mask-target}
Suppose a mask is sampled as detailed above and the critic class contains the
conditional mean of $G$ given its inputs. Then for every
$S\subseteq N$ in its support, any population minimizer of \eqref{eq:mask-loss} satisfies
\begin{equation}
\label{eq:critic-target}
\hat q_{\phi}
\!\left(H,m(S),A^{(S)}\right)
=
\E\left[G\mid H,A^{(S)}\right]
\end{equation}
almost surely.
\end{lemma}
Lemma~\ref{lem:mask-target} is a population statement with no finite-sample rate: the critic learns the retention values \eqref{eq:retention-conditional-mean}, we drop the superscript, and $\widehat\Delta_{i\mid S}:=G-\hat q_{\phi}(S)$.

\begin{proposition}[Gradient correctness and attribution error]
\label{prop:critic}
Let $S\subseteq N\setminus\{i\}$ be fixed or drawn conditionally
independently of $A^{(i)}$ given $H$, and let
$\hat q_{\phi}$ be any fixed or cross-fitted square-integrable
critic that excludes $A^{(i)}$. Then
\begin{align}
\label{eq:critic-gradient}
\E\left[\Psi_i\,\widehat\Delta_{i\mid S}\right]
&=
\nabla_{\theta_i}J_{\mathrm{sys}},\\
\label{eq:critic-bias}
\E\left[
\widehat\Delta_{i\mid S}
\,\Big|\,
H,A^{(S\cup\{i\})}
\right]
-
\Delta_{i\mid S}
&=
\E\left[
G\mid H,A^{(S)}
\right]
-
\hat q_{\phi}(S).
\end{align}
\end{proposition}
\vspace{-2pt}
Critic accuracy is therefore unnecessary for gradient correctness, and for attribution the critic's error at $S$ passes through unchanged; the critic difference $\hat q_{\phi}(S\cup\{i\})-\hat q_{\phi}(S)$ reads $A^{(i)}$ and is not covered.

\subsection{Learning Removal Contributions under Routing}
\label{sec:routing-est}

Under routing, the data contain every candidate, so retention values remain learnable, but only the selected candidate is rewarded. Assume selection does not change a candidate's expected reward,
\begin{equation}
\label{eq:candidate-outcome}
\E\left[R^{(j)}\mid H,A^{(1:K)},I=j\right]
=
\E\left[R^{(j)}\mid H,A^{(j)}\right]
=:
g_j(H,A^{(j)}),
\end{equation}
so that, given the outputs, the router's expected reward is $\bar g=\sum_{j}p_jg_j$. Let $p_j^{-i}$ be the probabilities that rule \eqref{eq:softmax-router} assigns when run on the $K-1$ candidates $j\ne i$ (Appendix~\ref{app:dr-details}); for $\epsilon>0$ this differs from renormalizing $p_j$, which would depend on $A^{(i)}$. The removal contribution
$\Delta_i^{\mathrm{rem}}:=\sum_{j}p_jg_j-\sum_{j\ne i}p_j^{-i}g_j$
is the value of making candidate $i$ available to the router. The reward after removing candidate $i$ does not depend on $A^{(i)}$, so it is a valid baseline:

\begin{corollary}[Removal is a valid routing signal]
\label{cor:routing-removal-gradient}
Let $G^{-i}_{\mathrm{route}}$ be a draw from the router after candidate $i$ is
removed, conditional on the realized remaining candidates. Then
\begin{equation}
\label{eq:routing-removal-gradient}
\E\left[
\Psi_i\left(G-G^{-i}_{\mathrm{route}}\right)
\right]
=
\nabla_{\theta_i}J_{\mathrm{sys}} .
\end{equation}
The identity remains valid if $G^{-i}_{\mathrm{route}}$ is replaced by any
fixed or cross-fitted square-integrable predictor based only on
$H$ and $A^{(-i)}$.
\end{corollary}

The second part covers the \emph{plug-in removal} signal $G-\sum_{j\ne i}p_j^{-i}\widehat\mu_j$, where $\widehat\mu_j$ is a learned estimate of $g_j$ from $(H,A^{(j)})$: it is gradient-correct but not unbiased for $\Delta^{\mathrm{rem}}_i$. Since $g_j$ is observed only when $j$ is selected, a doubly robust (DR) estimator recovers it from the same data:
\begin{equation}
\label{eq:dr-candidate}
\widehat g_j
=
\widehat\mu_j
+
\frac{\1_{\{I=j\}}}{p_j}
\left(G-\widehat\mu_j\right),
\qquad
\widehat\Delta_i^{\mathrm{rem}}
=
\sum_{j=1}^{K}p_j\widehat g_j
-
\sum_{j\ne i}p_j^{-i}\widehat g_j .
\end{equation}
The propensity term corrects for candidate $j$'s reward being seen only when $j$ is selected \citep{dudik2011doubly} (Appendix~\ref{app:dr-details}).

\begin{corollary}[The removal contribution is identified from training data]
\label{cor:routing-dr}
Under \eqref{eq:candidate-outcome}, suppose
$p_j\ge p_{\min}>0$ for every candidate and treat $p_j$ and $p_j^{-i}$ as known.
Conditional on $H$ and the realized actions,
$\widehat\Delta_i^{\mathrm{rem}}$ is unbiased for
$\Delta_i^{\mathrm{rem}}$ for any fixed or cross-fitted outcome model
$\widehat\mu_j$.
\end{corollary}

Agent $i$'s output enters $\widehat\Delta_i^{\mathrm{rem}}$ through every $p_j$ and the selection event, so Corollary~\ref{cor:routing-removal-gradient} does not cover it; it is still gradient-correct:

\begin{corollary}[The estimated removal signal is gradient-correct]
\label{cor:dr-gradient}
In the setting of Corollary~\ref{cor:routing-dr}, suppose the router and the bounded outcome model are fixed (or cross-fitted) during the update and $\widehat\Delta_i^{\mathrm{rem}}$ uses the exact factual propensities, without flooring or clipping. Then
$\E[\Psi_i\widehat\Delta_i^{\mathrm{rem}}]=\nabla_{\theta_i}J_{\mathrm{sys}}$, and Proposition~\ref{prop:variance-monotone}(iii) applies with $\beta=\sum_{j\ne i}p_j^{-i}g_j$, the baseline of exact removal.
\end{corollary}

Given all outputs, the mean of $\widehat\Delta_i^{\mathrm{rem}}$ is $\bar g-\sum_{j\ne i}p_j^{-i}g_j$: DR keeps the baseline of exact removal but replaces the router noise $\nu$ by its own estimation noise, which vanishes as $\widehat\mu\to g$ when rewards are functions of prompt and candidate, as on GSM8K.

\paragraph{Dropping the propensity correction of the baseline.}
In \eqref{eq:dr-candidate} the correction $\1_{\{I=j\}}(G-\widehat\mu_j)/p_j$ enters the factual sum with weight $p_j$, where the propensity cancels and leaves a control variate for the router's draw \citep{gu2017qprop,ciosek2018expected}, and the removal baseline with weight $p_j^{-i}$, where the ratio $p_j^{-i}/p_j$ remains. Attribution needs the second; gradient-correctness does not. The \emph{hybrid} signal keeps only the first, with an outcome model $\widehat\mu_j^{-i}$ in the baseline that does not read agent $i$'s output \citep{wu2018variance,castellini2021difference}:
\begin{equation}
\label{eq:hybrid}
X_i^{\mathrm H}
=
G-\widehat\mu_I+\sum_{j=1}^{K}p_j\widehat\mu_j
-\sum_{j\ne i}p_j^{-i}\,\widehat\mu_j^{-i}.
\end{equation}

\begin{corollary}[The hybrid signal]
\label{cor:hybrid}
With router and outcome models fixed during the update, $\widehat\mu_j$ bounded functions of $(H,A^{(1:K)})$, $\widehat\mu_j^{-i}$ bounded functions of $(H,A^{(-i)})$ (so neither reads $I$ or the observed $G$), and the exact $p_j$ in \eqref{eq:hybrid}, $X_i^{\mathrm H}$ is gradient-correct, and Proposition~\ref{prop:variance-monotone}(iii) applies with $\beta=\sum_{j\ne i}p_j^{-i}\widehat\mu_j^{-i}$ and $\E[\zeta^2\mid H,A^{(1:K)}]=\Var(G-\widehat\mu_I\mid H,A^{(1:K)})$, without a propensity ratio.
\end{corollary}

$X_i^{\mathrm H}$ is not unbiased for $\Delta_i^{\mathrm{rem}}$; Appendix~\ref{app:slates} extends the corollary to outcome models that read all $m$ completions and to the centred update.

\begin{table}[H]
\centering
\small
\setlength{\tabcolsep}{3pt}
\renewcommand{\arraystretch}{1.1}
\begin{tabular}{@{}llll@{}}
\toprule
Signal & $X_i$ & Gradient-correct & Baseline $\beta$; noise $\zeta$ in \eqref{eq:three-term} \\
\midrule
Winner-take-all (W) & $\1_{\{I=i\}}G$ & no (Thm.~\ref{thm:wta}) & not covered (biased) \\
Shared reward (SR) & $G$ & yes & $0$; $\nu$ \\
Shared, prompt baseline & $G-B(H)$ & yes (Lem.~\ref{lem:baseline}) & $B(H)$; $\nu$ \\
Retention residual & $G-\hat q(S)$ & yes (Prop.~\ref{prop:critic}) & $\hat q(S)$; $\nu$ \\
Critic difference & $\hat q(S\cup\{i\})-\hat q(S)$ & not guaranteed & -- \\
Retention leave-one-out & $G-q^{\mathrm{ret}}(N\setminus\{i\})$ & yes (Cor.~\ref{cor:retention-pg}) & $q^{\mathrm{ret}}(N\setminus\{i\})$; $\nu$ \\
Plug-in removal & $G-\sum_{j\ne i}p_j^{-i}\widehat\mu_j^{-i}$ & yes (Cor.~\ref{cor:routing-removal-gradient}) & $\sum_{j\ne i}p_j^{-i}\widehat\mu_j^{-i}$; $\nu$ \\
DR removal (DR, DR-agree) & $\widehat\Delta_i^{\mathrm{rem}}$, eq.~\eqref{eq:dr-candidate} & yes (Cor.~\ref{cor:dr-gradient}) & $\sum_{j\ne i}p_j^{-i}g_j$; $\zeta^{\mathrm{DR}}$, with $p_j^{-i}/p_j$ \\
Hybrid (H) & $X_i^{\mathrm H}$, eq.~\eqref{eq:hybrid} & yes (Cor.~\ref{cor:hybrid}) & $\sum_{j\ne i}p_j^{-i}\widehat\mu_j^{-i}$; $\nu-(\widehat\mu_I-\sum_jp_j\widehat\mu_j)$ \\
Best baseline & $G-B^\star$ & yes (Prop.~\ref{prop:variance-monotone}) & $B^\star$; $\nu$ \\
\bottomrule
\end{tabular}
\caption{Training signals for agent $i$ in a routed layer. Gradient-correct signals differ only in $\beta$ and $\zeta$ (Proposition~\ref{prop:variance-monotone}(iii)); $\nu=G-\bar g$ is router noise; $\zeta^{\mathrm{DR}}$: Appendix~\ref{app:slates}. Abbreviations mark the signals of Section~\ref{sec:experiment}; DR-agree is DR with the agreement model of Section~\ref{sec:noise}.}
\label{tab:signals}
\end{table}

\input{sections/experiment}

\section{Conclusion}
\label{sec:conclusion}
We proved that winner-take-all is biased by exactly the return each agent displaces, that contribution must depend on which other agents are present, that every baseline ignoring an agent's own output keeps the system gradient, and that the resulting gradient variance splits into an irreducible term, a baseline mismatch and noise, which gives the best baseline and decides when a signal beats shared reward. Under routing, removal contributions are identified from logged data without extra generation. The experiments confirm this mechanism: measured noise followed the decomposition, which diagnosed and fixed a noisy estimator, and with nine agents a contribution signal learned faster than shared reward in a pre-registered test and tolerated twice the learning rate without collapsing.

\textbf{Limitations.}
Experiments cover routing, one model, one training dataset, and the update the theory covers, not GRPO's standardization. Nine-agent teams do not interact, and there winner-take-all was not beaten. Lemma~\ref{lem:mask-target} and Theorem~\ref{thm:selective} are population and oracle results.

\section*{AI Use Statement}
We used generative AI tools for two tasks. First, to write and debug the code for the experiments in Section~\ref{sec:experiment} and the appendices (training, measurement and analysis scripts), following the authors' specifications. Second, to go over the text of the paper and edit and polish it for clarity and grammar. Code written with AI assistance was checked by unit tests and by independent reimplementations of the training signals and analyses, and every number reported in the paper is generated by script from the stored result files. The authors reviewed all AI-assisted code and text and take full responsibility for the content of the paper, including its claims, proofs and results.

\bibliographystyle{unsrt}
\bibliography{references}

\newpage
\appendix
\raggedbottom

\section{Safety and Alignment}

As AI systems gain increasing capability and power, the authors of this paper are deeply concerned with the
alignment of modern ML systems to humanity and its goals,  as well as the safety concerns which it raises. We believe one crucial part of alignment
is mechanistic understanding of these systems, which explains their
behavior and their decisions. This belief is a main motivation for this paper: as ML systems become more agentic, autonomous, and
decentralized, causally understanding the roles and contributions of
individual agents is necessary for keeping them safe.

Our results sharpen this concern. In multi-agent training, the credit
signal defines what each agent is trained to want. Under the default
winner-take-all signal, agents are trained to want selection, not
system success. This pressure can also favor risky outputs over good safe ones, and none of this need surface in reward metrics. Any gradient-correct signal, including shared reward, removes this incentive mismatch; contribution signals aim to do so with lower variance. We consider gradient correctness a minimal alignment property of any training signal derived from a shared reward.

In the future, we plan to investigate the mean-field limit of
multi-agent systems: whether the displacement
bias survives the many-agent limit, where each agent's bias shrinks
because no single candidate can move the router much while the
population as a whole may still be trained toward the wrong collective
behavior, a misalignment invisible agent by agent; and whether the
exponentially many entries of the contribution field collapse, for
large populations of agents, to a function of the
background's size and its mix of policies.

\section{Optimal Allocation of Counterfactual Evaluations}
\label{sec:selective}
Work on allocating rollouts in RL with verifiable rewards targets gradient variance \citep{vip_2026,trace_2026}; here we allocate exact counterfactual evaluations, which gives the classical square-root rule \citep{neyman1934}.

Learned contributions can be wrong. An exact evaluation reruns the counterfactual that defines $\Delta_{i\mid S}$, at the cost of extra model calls, and returns $Y_{i\mid S}=G^\rho_{S\cup\{i\}}-G^\rho_S$, with mean $\Delta_{i\mid S}$, variance $\sigma_{i\mid S}^2$, and cost $c_{i\mid S}>0$. We run it with probability $p_{i\mid S}$, indicated by $Z_{i\mid S}\sim\mathrm{Bernoulli}(p_{i\mid S})$, and write $\eta_{i\mid S}:=\widehat\Delta_{i\mid S}-\Delta_{i\mid S}$ for the error of the learned estimate, which is fixed given the data, and $a_i:=\|\Psi_i\|$.

\begin{theorem}[Optimal allocation of counterfactual evaluations (oracle)]
\label{thm:selective}
Condition on the factual batch and the learned estimates, and assume the pairs $(Z_{i\mid S},Y_{i\mid S})$ are independent across contributions, $Z_{i\mid S}\perp Y_{i\mid S}$, and $p_{i\mid S}>0$. The correction $\widetilde\Delta_{i\mid S}=\widehat\Delta_{i\mid S}+\frac{Z_{i\mid S}}{p_{i\mid S}}(Y_{i\mid S}-\widehat\Delta_{i\mid S})$ is unbiased for $\Delta_{i\mid S}$ for every $p$, with mean-squared error $(\sigma_{i\mid S}^2+\eta_{i\mid S}^2)/p_{i\mid S}-\eta_{i\mid S}^2$. Restrict to the contributions with $a_i^2(\sigma_{i\mid S}^2+\eta_{i\mid S}^2)>0$; the rest are never worth evaluating. Minimizing the gradient-weighted risk $\mathcal R(p)=\sum a_i^2\,\E[(\widetilde\Delta_{i\mid S}-\Delta_{i\mid S})^2]$ over $0<p_{i\mid S}\le1$ under the budget $\sum c_{i\mid S}\,p_{i\mid S}\le C_{\max}$ gives $p^\star_{i\mid S}=1$ everywhere when $C_{\max}\ge\sum c_{i\mid S}$, and otherwise the unique optimum
\begin{equation}
\label{eq:optimal-query}
p_{i\mid S}^\star
=
\min\left\{
1,\;
\lambda\,
\frac{
a_i\sqrt{\sigma_{i\mid S}^2+\eta_{i\mid S}^2}
}{
\sqrt{c_{i\mid S}}
}
\right\},
\qquad\text{with $\lambda$ set by the budget.}
\end{equation}
\end{theorem}

An entry is evaluated more often when it has more gradient leverage, a larger error or evaluation variance, and a lower cost; the budget affects the variance of the corrected estimate, not its unbiasedness. Appendix~\ref{subseq:allocations} gives the attained risk and a matching lower bound. Theorem~\ref{thm:selective} is an oracle result: it assumes $\sigma^2_{i\mid S}+\eta^2_{i\mid S}$ is known. Given the data this sum equals $\E[(Y_{i\mid S}-\widehat\Delta_{i\mid S})^2]$, which a pilot round of uniform exact evaluations can estimate. Appendix~\ref{app:implementation} evaluates the rule with such a pilot, and its cost, on logged slates; training does not use Theorem~\ref{thm:selective}.

\section{Proofs of Main Results}
\label{app:proofs}

The proofs keep the layer index $\ell$ that Sections~\ref{sec:axioms}--\ref{sec:critic} drop.

\subsection{Proof of Proposition~\ref{prop:shared-credit}}

\begin{proof}
Consider two systems with mechanisms $M$ and $M'$. Each has at least two agents $\{1,2\}$ either in the same layer or in separate layers. The agents are identical in both systems. In addition, the systems have an identical constant initial history $H_1=h$, and independent actions
\[
A^{(1)},A^{(2)}
\sim
\operatorname{Bernoulli}(1/2).
\]
Define the identity reward function in both systems,
\[
r(h,y):=y \in \{0,1\}.
\]

In the first system, $M$ deploys agent $1$'s output:
\[
Y=M(H_1,A)=A^{(1)},
\qquad
G=r(H_1,Y)=A^{(1)}.
\]
In the second system, the mechanism deploys agent $2$'s output:
\[
Y'=M'(H_1,A)=A^{(2)},
\qquad
G'=r(H_1,Y')=A^{(2)}.
\]
Both deployed rewards are distributed as
$\operatorname{Bernoulli}(1/2)$, so
\[
G\overset{d}{=}G'.
\]

In the first system,
\[
\E[G\mid H_1,A^{(1)}]
-
\E\left[G\mid H_1\right]
=
A^{(1)}-\frac{1}{2} \in\left\{ \pm\frac12\right\}.
\]
In the second system, conditional independence gives
\[
\E[G'\mid H_1,A^{(1)}]
-
\E[G'\mid H_1]
=
\E[A^{(2)}\mid H_1,A^{(1)}]
-
\frac{1}{2}
=
0
\]
almost surely.

Together, the two systems have identical policies and identical reward laws,
yet agent $1$'s action determines the reward in the first and never affects it
in the second, as witnessed by \eqref{eq:action-conditioned-relevance}. A marginal
contribution measure must therefore assign them different values, while any
mapping defined on the deployed-reward law alone receives the same input in both cases, and thus will return the same value. This follows the classical treatment of identification as in
\citep{rothenberg1971identification}, where identifiability of a quantity from given
data means the data determines it.
\end{proof}

\subsection{Proof of Theorem \ref{thm:wta}}
\begin{proof}
Consider a single routed layer with agents \(N=\{1,\ldots,K\}\). Since
\(\theta_i\) enters the system only through agent \(i\)'s policy, the
score-function identity gives
\[
\nabla_{\theta_i} U_i
=
\nabla_{\theta_i}\E[\1_{\{I=i\}} G]
=
\mathbb{E}\left[
\mathbf{1}_{\{I=i\}} \Psi_i G
\right].
\]
Equivalently for the system
\[
\nabla_{\theta_i} J_{\mathrm{sys}}
=
\mathbb{E}\left[
\Psi_i G
\right].
\]
Using \(J_{\mathrm{sys}}=\sum_j U_j\),
\[
b_i
=
\nabla_{\theta_i} U_i
-
\nabla_{\theta_i} J_{\mathrm{sys}}
=
-
\nabla_{\theta_i}
\sum_{j\ne i} U_j.
\]
From the above we get
\[
b_i
=
-
\mathbb{E}\left[
\Psi_i \mathbf{1}_{\{I\ne i\}} G
\right].
\]

We now average out the router randomness. Conditional on the realized
history and candidates \((H,A^{(N)})\), the router selects candidate
\(j\) with probability \(p_j\), and on the event \(\{I=j\}\) the deployed
reward is \(G=R^{(j)}\). Thus
\[
\mathbf{1}_{\{I\ne i\}}G
=
\sum_{j\ne i}
\mathbf{1}_{\{I=j\}}R^{(j)},
\]
and therefore
\begin{align*}
\mathbb{E}\left[
\mathbf{1}_{\{I\ne i\}}G
\mid
H,A^{(N)}
\right]
&=
\sum_{j\ne i}
R^{(j)}
\mathbb{E}\left[
\mathbf{1}_{\{I=j\}}
\mid
H,A^{(N)}
\right]\\
&=
\sum_{j\ne i}
R^{(j)}
\mathbb{P}\left(
I=j
\mid
H,A^{(N)}
\right)\\
&=
\sum_{j\ne i}p_jR^{(j)}
=
D_i.
\end{align*}
Since \(\Psi_i\) is determined by \((H,A^{(i)})\), and hence is measurable
with respect to \((H,A^{(N)})\), the tower property gives
\begin{align*}
\mathbb{E}\left[
\Psi_i\mathbf{1}_{\{I\ne i\}}G
\right]
&=
\mathbb{E}\left[
\mathbb{E}\left[
\Psi_i\mathbf{1}_{\{I\ne i\}}G
\mid
H,A^{(N)}
\right]
\right]\\
&=
\mathbb{E}\left[
\Psi_i
\mathbb{E}\left[
\mathbf{1}_{\{I\ne i\}}G
\mid
H,A^{(N)}
\right]
\right]\\
&=
\mathbb{E}\left[
\Psi_iD_i
\right].
\end{align*}
Hence
\[
b_i
=
-
\mathbb{E}\left[
\Psi_iD_i
\right].
\]

We next average out the competitors' actions while retaining
\((H,A^{(i)})\). Define
\[
d_i(H,A^{(i)})
=
\mathbb{E}\left[
D_i
\mid
H,A^{(i)}
\right].
\]
Since \(\Psi_i\) is measurable with respect to \((H,A^{(i)})\), another
application of the tower property gives
\begin{align*}
\mathbb{E}\left[
\Psi_iD_i
\right]
&=
\mathbb{E}\left[
\mathbb{E}\left[
\Psi_iD_i
\mid
H,A^{(i)}
\right]
\right]\\
&=
\mathbb{E}\left[
\Psi_i
\mathbb{E}\left[
D_i
\mid
H,A^{(i)}
\right]
\right]\\
&=
\mathbb{E}\left[
\Psi_i d_i(H,A^{(i)})
\right].
\end{align*}
Moreover, the score identity gives
\[
\mathbb{E}\left[
\Psi_i
\mid H
\right]
=
0.
\]
Thus, considering
\[
\bar d_i(H)
=
\mathbb{E}\left[
D_i
\mid H
\right]
\]
which depends only on \(H\),
\[
\mathbb{E}\left[
\Psi_i \bar d_i(H)
\right]
=
0.
\]
Subtracting this zero term yields
\[
b_i
=
-
\mathbb{E}\left[
\Psi_i
\left(
d_i(H,A^{(i)})
-
\bar d_i(H)
\right)
\right],
\]
which proves the result. It follows that \(b_i=0\) if and only if
\[
d_i(H,A^{(i)})-\bar d_i(H)
\]
is \(L^2\)-orthogonal to every coordinate of \(\Psi_i\). In particular, if
\[
\mathbb{E}\left[
D_i
\mid
H,A^{(i)}
\right]
=
\mathbb{E}\left[
D_i
\mid H
\right],
\]
then
\[
d_i(H,A^{(i)})
=
\bar d_i(H)
\]
almost surely, and therefore \(b_i=0\).

Finally, suppose the router is candidate-blind for agent \(i\). Under the
layer-wise conditional independence assumption,
\[
A^{(-i)}
\perp
A^{(i)}
\mid H.
\]
Candidate-blindness ensures that, with \((H,A^{(-i)})\) fixed, the
competitors' routing probabilities \(p_j\), \(j\ne i\), do not vary with
\(A^{(i)}\). Hence
\[
\mathbb{E}\left[
D_i
\mid
H,A^{(i)}
\right]
=
\mathbb{E}\left[
D_i
\mid H
\right],
\]
and therefore \(b_i=0\).
\end{proof}

\subsection{Proof of Proposition~\ref{prop:modularity}}

\begin{proof}
We show that (i) and (iii) are equivalent, and that (ii) and (iii) are equivalent.

First suppose (i): scalars $\{\Delta_i\}_{i\in N}$ satisfy totality. The
singleton cases already determine every credit: taking $S=\{i\}$ gives
\[
\Delta_i=v(\{i\}).
\]
Thus the credits are unique. Substituting these forced values back into
totality gives
\[
v(S)=\sum_{i\in S}v(\{i\})
\qquad\text{for every }S\subseteq N,
\]
which is (iii).

Conversely, suppose (iii). Define $\Delta_i:=v(\{i\})$. Then
\[
\sum_{i\in S}\Delta_i
=
\sum_{i\in S}v(\{i\})
=
v(S),
\]
so totality holds. Recall that nullity holds as well: if $i$ is null, then
$v(\{i\})=v(\varnothing)=0$, and hence $\Delta_i=0$. Thus (i) holds.

It remains to show that (ii) and (iii) are equivalent. If $v$ is additive,
then adding agent $j$ to any coalition not containing it changes the value
by exactly $v(\{j\})$. Hence both sides of
\eqref{eq:no-pairwise-interaction} equal $v(\{j\})$, so (ii) holds.

Now suppose (ii). We prove additivity by induction on $|S|$. The claim is
immediate for $|S|\leq 1$. Suppose it holds for all coalitions of size at
most $m$, and let $|S|=m+1$. Choose distinct $i,j\in S$ and write
$T=S\setminus\{i,j\}$. By \eqref{eq:no-pairwise-interaction},
\[
v(S)
=
v(T\cup\{i\})+v(T\cup\{j\})-v(T).
\]
All three coalitions on the right have size at most $m$, so the inductive
hypothesis gives
\[
v(S)
=
\sum_{k\in T}v(\{k\})+v(\{i\})+v(\{j\})
=
\sum_{k\in S}v(\{k\}).
\]
Thus (iii) holds, completing the proof.
\end{proof}
\begin{remark}[A complementary pair] \label{rem:pair}
Take $N=\{1,2\}$ with $v(\{1\})=v(\{2\})=0$ and $v(\{1,2\})=1$: the
agents are valuable only together. No scalars satisfy totality, since
$\Delta_1+\Delta_2$ would need to equal both $0$ and $1$, and condition
\eqref{eq:no-pairwise-interaction} fails. The contribution field records
the interaction instead: $\Delta_{1\mid\varnothing}=0$ while
$\Delta_{1\mid\{2\}}=1$.
\end{remark}

\subsection{Proof of Proposition~\ref{prop:field-unique}}

\begin{proof}
We prove the following: for every group $T\subseteq N$
and every ordering $i_1,\dots,i_m$ of its elements,
\begin{equation}
\label{eq:chain-exactness}
\sum_{k=1}^{m}\Delta_{i_k\mid\{i_1,\dots,i_{k-1}\}}
=
v(T),
\end{equation}
so that crediting the agents one at a time, each against those already
accounted for, recovers the group's value regardless of the order.

($\Rightarrow$) Suppose the candidates satisfy \eqref{eq:chain-exactness},
and fix any $i\in N$ and $S\subseteq N\setminus\{i\}$. The pair $(i,S)$ appears as the final step of a chain: order the
elements of $S$ arbitrarily as $i_1,\dots,i_m$ and append $i$ as element
$i_{m+1}$, which is an ordering of the group $S\cup\{i\}$. Apply
\eqref{eq:chain-exactness} twice, once to $S\cup\{i\}$ with this ordering
and once to $S$ with the same ordering less its final element:
\[
\sum_{k=1}^{m}\Delta_{i_k\mid\{i_1,\dots,i_{k-1}\}}
+
\Delta_{i\mid S}
=
v(S\cup\{i\}),
\qquad
\sum_{k=1}^{m}\Delta_{i_k\mid\{i_1,\dots,i_{k-1}\}}
=
v(S).
\]
The two chains share every term except the last, so subtracting leaves
\[
\Delta_{i\mid S}=v(S\cup\{i\})-v(S).
\]
Since $i$ and $S$ were arbitrary, this holds for every contribution.

($\Leftarrow$) Conversely, define
$\Delta_{i\mid S}:=v(S\cup\{i\})-v(S)$ for all admissible pairs and let
$T$ and an ordering $i_1,\dots,i_m$ be given. Writing
$T_k:=\{i_1,\dots,i_k\}$ for the group after $k$ steps, with
$T_0=\varnothing$, each summand is $v(T_k)-v(T_{k-1})$, and the sum
telescopes
\[
\sum_{k=1}^{m}\Delta_{i_k\mid T_{k-1}}
=
\sum_{k=1}^{m}\big(v(T_k)-v(T_{k-1})\big)
=
v(T_m)-v(T_0)
=
v(T),
\]
using $v(\varnothing)=0$. The differences therefore satisfy the demand for
every $v$, and for every ordering, completing the proof.
\end{proof}

\subsection{Proof of Lemma~\ref{lem:baseline}}

\begin{proof}
Throughout, fix the layer $\ell$ of agent $i$ and abbreviate
$H:=H_\ell$, $A^{(i)}:=A_\ell^{(i)}$. The proof has three steps: all
expectations exist; the score has conditional mean zero, which yields the
gradient identity for the raw reward; and the baseline condition
\eqref{eq:baseline-cmi} makes the subtracted term vanish.

\emph{Step 0: integrability.} The reward is bounded and
$\E\norm{\Psi_i}^2<\infty$ by the standing assumption, so $\E\norm{\Psi_iG}$
is finite. For the baseline term, Cauchy--Schwarz gives
\[
\E\norm{\Psi_iB_i}
\le
\left(\E\norm{\Psi_i}^2\right)^{1/2}
\left(\E\left[B_i^2\right]\right)^{1/2}
<\infty,
\]
since $B_i$ is square-integrable. Every expectation below is therefore well-defined. The standing assumption
permits differentiation under the expectations defining the system objective.

\emph{Step 1: the score has conditional mean zero.} Conditionally on $H$,
the action $A^{(i)}$ has density $\pi_{\theta_i}(\cdot\mid H)$, so
\[
\begin{aligned}
\E\left[\Psi_i\mid H\right]
&=
\int
\nabla_{\theta_i}\log\pi_{\theta_i}(a\mid H)\,
\pi_{\theta_i}(a\mid H)\,da
=
\int
\nabla_{\theta_i}\pi_{\theta_i}(a\mid H)\,da\\
&=
\nabla_{\theta_i}
\int
\pi_{\theta_i}(a\mid H)\,da
=
0,
\end{aligned}
\]
where the second equality is the identity
$\pi\nabla\log\pi=\nabla\pi$ and the last holds because the density
integrates to one for every $\theta_i$. The same computation with a
$\theta_i$-independent integrand gives the gradient identity: since
$\theta_i$ enters the joint law only through $\pi_{\theta_i}(A^{(i)}\mid H)$
(indices are local to a layer, and every agent of every layer has its own parameters),
by the factorization \eqref{eq:layer-factorization},
\[
\nabla_{\theta_i}J_{\mathrm{sys}}
=
\nabla_{\theta_i}\E\left[G\right]
=
\E\left[\Psi_iG\right].
\]

\emph{Step 2: the baseline vanishes.} It remains to show
$\E[\Psi_iB_i]=0$. Condition on $(H,A^{(i)})$, with respect to which
$\Psi_i$ is measurable, and apply the tower property together with the
hypothesis \eqref{eq:baseline-cmi}:
\[
\E\left[\Psi_iB_i\right]
=
\E\left[\Psi_i\,\E\left[B_i\mid H,A^{(i)}\right]\right]
=
\E\left[\Psi_i\,\E\left[B_i\mid H\right]\right]
=
\E\left[\E\left[B_i\mid H\right]\,\E\left[\Psi_i\mid H\right]\right]
=
0,
\]
where the third equality conditions on $H$ and uses that
$\E[B_i\mid H]$ is a function of $H$ alone, and the last is Step 1.
Subtracting, $\E[\Psi_i(G-B_i)]=\E[\Psi_iG]=\nabla_{\theta_i}J_{\mathrm{sys}}$,
which is \eqref{eq:baseline-gradient}.

\emph{Step 3: indexed baselines.}
Let $B_i=B_i(S)$ satisfy the indexed conditions in the lemma. Then
\begin{align*}
\E\left[
B_i(S)
\mid
H,A^{(i)}
\right]
&=
\E\left[
\E\left[
B_i(S)
\mid
H,A^{(i)},S
\right]
\middle|
H,A^{(i)}
\right]\\
&=
\E\left[
\E\left[
B_i(S)
\mid
H,S
\right]
\middle|
H,A^{(i)}
\right]\\
&=
\E\left[
\E\left[
B_i(S)
\mid
H,S
\right]
\middle|
H
\right]\\
&=
\E\left[
B_i(S)
\mid H
\right].
\end{align*}
The second equality is the indexed baseline condition, and the third uses
$S\perp A^{(i)}\mid H$. Thus the mixture satisfies
\eqref{eq:baseline-cmi}, and Step 2 applies. The same draw of $S$ may serve
every agent for which these two indexed conditions hold.
\end{proof}

\subsection{Proof of Corollary~\ref{cor:retention-pg}}

In general, $\Delta^\rho_{i\mid S}$ recovers the factual gradient whenever $q^\rho_\ell(S)$ is an admissible baseline under Lemma~\ref{lem:baseline} and $\E[\Psi_i\,q^\rho_\ell(S\cup\{i\})]=\E[\Psi_iG]$.

\begin{proof}
Fix $S\subseteq N_\ell\setminus\{i\}$ and abbreviate
$H:=H_\ell$. The baseline
\[
q^{\mathrm{ret}}_\ell(S)
=
\E\left[
G
\mid
H,A_\ell^{(S)}
\right]
\]
is bounded and depends only on $H$ and $A_\ell^{(S)}$. By
\eqref{eq:layer-factorization},
$A_\ell^{(S)}\perp A_\ell^{(i)}\mid H$, so
$q^{\mathrm{ret}}_\ell(S)$ satisfies \eqref{eq:baseline-cmi}. Lemma~\ref{lem:baseline}
therefore gives
\[
\E\left[
\Psi_i\left(G-q^{\mathrm{ret}}_\ell(S)\right)
\right]
=
\nabla_{\theta_i}J_{\mathrm{sys}}.
\]

By \eqref{eq:evaluation-rule-field} and
\eqref{eq:retention-conditional-mean},
\[
\Delta^{\mathrm{ret}}_{i\mid S}
=
q^{\mathrm{ret}}_\ell(S\cup\{i\})-q^{\mathrm{ret}}_\ell(S).
\]
Since $\Psi_i$ is measurable with respect to
$\left(H,A_\ell^{(S\cup\{i\})}\right)$, the tower property gives
\[
\E\left[
\Psi_iq^{\mathrm{ret}}_\ell(S\cup\{i\})
\right]
=
\E[\Psi_iG].
\]
The admissible-baseline argument gives
$\E[\Psi_iq^{\mathrm{ret}}_\ell(S)]=0$, hence
\[
\E[\Psi_i\Delta^{\mathrm{ret}}_{i\mid S}]
=
\E[\Psi_iG]
=
\nabla_{\theta_i}J_{\mathrm{sys}}.
\]
This proves \eqref{eq:retention-pg}. The same definitions also give
\[
\E\left[
G-q^{\mathrm{ret}}_\ell(S)
\mid
H,A_\ell^{(S\cup\{i\})}
\right]
=
q^{\mathrm{ret}}_\ell(S\cup\{i\})-q^{\mathrm{ret}}_\ell(S)
=
\Delta^{\mathrm{ret}}_{i\mid S}.
\]

For a random $S$, the argument holds conditionally on $S$ and then
averages, using the indexed clause of Lemma~\ref{lem:baseline}.
\end{proof}

\subsection{Proof of Proposition~\ref{prop:variance-monotone}}

\begin{proof}
\emph{Part (i).} The reward is bounded, so $G$ is square-integrable and all variances below
are finite. Fix nested backgrounds $S\subseteq T\subseteq N_\ell$ and write
$\mathcal F_S:=\sigma\left(H_\ell,A_\ell^{(S)}\right)$ and
$\mathcal F_T:=\sigma\left(H_\ell,A_\ell^{(T)}\right)$, so that
$\mathcal F_S\subseteq\mathcal F_T$. The law of total variance, applied to $G$ with the two nested conditionings,
gives
\[
\E\left[\Var\left(G\mid\mathcal F_S\right)\right]
=
\E\left[\Var\left(G\mid\mathcal F_T\right)\right]
+
\E\left[\Var\left(\E\left[G\mid\mathcal F_T\right]\mid\mathcal F_S\right)\right].
\]
The second term on the right is nonnegative, being an expected variance, and
part~(i) follows.

\emph{Part (ii).} Abbreviate $H:=H_\ell$, $\omega:=\omega_i$, $\bar\omega:=\bar\omega_i$, and fix $T\subseteq N_\ell\setminus\{i\}$, with $\mathcal F_T:=\sigma(H,A_\ell^{(T)})$ and $B^\star:=B^\star_T$.

\emph{Step 1: the weight's conditional mean.} The weight $\omega=\norm{\Psi_i}^2$ is a function of $(H,A_\ell^{(i)})$, and $A_\ell^{(i)}\perp A_\ell^{(T)}\mid H$ by \eqref{eq:layer-factorization}. Hence $\E[\omega\mid\mathcal F_T]=\E[\omega\mid H]=\bar\omega$. On $\{\bar\omega=0\}$ we have $\omega=0$ almost surely, so every term below vanishes there; we work on $\{\bar\omega>0\}$. Since $|\E[\omega G\mid\mathcal F_T]|\le\sup|G|\,\bar\omega$, the baseline $B^\star$ is bounded by $\sup|G|$.

\emph{Step 2: reduce the trace to a weighted second moment.} Let $B$ be bounded and $\mathcal F_T$-measurable. It satisfies \eqref{eq:baseline-cmi} by the same factorization, so $\E[\Psi_i(G-B)]=\nabla_{\theta_i}J_{\mathrm{sys}}$ by Lemma~\ref{lem:baseline}. Therefore
\[
\operatorname{tr}\operatorname{Cov}\left(\Psi_i(G-B)\right)
=
\E\left[\norm{\Psi_i}^2(G-B)^2\right]-\norm{\nabla_{\theta_i}J_{\mathrm{sys}}}^2
=
\E\left[\omega(G-B)^2\right]-\norm{\nabla_{\theta_i}J_{\mathrm{sys}}}^2 .
\]
The same holds with $B=B^\star$.

\emph{Step 3: expand around $B^\star$.} Write $G-B=(G-B^\star)+(B^\star-B)$. Then
\[
\E\left[\omega(G-B)^2\right]
=
\E\left[\omega(G-B^\star)^2\right]
+
2\,\E\left[\omega(G-B^\star)(B^\star-B)\right]
+
\E\left[\omega(B^\star-B)^2\right].
\]
The factor $B^\star-B$ is $\mathcal F_T$-measurable, so by the tower property and the definition of $B^\star$,
\[
\E\left[\omega(G-B^\star)(B^\star-B)\right]
=
\E\left[(B^\star-B)\left(\E[\omega G\mid\mathcal F_T]-B^\star\,\bar\omega\right)\right]
=0 .
\]
Similarly, $\E[\omega(B^\star-B)^2]=\E[\bar\omega(B^\star-B)^2]$ by Step~1. Subtracting $\norm{\nabla_{\theta_i}J_{\mathrm{sys}}}^2$ from both sides gives \eqref{eq:baseline-variance-identity}.

\emph{Step 4: consequences.} The last term of \eqref{eq:baseline-variance-identity} is nonnegative and vanishes only if $B=B^\star$ almost surely on $\{\bar\omega>0\}$, so $B^\star_T$ is the best baseline in its class. For $S\subseteq T$, the baseline $B^\star_S$ is $\mathcal F_T$-measurable, so applying \eqref{eq:baseline-variance-identity} in the class of $T$ with $B=B^\star_S$ gives
$\operatorname{tr}\operatorname{Cov}(\Psi_i(G-B^\star_S))=\operatorname{tr}\operatorname{Cov}(\Psi_i(G-B^\star_T))+\E[\bar\omega(B^\star_S-B^\star_T)^2]$; the best attainable variance does not increase with $T$. Finally, $q^{\mathrm{ret}}_\ell(T)=\E[G\mid\mathcal F_T]$, and
\[
B^\star_T-q^{\mathrm{ret}}_\ell(T)
=
\frac{\E[\omega G\mid\mathcal F_T]-\bar\omega\,\E[G\mid\mathcal F_T]}{\bar\omega}
=
\frac{\operatorname{Cov}(\omega,G\mid\mathcal F_T)}{\bar\omega},
\]
so the excess variance of the retention baseline is $\E[\operatorname{Cov}(\omega,G\mid\mathcal F_T)^2/\bar\omega]$, which is zero exactly when the conditional covariance vanishes almost surely on $\{\bar\omega>0\}$.

\emph{Part (iii).} Let $\mathcal F:=\sigma(H,A_\ell^{(N_\ell)})$ and $\mathcal F_{-i}:=\sigma(H,A_\ell^{(N_\ell\setminus\{i\})})$, and write $B^\star:=B^\star_{N_\ell\setminus\{i\}}$.

\emph{Step 1: gradient correctness.} The score $\Psi_i$ is $\mathcal F$-measurable, so $\E[\Psi_iX]=\E[\Psi_i\E[X\mid\mathcal F]]=\E[\Psi_i\bar g]-\E[\Psi_i\beta]$. The tower property gives $\E[\Psi_i\bar g]=\E[\Psi_iG]=\nabla_{\theta_i}J_{\mathrm{sys}}$. The term $\beta$ is $\mathcal F_{-i}$-measurable and square-integrable, since $\E[\beta^2]\le2\E[\bar g^2]+2\E[(\E[X\mid\mathcal F])^2]<\infty$; by the factorization it satisfies \eqref{eq:baseline-cmi}, so $\E[\Psi_i\beta]=0$ by Step~2 of the proof of Lemma~\ref{lem:baseline}. Hence $X$ is gradient-correct.

\emph{Step 2: split off the noise.} By Step~1, $\operatorname{tr}\operatorname{Cov}(\Psi_iX)=\E[\omega X^2]-\norm{\nabla_{\theta_i}J_{\mathrm{sys}}}^2$. Write $X=(\bar g-\beta)+\zeta$ with $\E[\zeta\mid\mathcal F]=0$. Both $\omega$ and $\bar g-\beta$ are $\mathcal F$-measurable, so the cross term $\E[\omega(\bar g-\beta)\zeta]=\E[\omega(\bar g-\beta)\E[\zeta\mid\mathcal F]]$ vanishes, and
\[
\E\left[\omega X^2\right]=\E\left[\omega(\bar g-\beta)^2\right]+\E\left[\omega\zeta^2\right].
\]
Both terms are finite, because conditional Jensen gives $\E[\omega(\bar g-\beta)^2]\le\E[\omega X^2]$.

\emph{Step 3: split off the mismatch.} Write $\bar g-\beta=(\bar g-B^\star)+(B^\star-\beta)$. The factor $B^\star-\beta$ is $\mathcal F_{-i}$-measurable, and since $\omega$ is $\mathcal F$-measurable,
$\E[\omega\bar g\mid\mathcal F_{-i}]=\E[\omega G\mid\mathcal F_{-i}]=B^\star\,\bar\omega$. The argument of part~(ii), Step~3, with $\bar g$ in place of $G$ and $\mathcal F_{-i}$ in place of $\mathcal F_T$, gives
\[
\E\left[\omega(\bar g-\beta)^2\right]
=
\E\left[\omega(\bar g-B^\star)^2\right]
+
\E\left[\bar\omega(B^\star-\beta)^2\right].
\]
Finally, $\bar g-B^\star$ is itself a signal of the form in part~(iii), with baseline $B^\star$ and no noise, so by Step~1 it is gradient-correct and $\E[\omega(\bar g-B^\star)^2]-\norm{\nabla_{\theta_i}J_{\mathrm{sys}}}^2=\operatorname{tr}\operatorname{Cov}(\Psi_i(\bar g-B^\star))$. Combining the three displays gives \eqref{eq:three-term}.
\end{proof}

\subsection{Proof of Lemma~\ref{lem:mask-target}}

\begin{proof}
Abbreviate $H:=H_\ell$, $A:=A_\ell^{(N_\ell)}$. The proof has two steps:
the squared loss is minimized by a conditional mean, and the mask identifies that conditional mean as the one we want.

\emph{Step 1: the minimizer is a conditional mean.} The loss
\eqref{eq:mask-loss} is an expected squared error of a predictor of $G$
from the inputs $\left(H,m(S),A^{(S)}\right)$. For any such prediction
problem, the population minimizer over all square-integrable functions
of the inputs is the conditional mean
$\E\left[G\mid H,S,A^{(S)}\right]$, and the risk decomposes as
\[
\mathcal L(\phi)
=
\E\left[\left(G-\E\left[G\mid H,S,A^{(S)}\right]\right)^2\right]
+
\E\left[\left(\E\left[G\mid H,S,A^{(S)}\right]-\hat q_\phi\right)^2\right],
\]
so if the critic class contains this conditional mean, any population
minimizer agrees with it almost surely on the support of the inputs;
this is where
positive mask probability on $S$ is used.

\emph{Step 2: independence removes $S$ from the conditional mean.} It
remains to show $\E[G\mid H,S,A^{(S)}]=\E[G\mid H,A^{(S)}]$. By hypothesis
$S$ is drawn with fresh randomness, conditionally independently of $(A,G)$
given $H$. Fix a background $s$ in the support. Conditionally on $H$, the
event $\{S=s\}$ is independent of $(A,G)$, so conditioning on it changes no
conditional law of $(A^{(s)},G)$. Hence
\[
\E\left[G\mid H,S=s,A^{(s)}\right]
=
\E\left[G\mid H,A^{(s)}\right]
\quad\text{almost surely.}
\]
In other words, the mask tells the critic which outputs it may see, and nothing else.
Combining the steps, any population minimizer equals
$\E[G\mid H,A^{(S)}]$ at every supported background, which is
\eqref{eq:critic-target}. Had the mask been allowed to depend on the
actions or the return, Step 2 would fail and the target would be $\E[G\mid H,S,A^{(S)}]$, a mask-contaminated quantity; the independent mask is not a technicality but the identifying assumption.
\end{proof}

\subsection{Proof of Proposition~\ref{prop:critic}}

\begin{proof}
The gradient identity \eqref{eq:critic-gradient} is an instance of
Lemma~\ref{lem:baseline} with $B_i=\hat q_\phi(S)$; we verify its two
hypotheses.

\emph{Square-integrability} holds by assumption on the critic.

\emph{Conditional mean independence.} Fix first a deterministic
$S\subseteq N_\ell\setminus\{i\}$. A fixed or cross-fitted critic is a
deterministic function of its inputs on the trajectory where it is applied
(cross-fitting ensures the function was built without that trajectory,
so no dependence on the current $A_\ell^{(i)}$ enters through $\phi$).
The inputs $(H_\ell,m(S),A_\ell^{(S)})$ exclude
$A_\ell^{(i)}$, and by the within-layer factorization
\eqref{eq:layer-factorization} the retained outputs $A_\ell^{(S)}$ are
conditionally independent of $A_\ell^{(i)}$ given $H_\ell$. Hence
\[
\E\left[\hat q_\phi(S)\mid H_\ell,A_\ell^{(i)}\right]
=
\E\left[\hat q_\phi(S)\mid H_\ell\right],
\]
which is \eqref{eq:baseline-cmi}. Lemma~\ref{lem:baseline} then gives
$\E\left[\Psi_i\left(G-\hat q_\phi(S)\right)\right]
=\nabla_{\theta_i}J_{\mathrm{sys}}$, which is \eqref{eq:critic-gradient}.
For $S$ drawn conditionally independently of $A_\ell^{(i)}$ given
$H_\ell$, the indexed clause of the lemma applies, and one draw serves
every excluded agent. Nowhere was accuracy of the critic used: the
argument is identical for a perfect critic and an untrained one.

For \eqref{eq:critic-bias}, condition on
$(H_\ell,A_\ell^{(S\cup\{i\})})$. The critic term
$\hat q_\phi(S)$ is a function of $(H_\ell,A_\ell^{(S)})$,
which the conditioning contains, so it passes through:
\[
\E\left[G-\hat q_\phi(S)\,\Big|\,H_\ell,A_\ell^{(S\cup\{i\})}\right]
=
\E\left[G\mid H_\ell,A_\ell^{(S\cup\{i\})}\right]-\hat q_\phi(S).
\]
Subtract the marginal contribution
\[
\Delta_{i\mid S}
=
\E\left[G\mid H_\ell,A_\ell^{(S\cup\{i\})}\right]
-
\E\left[G\mid H_\ell,A_\ell^{(S)}\right].
\]
The first conditional expectation cancels, leaving
\[
\E\left[G\mid H_\ell,A_\ell^{(S)}\right]-\hat q_\phi(S),
\]
which is \eqref{eq:critic-bias}: the conditional attribution error is the
critic's error at the baseline background.
\end{proof}

\subsection{Proof of Corollary~\ref{cor:routing-removal-gradient}}

\begin{proof}
Both claims are instances of Lemma~\ref{lem:baseline}; we verify its
hypothesis for each baseline.

\emph{The removed-router draw.} By construction,
$G^{-i}_{\mathrm{route}}$ is generated from the history, the remaining
realized candidates, and fresh router randomness: it is a function of
$\left(H,A^{(-i)},\xi'\right)$ with $\xi'$ drawn independently of
everything else. Under independent candidate generation (the
within-layer factorization \eqref{eq:layer-factorization} applied to the
routing layer), $A^{(-i)}$ is conditionally independent of $A^{(i)}$
given $H$, and $\xi'$ carries no information about $A^{(i)}$ at all. Hence
\[
\E\left[G^{-i}_{\mathrm{route}}\mid H,A^{(i)}\right]
=
\E\left[G^{-i}_{\mathrm{route}}\mid H\right]:
\]
what the reduced slate would deliver does not depend on the output of the
candidate no longer in it. The draw is bounded, hence square-integrable,
and Lemma~\ref{lem:baseline} gives
\eqref{eq:routing-removal-gradient}.

\emph{A predictor.} A fixed or cross-fitted predictor built from
$\left(H,A^{(-i)}\right)$ is a deterministic function of those inputs on
the trajectory where it is applied, cross-fitting ensuring that its
parameters were trained without that trajectory. The same conditional
independence gives the same conditional-mean identity, and the lemma
applies again. Accuracy of the predictor is used nowhere.
\end{proof}

\subsection{Proof of Corollary~\ref{cor:routing-dr}}

\begin{proof}
Condition on $\left(H,A^{(1:K)}\right)$. The router weights
$p_j$ and $p_j^{-i}$ are then fixed. By \eqref{eq:candidate-outcome},
\[
\E\left[
\1_{\{I=j\}}
\left(G-\widehat\mu_j\right)
\mid \left(H,A^{(1:K)}\right)
\right]
=
p_j\left(g_j-\widehat\mu_j\right).
\]
Therefore,
\[
\E\left[
\widehat g_j
\mid \left(H,A^{(1:K)}\right)
\right]
=
\widehat\mu_j
+
\frac{1}{p_j}
p_j\left(g_j-\widehat\mu_j\right)
=
g_j.
\]
Overlap makes the ratio well-defined. Taking the difference,
\[
\E\left[
\widehat\Delta_i^{\mathrm{rem}}
\mid \left(H,A^{(1:K)}\right)
\right]
=
\sum_jp_jg_j
-
\sum_{j\ne i}p_j^{-i}g_j
=
\Delta_i^{\mathrm{rem}}.
\]
The outcome model may therefore be misspecified without changing the
conditional mean of the corrected contrast.
\end{proof}

\subsection{Proof of Corollary~\ref{cor:dr-gradient}}

\begin{proof}
Let $\mathcal F:=\sigma(H,A^{(1:K)})$. The router and the outcome model are fixed during the update (or cross-fitted, so that they were built without the current trajectory). Hence, given $\mathcal F$, the propensities $p_j$ and $p_j^{-i}$ and the predictions $\widehat\mu_j$ are fixed numbers, and the only remaining randomness in $\widehat\Delta_i^{\mathrm{rem}}$ is the selection $I$ and the rewards.

\emph{Step 1: conditional mean.} The proof of Corollary~\ref{cor:routing-dr} gives $\E[\widehat g_j\mid\mathcal F]=g_j$ for every $j$, using the exact factual propensity $p_j$ in the denominator. Hence
\[
\E\left[\widehat\Delta_i^{\mathrm{rem}}\mid\mathcal F\right]
=
\sum_{j=1}^Kp_jg_j-\sum_{j\ne i}p_j^{-i}g_j .
\]
By \eqref{eq:candidate-outcome}, $\E[G\mid\mathcal F]=\sum_j\Prob(I=j\mid\mathcal F)\,\E[R^{(j)}\mid\mathcal F,I=j]=\sum_jp_jg_j=\bar g$. So $\E[\widehat\Delta_i^{\mathrm{rem}}\mid\mathcal F]=\bar g-\beta$ with $\beta:=\sum_{j\ne i}p_j^{-i}g_j$.

\emph{Step 2: the baseline does not involve agent $i$.} The removed propensities $p_j^{-i}$ are computed from the scores $s(H,A^{(j)})$ with $j\ne i$ only (Appendix~\ref{app:dr-details}), and each $g_j$ is a function of $(H,A^{(j)})$. Hence $\beta$ is a function of $(H,A^{(-i)})$.

\emph{Step 3: integrability.} With $p_j\ge p_{\min}>0$ and bounded rewards and predictions, $|\widehat g_j|\le C/p_{\min}$ for a constant $C$, so $\widehat\Delta_i^{\mathrm{rem}}$ is bounded, and both $\E[(\widehat\Delta_i^{\mathrm{rem}})^2]$ and $\E[\omega_i(\widehat\Delta_i^{\mathrm{rem}})^2]$ are finite.

Steps~1--3 are the hypotheses of Proposition~\ref{prop:variance-monotone}(iii), which gives gradient correctness and the decomposition with this $\beta$.

If the denominator $p_j$ were replaced by a floored value $\max\{p_j,c\}$, Step~1 would give $\E[\widehat g_j\mid\mathcal F]=\widehat\mu_j+\frac{p_j}{\max\{p_j,c\}}(g_j-\widehat\mu_j)$, which differs from $g_j$ whenever $p_j<c$ and $\widehat\mu_j\ne g_j$. The conditional mean would then depend on the outcome model's errors and, through every $p_j$, on $A^{(i)}$, and the gradient identity could fail.
\end{proof}

\subsection{Proof of Theorems~\ref{thm:selective} and~\ref{thm:selective-lb}} \label{subseq:allocations}

We begin by stating Theorem 3.

\begin{theorem}[Matching lower bound]
\label{thm:selective-lb}
In the setting of Theorem~\ref{thm:selective}, write $\mathcal A$ for
the marginal contributions with $p^\star_{i\mid S}=1$, $\mathcal U$ for
the rest, and
$C_{\mathcal U}=C_{\max}-\sum_{(i,S)\in\mathcal A}c_{i\mid S}$. Every
allocation of the same expected cost satisfies
\begin{equation}
\label{eq:query-lower-bound}
\mathcal R(p)
\ge
\sum_{(i,S)\in\mathcal A}
a_i^2\sigma_{i\mid S}^2
+
\frac{
\left(
\sum_{(i,S)\in\mathcal U}
a_i
\sqrt{\sigma_{i\mid S}^2+\eta_{i\mid S}^2}
\sqrt{c_{i\mid S}}
\right)^2
}{
C_{\mathcal U}
}
-
\sum_{(i,S)\in\mathcal U}
a_i^2\eta_{i\mid S}^2,
\end{equation}
with equality at $p^\star$.
\end{theorem}

The following proof addresses both theorems.

\begin{proof}
Condition throughout on the data and learned estimates, so
$\Delta_{i\mid S}$, $\eta_{i\mid S}$, $a_i$, and $c_{i\mid S}$ are fixed.

First, since $Z_{i\mid S}$ is independent of $Y_{i\mid S}$ and
$\E[Z_{i\mid S}]=p_{i\mid S}$,
\[
\E\left[\widetilde\Delta_{i\mid S}\right]
=
\widehat\Delta_{i\mid S}
+
\frac{\E[Z_{i\mid S}]}{p_{i\mid S}}
\E\left[
Y_{i\mid S}-\widehat\Delta_{i\mid S}
\right]
=
\Delta_{i\mid S}.
\]
Thus the correction is unbiased.

Next,
\[
\widetilde\Delta_{i\mid S}-\Delta_{i\mid S}
=
\eta_{i\mid S}
+
\frac{Z_{i\mid S}}{p_{i\mid S}}
\left(
Y_{i\mid S}-\widehat\Delta_{i\mid S}
\right).
\]
Since
\[
\E\left[
Y_{i\mid S}-\widehat\Delta_{i\mid S}
\right]
=
-\eta_{i\mid S},
\qquad
\E\left[
\left(
Y_{i\mid S}-\widehat\Delta_{i\mid S}
\right)^2
\right]
=
\sigma_{i\mid S}^2+\eta_{i\mid S}^2,
\]
and $Z_{i\mid S}^2=Z_{i\mid S}$, we obtain
\begin{align*}
\E\left[
\left(
\widetilde\Delta_{i\mid S}-\Delta_{i\mid S}
\right)^2
\right]
={}&
\eta_{i\mid S}^2
+
2\eta_{i\mid S}
\E\left[
\frac{Z_{i\mid S}}{p_{i\mid S}}
\left(
Y_{i\mid S}-\widehat\Delta_{i\mid S}
\right)
\right] \\
&+
\E\left[
\frac{Z_{i\mid S}}{p_{i\mid S}^2}
\left(
Y_{i\mid S}-\widehat\Delta_{i\mid S}
\right)^2
\right] \\
={}&
\eta_{i\mid S}^2
-
2\eta_{i\mid S}^2
+
\frac{\sigma_{i\mid S}^2+\eta_{i\mid S}^2}{p_{i\mid S}} \\
={}&
\frac{\sigma_{i\mid S}^2+\eta_{i\mid S}^2}{p_{i\mid S}}
-
\eta_{i\mid S}^2.
\end{align*}

It remains to optimize. Write $\mathcal M_+$ for the set of marginal contributions with $a_i^2(\sigma_{i\mid S}^2+\eta_{i\mid S}^2)>0$; inactive pairs contribute zero risk for every $p$ and may be skipped. If the budget
permits evaluating every active marginal contribution, then
$p_{i\mid S}^\star=1$ for all $(i,S)\in\mathcal M_+$. Otherwise the budget
binds, and terms independent of $p$ may be dropped, leaving
\[
\min_p
\sum_{(i,S)\in\mathcal M_+}
\frac{
a_i^2\left(\sigma_{i\mid S}^2+\eta_{i\mid S}^2\right)
}{
p_{i\mid S}
}
\]
subject to
\[
\sum_{(i,S)\in\mathcal M_+}
c_{i\mid S}p_{i\mid S}
=C_{\max},
\qquad
0<p_{i\mid S}\le1.
\]

The objective is strictly convex on $\mathcal M_+$. Let $\nu>0$ be the
multiplier for the budget constraint and $\mu_{i\mid S}\ge0$ the multiplier
for $p_{i\mid S}\le1$. Stationarity gives
\[
-
\frac{
a_i^2\left(\sigma_{i\mid S}^2+\eta_{i\mid S}^2\right)
}{
p_{i\mid S}^2
}
+
\nu c_{i\mid S}
+
\mu_{i\mid S}
=
0.
\]
When $p_{i\mid S}<1$, complementary slackness gives
$\mu_{i\mid S}=0$, so
\[
p_{i\mid S}
=
\frac{1}{\sqrt{\nu}}\,
\frac{
a_i\sqrt{\sigma_{i\mid S}^2+\eta_{i\mid S}^2}
}{
\sqrt{c_{i\mid S}}
}.
\]
When this value exceeds one, the constraint binds. Setting
$\lambda:=1/\sqrt{\nu}$ gives \eqref{eq:optimal-query}, with $\lambda$
chosen so that the budget binds. Strict convexity gives uniqueness.

It remains to evaluate the risk at the optimum. Let $\mathcal A$ contain
the marginal contributions with $p_{i\mid S}^\star=1$, and let
$\mathcal U=\mathcal M_+\setminus\mathcal A$. On $\mathcal A$,
\[
a_i^2
\left(
\frac{\sigma_{i\mid S}^2+\eta_{i\mid S}^2}{p_{i\mid S}^\star}
-
\eta_{i\mid S}^2
\right)
=
a_i^2\sigma_{i\mid S}^2.
\]
On $\mathcal U$,
\[
p_{i\mid S}^\star
=
\lambda\,
\frac{
a_i\sqrt{\sigma_{i\mid S}^2+\eta_{i\mid S}^2}
}{
\sqrt{c_{i\mid S}}
}.
\]
Since the remaining budget is
\[
C_{\mathcal U}
=
C_{\max}-\sum_{(i,S)\in\mathcal A}c_{i\mid S},
\]
the binding budget constraint gives
\[
\lambda
=
\frac{
C_{\mathcal U}
}{
\sum_{(i,S)\in\mathcal U}
a_i
\sqrt{\sigma_{i\mid S}^2+\eta_{i\mid S}^2}
\sqrt{c_{i\mid S}}
}.
\]
Substituting into the risk yields
\[
\mathcal R(p^\star)
=
\sum_{(i,S)\in\mathcal A}
a_i^2\sigma_{i\mid S}^2
+
\frac{
\left(
\sum_{(i,S)\in\mathcal U}
a_i
\sqrt{\sigma_{i\mid S}^2+\eta_{i\mid S}^2}
\sqrt{c_{i\mid S}}
\right)^2
}{
C_{\mathcal U}
}
-
\sum_{(i,S)\in\mathcal U}
a_i^2\eta_{i\mid S}^2.
\]
Because $p^\star$ is the unique global minimizer,
$\mathcal R(p)\ge\mathcal R(p^\star)$ for every feasible allocation. This is
\eqref{eq:query-lower-bound}, with equality at $p^\star$.
\end{proof}

\section{Refinements of Theorem~\ref{thm:wta}}
\label{app:routing-refinements}

The results below refine Theorem~\ref{thm:wta}; each is referenced at
the end of Section~\ref{sec:definitions}.

\begin{proposition}[Selection makes private utility risk-sensitive]
\label{prop:risk-sensitive}
Under softmax routing \eqref{eq:softmax-router}, private utility is not
determined by the expected candidate reward. Thus, there exist
candidate-reward distributions with equal conditional means that induce
different private utilities.
\end{proposition}

\begin{proof}
Condition on a fixed history. Take $K=2$, $\epsilon=0$, $\tau=1$, and
$s(H,A^{(j)})=R^{(j)}$, and let $R^{(2)}=c$ be deterministic.
Compare $R^{(1)}_{\text{safe}}=\tfrac12$ almost surely with
$R^{(1)}_{\text{risky}}\sim\operatorname{Bernoulli}(\tfrac12)$. Agent $1$ is selected
with probability
\[
p_1(r)
=
\frac{e^r}{e^r+e^c}
=
\operatorname{sigmoid}(r-c),
\]
where $\operatorname{sigmoid}(u)=(1+e^{-u})^{-1}$. Under the constant policy,
\[
U_1^{\mathrm{safe}}
=
p_1\!\left(\frac12\right)\frac12
=
\frac12\,\operatorname{sigmoid}\!\left(\frac12-c\right).
\]
Under the Bernoulli policy, the zero-reward branch contributes nothing, so
\[
U_1^{\mathrm{risky}}
=
\frac12\,p_1(1)
=
\frac12\,\operatorname{sigmoid}(1-c).
\]
The policies have the same mean reward, but
$U_1^{\mathrm{risky}}>U_1^{\mathrm{safe}}$ because the sigmoid function is strictly
increasing. Private utility is therefore not determined by expected
candidate reward.
\end{proof}

\begin{remark}[Converse under score completeness]
\label{rem:completeness}
The reverse implication in Theorem~\ref{thm:wta}, from $b_i=0$ to
conditional mean independence, holds when the score is complete: when
the closed linear span in $L^2$ of the coordinates of $\Psi_i$ contains
every square-integrable function of $(H,A^{(i)})$ with zero conditional
mean given $H$. Completeness holds when each history--action probability has its own free parameter, as in a finite tabular softmax, but can fail when parameters are shared across histories, as in an autoregressive model.
\end{remark}

\begin{proposition}[Directional inflation under monotone routing]
\label{prop:wta-sign}
Assume the router \eqref{eq:softmax-router} with score
$s(H,A^{(j)})=\kappa(R^{(j)})$ for an increasing $\kappa$, and
conditionally independent nonnegative candidate rewards given $H$. Fix
a direction $w_i$ and let
\[
\zeta_{w_i}(h,r)
:=
\E\left[
w_i^\top\Psi_i
\mid
H=h,\ R^{(i)}=r
\right].
\]
If $r \mapsto \zeta_{w_i}(h,r)$ is nondecreasing for almost every $h$,
then
\begin{equation}
\label{eq:wta-directional-sign}
w_i^\top b_i\ge0,
\qquad
\E\left[
\1_{\{I=i\}}\,w_i^\top\Psi_iG
\right]
\ge
w_i^\top\nabla_{\theta_i}J_{\mathrm{sys}}.
\end{equation}
The inequality is strict when no factor degenerates: the router depends on scores ($\epsilon<1$), $R^{(i)}$ varies on a set of histories of positive probability, $\kappa$ and $\zeta_{w_i}(h,\cdot)$ are strictly increasing on its conditional support, and on those histories some competitor has positive reward with positive conditional probability.
\end{proposition}

\begin{proof}
Write $Z_i:=w_i^\top\Psi_i$. By \eqref{eq:displacement-bias-identity} and
$\E[\Psi_i\,\bar d_i(H)]=0$, as in the proof of Theorem~\ref{thm:wta},
\[
w_i^\top b_i
=
-\E[Z_iD_i].
\]
For fixed $h$, define
\[
m_h(r)
:=
\E\left[
D_i
\mid
H=h,\ R^{(i)}=r
\right].
\]
Conditional independence of candidate rewards gives
\[
m_h(r)
=
\E\left[
(1-\epsilon)
\frac{
\sum_{j\ne i}
e^{\kappa(R^{(j)})/\tau}R^{(j)}
}{
e^{\kappa(r)/\tau}
+
\sum_{k\ne i}e^{\kappa(R^{(k)})/\tau}
}
+
\frac{\epsilon}{K}
\sum_{j\ne i}R^{(j)}
\;\middle|\;
H=h
\right].
\]
Since rewards are nonnegative and $\kappa$ is increasing,
$m_h(\cdot)$ is nonincreasing.

By the layerwise factorization, conditioning further on $R^{(i)}$ preserves the independence
of $A^{(-i)}$ from $A^{(i)}$. Under the softmax router \eqref{eq:softmax-router} with score $\kappa(R^{(j)})$, $D_i$ depends on agent $i$ only through
$R^{(i)}$, so
\[
Z_i
\perp
D_i
\mid
\left(H,R^{(i)}\right).
\]
Therefore,
\[
\E\left[
Z_iD_i
\mid
H,R^{(i)}
\right]
=
\zeta_{w_i}\!\left(H,R^{(i)}\right)
m_H\!\left(R^{(i)}\right).
\]
The score identity also gives
\[
\E\left[
\zeta_{w_i}\left(H,R^{(i)}\right)
\mid H
\right]
=
\E[Z_i\mid H]
=
0.
\]

Let $\widetilde R^{(i)}$ be an independent copy of $R^{(i)}$
conditional on $H$. The covariance identity yields
\begin{align*}
2\,\E\left[
\zeta_{w_i}\left(H,R^{(i)}\right)
m_H\left(R^{(i)}\right)
\mid H
\right]
={}&
\E\Big[
\big(
\zeta_{w_i}(H,R^{(i)})
-
\zeta_{w_i}(H,\widetilde R^{(i)})
\big)\\
&\qquad\cdot
\big(
m_H(R^{(i)})
-
m_H(\widetilde R^{(i)})
\big)
\mid H
\Big]
\le 0,
\end{align*}
because the first factor is nondecreasing in the reward and the second is
nonincreasing. Hence
\[
\E[Z_iD_i]\le0,
\qquad
w_i^\top b_i\ge0.
\]

Under the stated strictness conditions, $m_h$ is strictly decreasing on a
nondegenerate conditional reward support for a set of histories of positive
probability. The covariance above is then strictly negative, so
$w_i^\top b_i>0$. Finally,
\[
\E\left[
\1_{\{I=i\}}Z_iG
\right]
=
w_i^\top\nabla_{\theta_i}J_{\mathrm{sys}}
+
w_i^\top b_i,
\]
which proves the second inequality and its strict form.
\end{proof}

\section{Doubly Robust Routing Details}
\label{app:dr-details}

The routing observation is $O=\left(H,A^{(1:K)},I,G,p_{1:K}\right)$, with
$G=R^{(I)}$ the selected candidate's reward. Under the $\epsilon$-mixed
softmax with weights $w_j=e^{s(H,A^{(j)})/\tau}$, the
factual and removed probabilities are
\[
p_j
=
(1-\epsilon)\frac{w_j}{\sum_{k=1}^{K}w_k}+\frac{\epsilon}{K},
\qquad
p_j^{-i}
=
(1-\epsilon)\frac{w_j}{\sum_{k\ne i}w_k}+\frac{\epsilon}{K-1}
\quad (j\ne i),
\]
so removal renormalizes both the softmax mass and the exploration mass
over the surviving agents. Removal changes only the probabilities being evaluated, $p_j$ versus
$p_j^{-i}$ in the two sums; the inverse-propensity correction always
divides by the factual $p_j$, because that is the router that
generated the data.

With
$w_{ij}=p_j-\1_{\{j\ne i\}}\,p_j^{-i}$,
the estimator admits the centered representation
\begin{align*}
\widehat\Delta_i^{\mathrm{rem}}-\Delta_i^{\mathrm{rem}}
={}&
\sum_{j=1}^{K}w_{ij}
\frac{\1_{\{I=j\}}}{p_j}
\left(G-g_j\right)\\
&+
\sum_{j=1}^{K}w_{ij}
\left(
1-\frac{\1_{\{I=j\}}}{p_j}
\right)
\left(\widehat\mu_j-g_j\right).
\end{align*}
Conditional on $(H,A^{(1:K)})$, both sums have mean zero: the first is
selected-outcome noise, and the second is the augmentation term correcting
outcome-model error. This is the doubly robust form specialized to known
factual propensities. Overlap governs both identification and variance,
since the correction scales as $1/p_j$.

For actor training without queries, the direct baseline
$G-\sum_{j\ne i}p_j^{-i}\widehat\mu_j$ is gradient-correct whenever the
cross-fitted predictor excludes $A^{(i)}$
(Corollary~\ref{cor:routing-removal-gradient}), but it is not in general
an unbiased estimate of the removal attribution; the augmented
inverse-propensity contrast restores attribution under the known factual
propensities for any fixed or cross-fitted outcome model and is the
quantity selectively corrected in Section~\ref{sec:selective}.

\input{sections/slates}

\input{sections/implementation}

\input{sections/e8}

\end{document}

%% file: tables/numbers2.tex
\newcommand{\agentPatternSentence}{On agent~1 DR gains more than H and SR; on agents 2 and 3 it gains less.}

\newcommand{\cellDiagSvamp}{3.3\%, 0.27\% and 0.02\%; 4.1\%, 0.23\% and 0.01\%}

\newcommand{\diffHminusDR}{$-$0.001}
\newcommand{\diffHminusDRCI}{[$-$0.063, $+$0.060]}

\newcommand{\exactRatioAgentOne}{3.7\%}

\newcommand{\fbNSeeds}{5}

\newcommand{\firstBatchSentence}{Seeds 11--15 were run and analysed first (H $-$ DR $+$0.029 [$+$0.010, $+$0.049], $p=0.007$; SR $-$ DR $+$0.022 [$+$0.008, $+$0.036], $p=0.006$; H $-$ SR $+$0.007 [$-$0.017, $+$0.032]); seeds 16 and 18 were added afterwards as planned, with no change to code or protocol, and Section~\ref{sec:results} reports all seeds.}
\newcommand{\gainDR}{$+$0.034}
\newcommand{\gainH}{$+$0.033}
\newcommand{\gainS}{$+$0.046}

\newcommand{\incompleteRunsSentence}{}
\newcommand{\nCells}{18}
\newcommand{\nCellsNewSeeds}{12}

\newcommand{\nSeedsPaired}{7}

\newcommand{\pHminusDR}{0.517}

\newcommand{\rDREarly}{1.90}
\newcommand{\rDREarlyCI}{[1.75, 2.07]}
\newcommand{\rDRLate}{15.4}
\newcommand{\rDRLateCI}{[13.7, 17.4]}

\newcommand{\rDRSvamp}{1.72}

\newcommand{\rDRoAll}{0.37}

\newcommand{\rDRoEarly}{0.26}
\newcommand{\rDRoEarlyCI}{[0.25, 0.27]}
\newcommand{\rDRoLate}{0.48}
\newcommand{\rDRoLateCI}{[0.45, 0.51]}

\newcommand{\rDRoSvamp}{0.30}

\newcommand{\rHAgentTwoLate}{0.93}

\newcommand{\rHAll}{0.73}
\newcommand{\rHAllCI}{[0.70, 0.76]}

\newcommand{\rHEarly}{0.66}
\newcommand{\rHEarlyCI}{[0.63, 0.70]}
\newcommand{\rHFormB}{1.92}

\newcommand{\rHLate}{0.80}
\newcommand{\rHLateCI}{[0.76, 0.83]}
\newcommand{\rHNewSeeds}{0.72}
\newcommand{\rHNewSeedsCI}{[0.69, 0.75]}

\newcommand{\rHSvamp}{0.70}
\newcommand{\rHSvampCI}{[0.59, 0.84]}
\newcommand{\rHlAll}{1.10}

\newcommand{\rWEarly}{0.59}
\newcommand{\rWEarlyCI}{[0.56, 0.63]}
\newcommand{\rWLate}{0.60}
\newcommand{\rWLateCI}{[0.56, 0.64]}

\newcommand{\smokeRatioSvamp}{1.5\%, 0.06\% and 0.01\%}

%% file: tables/numbers_e4.tex
\newcommand{\eFourAgreeAdam}{36}

\newcommand{\eFourAgreeRest}{12}
\newcommand{\eFourAgreeSgd}{37}
\newcommand{\eFourDiffAdam}{$-$0.141}
\newcommand{\eFourDiffAdamMax}{$-$0.307}
\newcommand{\eFourDiffAdamMin}{$-$0.035}
\newcommand{\eFourDiffSgd}{$-$0.133}
\newcommand{\eFourLowMax}{0.80}
\newcommand{\eFourLrAdam}{0.03}
\newcommand{\eFourLrSgd}{1}
\newcommand{\eFourNegAdam}{8}
\newcommand{\eFourNegSgd}{8}
\newcommand{\eFourPairs}{40}
\newcommand{\eFourPairsLow}{24}
\newcommand{\eFourPairsRest}{16}
\newcommand{\eFourRuns}{100}

\newcommand{\eFourSensLr}{0.003}

\newcommand{\eFourSensS}{0.788}
\newcommand{\eFourSensW}{0.785}
\newcommand{\eFourSteps}{4{,}000}

%% file: tables/numbers_compute.tex
\newcommand{\gpuHoursBilled}{72}
\newcommand{\gpuHoursCells}{16.9}
\newcommand{\gpuHoursGen}{6.2}

\newcommand{\gpuHoursTraining}{43.3}
\newcommand{\hoursPerTrainingRun}{2.1}
\newcommand{\minPerCell}{49}
\newcommand{\nCellJobs}{18}
\newcommand{\nTrainingRuns}{21}

%% file: tables/numbers_e5.tex
\newcommand{\eFiveCellMax}{0.65}
\newcommand{\eFiveCellMin}{0.13}
\newcommand{\eFiveFeasible}{0.24}
\newcommand{\eFiveFeasibleCI}{[0.18, 0.32]}
\newcommand{\eFiveFeasibleEarly}{0.47}
\newcommand{\eFiveFeasibleLate}{0.21}
\newcommand{\eFiveFeasiblePL}{0.65}
\newcommand{\eFiveLateShare}{86\%}
\newcommand{\eFiveLeverage}{0.90}
\newcommand{\eFiveLeverageCI}{[0.78, 1.05]}
\newcommand{\eFiveLeveragePL}{0.67}
\newcommand{\eFiveNCells}{18}
\newcommand{\eFiveNSlates}{18{,}432}
\newcommand{\eFiveNoFeature}{1.00}
\newcommand{\eFiveOracle}{0.05}
\newcommand{\eFiveOracleCI}{[0.04, 0.07]}
\newcommand{\eFiveOraclePL}{0.38}
\newcommand{\eFiveRecovered}{0.80}
\newcommand{\eFiveRecoveredCI}{[0.72, 0.85]}
\newcommand{\eFiveSummary}{Of the three predictions fixed before the computation, three held}

%% file: tables/numbers_e8.tex
\newcommand{\eEightBestD}{$2\cdot10^{-5}$}
\newcommand{\eEightBestS}{$10^{-5}$}
\newcommand{\eEightBestW}{$10^{-5}$}
\newcommand{\eEightBlocks}{30}

\newcommand{\eEightCollCD}{0}
\newcommand{\eEightCollCS}{4}

\newcommand{\eEightDelta}{$+$0.055}
\newcommand{\eEightDeltaCI}{[$+$0.001, $+$0.200]}

\newcommand{\eEightGroupMax}{1.08}
\newcommand{\eEightHOneMeanD}{$+$0.088}
\newcommand{\eEightHOneP}{7.5\cdot10^{-5}}
\newcommand{\eEightHOnePos}{24}
\newcommand{\eEightHTwoMeanD}{$+$0.023}
\newcommand{\eEightHTwoP}{0.071}
\newcommand{\eEightHTwoPos}{21}
\newcommand{\eEightRealDThree}{0.08}
\newcommand{\eEightRealDTwo}{0.12}
\newcommand{\eEightRealWThree}{0.27}
\newcommand{\eEightRealWTwo}{0.32}

\newcommand{\eEightSeatsC}{10}

\newcommand{\eEightSensP}{1.3\cdot10^{-6}}

\newcommand{\eEightStepDTwoKnine}{0.09}
\newcommand{\eEightStepDTwoKsix}{0.15}
\newcommand{\eEightStepDTwoKthree}{0.41}

\newcommand{\eEightStepHTwoKnine}{0.54}

\newcommand{\eEightTeamDSixEarly}{0.18}
\newcommand{\eEightTeamDSixLate}{0.93}
\newcommand{\eEightTeamDThreeEarly}{0.66}
\newcommand{\eEightTeamDThreeLate}{4.86}
\newcommand{\eEightTeamHSixEarly}{0.51}
\newcommand{\eEightTeamHSixLate}{0.60}
\newcommand{\eEightTeamHThreeEarly}{0.64}
\newcommand{\eEightTeamHThreeLate}{0.75}
\newcommand{\eEightTeams}{21}
\newcommand{\eEightTolD}{$5\cdot10^{-6}$, $10^{-5}$ and $2\cdot10^{-5}$}
\newcommand{\eEightTolDN}{three}
\newcommand{\eEightTolS}{$10^{-5}$}
\newcommand{\eEightTolW}{$10^{-5}$}
\newcommand{\eEightTunedD}{$+$0.083}
\newcommand{\eEightTunedS}{$+$0.028}
\newcommand{\eEightTunedW}{$+$0.065}
\newcommand{\eEightUnifAcc}{0.63}
\newcommand{\eEightVoteAcc}{0.75}
\newcommand{\eEightVoteDThree}{0.16}
\newcommand{\eEightVoteDThreeCI}{[0.14, 0.17]}
\newcommand{\eEightVoteDTwo}{0.20}
\newcommand{\eEightVoteDTwoCI}{[0.17, 0.24]}

%% file: sections/experiment.tex
\section{Experiments}
\label{sec:experiment}

\begin{localparagraphgap}{-0.5ex}

Two findings organize this section. First (Section~\ref{sec:noise}), an estimate's own noise can swamp what it removes: DR divides by the router's picking probabilities, which become small once a learned router is confident, and the hybrid signal H, which drops the division, is quieter than shared reward. Second (Section~\ref{sec:train3}), a quieter signal pays off in training when shared reward is noisy, as with many agents: with three agents no signal separates; with nine and a uniform router, which bounds DR's division, the removal estimate learns faster than shared reward and tolerates twice the learning rate. Every prediction below was committed before its data; all are reported, including failures (Appendices~\ref{app:gradnoise-extra} and~\ref{app:e8}).

\paragraph{Setup.}
Each GSM8K problem \citep{cobbe2021training} is one prompt. We use GSM8K because grading is exact, so every answer can be graded offline as a diagnostic (training sees only the picked one), and untrained agents solve $51$--$61\%$ of problems, leaving room to learn. An agent is \texttt{Llama-3.2-3B-Instruct} with its own LoRA adapter \citep{hu2022lora}, role prompt and temperature (Appendix~\ref{app:implementation}); Algorithm~\ref{alg:step} gives one training step, with $m=4$.

\begin{algorithm}[H]
\caption{One training step in a routed layer. The signal $X$ is a row of Table~\ref{tab:signals}.}
\label{alg:step}
\small
\begin{algorithmic}[1]
\REQUIRE prompt $H$; agents $\pi_{\theta_i}$, one LoRA adapter each; router $p$ and outcome model $\widehat\mu$, fixed in the step
\STATE Each agent $i$ writes $m$ answers $A^{(i)}_s\sim\pi_{\theta_i}(\cdot\mid H)$, with scores $\Psi_{i,s}$; slate $s$ holds answer $s$ of every agent.
\STATE On each slate $s$ the router draws $I_s\sim p_s$ by \eqref{eq:softmax-router} and deploys answer $A^{(I_s)}_s$.
\STATE Only deployed answers are graded: $G_s=1$ if the final number is correct, else $0$.
\STATE Each agent $i$ gets a signal $X_{i,s}$ per slate from $G_s,I_s,p_s$ (and, for DR and H, from $\widehat\mu$ and $p^{-i}_s$).
\STATE Centre each signal by the mean of its other $m-1$: $u_i=\frac1m\sum_s\Psi_{i,s}\bigl(X_{i,s}-\frac{1}{m-1}\sum_{t\ne s}X_{i,t}\bigr)$ \eqref{eq:centred-update}.
\STATE Each agent takes one AdamW ascent step along $u_i$ on its own adapter; no standardization, no clipping.
\end{algorithmic}
\end{algorithm}

\end{localparagraphgap}

\subsection{Gradient noise of each signal}
\label{sec:noise}
\label{sec:results}

\begin{localparagraphgap}{-0.5ex}

We report $\kappa(X)$: the mean squared update of an agent under signal $X$ divided by that under shared reward, on identical rollouts and exact over the router's random picks; $\kappa<1$ means quieter than shared reward. After the hybrid signal below was designed on separate data and the predictions were committed, we measured \nCells{} new cells: three new policies, early (update $25$) and after training (update $150$), on three prompt lists (Table~\ref{tab:hybrid-main}).

\paragraph{The removal estimate is noisy where the decomposition says.}
The DR removal estimate of Section~\ref{sec:routing-est}, with the router's learned outcome model, has $\kappa=\rDREarly$ early and $\rDRLate$ after training, while the exact contribution it estimates has $\kappa=\rDRoAll$. Both have the same baseline, so by Proposition~\ref{prop:variance-monotone}(iii) the excess is in the noise term: DR's baseline divides by the router's picking probabilities, which become small once the router is confident.

\paragraph{Removing that division gives a signal quieter than shared reward.}
The hybrid signal H \eqref{eq:hybrid} stays gradient-correct without that division (Corollary~\ref{cor:hybrid}), using a training-free outcome model: the share of the prompt's other answers with the same final number. $\kappa(\mathrm H)=\rHAll$ \rHAllCI, below $1$ for every policy and checkpoint, and $\rHSvamp$ \rHSvampCI{} on the SVAMP dataset \citep{patel2021svamp}.

\paragraph{Limits.}
The gain is smallest for the agent the router picks most ($\rHAgentTwoLate$ after training) and is lost with the router's learned outcome model ($\rHlAll$). Under GRPO's standardization the order reverses ($\rHFormB$): the result holds for the update the theory covers. Nine of eleven predictions held. Two failed: agent~1's score kept a small nonzero mean (C8), and winner-take-all's bias was not resolved statistically (C9; Appendix Table~\ref{tab:predictions2}).

\end{localparagraphgap}

\begin{table}[H]
\centering
\scriptsize
\setlength{\tabcolsep}{3pt}
\begin{minipage}[t]{0.40\textwidth}
\centering
\vspace{0pt}
\begin{tabular}{@{}lcc@{}}
\toprule
Signal & update 25 & update 150 \\
\midrule
W (biased) & \rWEarly\,{\tiny\rWEarlyCI} & \rWLate\,{\tiny\rWLateCI} \\
\textbf{H} & \textbf{\rHEarly}\,{\tiny\rHEarlyCI} & \textbf{\rHLate}\,{\tiny\rHLateCI} \\
DR & \rDREarly\,{\tiny\rDREarlyCI} & \rDRLate\,{\tiny\rDRLateCI} \\
DRo$^\ast$ & \rDRoEarly\,{\tiny\rDRoEarlyCI} & \rDRoLate\,{\tiny\rDRoLateCI} \\
\bottomrule
\end{tabular}
\vspace{-1mm}
\caption{Gradient noise relative to shared reward, $\kappa(X)$; below $1$ is quieter ($95\%$ intervals over prompts). $^\ast$Exact contribution, graded offline; not a usable signal. All signals: Table~\ref{tab:hybrid-full}.}
\label{tab:hybrid-main}
\end{minipage}\hfill
\begin{minipage}[t]{0.57\textwidth}
\centering
\vspace{0pt}
\input{tables/e8_main}
\vspace{-1mm}
\caption{Nine-agent teams: mean accuracy gain over the untrained agent, and collapsed copies (gain $<-0.05$). Bold: best at that rate. Top three rates: the pre-registered test. $^\dagger$Rate used with three agents.}
\label{tab:e8}
\end{minipage}
\vspace{-3mm}
\end{table}

\subsection{Training}
\label{sec:train3}
\label{sec:train9}

\begin{localparagraphgap}{-0.5ex}

\paragraph{Three agents.}
The hybrid signal's lower noise did not improve training. Over \nSeedsPaired{} seeds of $500$ updates, paired across signals, the mean accuracy gains over the untrained agents on the $1{,}319$ test problems were \gainH{} (H), \gainS{} (SR) and \gainDR{} (DR). All four training predictions failed (H $-$ DR $=$ \diffHminusDR{} \diffHminusDRCI, one-sided $p=\pHminusDR$; Appendix Table~\ref{tab:hybrid-training}), and in one H run an agent collapsed. With three agents shared reward is only about $1.4$ times as noisy as H.

\paragraph{Nine agents: design.}
Shared reward gives each agent every other agent's luck, so its noise grows with team size: before training, the removal estimate's noise relative to shared reward (role~2) is \eEightStepDTwoKthree{} with three agents, \eEightStepDTwoKsix{} with six and \eEightStepDTwoKnine{} with nine (Appendix Table~\ref{tab:e8-step0}). A team has nine agents: three untrained agents of role~1 and three identical copies each of roles 2 and 3. The copies of a role share starting weights, prompts, picked answers and rewards, but each learns from a different signal, assigned at random: SR, W, or DR-agree (the removal estimate with the agreement outcome model). Comparing copies within a team removes the team's luck. The router picks uniformly at random, which bounds DR-agree's division ($p_j=1/9$) and, by Theorem~\ref{thm:wta}, makes W unbiased: all three signals have the same expected update and differ only in noise. Under this router DR-agree is also quieter than H (\eEightStepDTwoKnine{} against \eEightStepHTwoKnine{} for role~2 with nine agents), so the test uses DR-agree. \eEightTeams{} teams trained for $500$ updates at five learning rates (Table~\ref{tab:e8}). The outcome is greedy test accuracy minus the untrained agent's; a copy \emph{collapses} if it ends more than $0.05$ below. Everything was committed before any team ran (Appendix~\ref{app:e8}).

\paragraph{Nine agents: results.}
(i) \emph{DR-agree learns faster than shared reward}, the primary prediction: at the three test rates DR-agree's copy beat SR's in \eEightHOnePos{} of \eEightBlocks{} comparisons (exact sign-flip test, $p=\eEightHOneP$). (ii) \emph{It tolerates larger steps}: at $2\cdot10^{-5}$ SR collapsed in \eEightCollCS{} of \eEightSeatsC{} copies and DR-agree in \eEightCollCD; the best rate is \eEightBestS{} for SR and \eEightBestD{} for DR-agree, as the step-size argument of Section~\ref{sec:signals} predicts. (iii) Each at its own best rate, DR-agree gains \eEightDelta{} \eEightDeltaCI{} more than SR. (iv) \emph{DR-agree leads W but is not separated from it}: DR-agree's copy beat W's in \eEightHTwoPos{} of \eEightBlocks{} comparisons ($p=\eEightHTwoP$). The design favours W: the uniform router removes its bias (Theorem~\ref{thm:wta}) and makes it three times quieter than SR.

\paragraph{Exact checks.}
\label{sec:other-checks}
In eight small routed systems with exact gradients (Appendix~\ref{app:toy}), winner-take-all ends below shared reward in all \eFourNegAdam{} (Theorem~\ref{thm:wta}), and the quieter of two unbiased signals learns faster in \eFourAgreeAdam{} of \eFourPairs{} pairs. On the logged rollouts of Section~\ref{sec:noise}, Theorem~\ref{thm:selective}'s rule with a costed pilot cuts the error of the corrected removal estimates to \eFiveFeasible{} \eFiveFeasibleCI{} of uniform grading at a $25\%$ budget (Appendix~\ref{app:implementation}).

\end{localparagraphgap}

%% file: tables/e8_main.tex
\begin{tabular}{@{}lc ccc ccc@{}}
\toprule
& & \multicolumn{3}{c}{Mean gain} & \multicolumn{3}{c}{Collapsed} \\
\cmidrule(lr){3-5}\cmidrule(l){6-8}
Rate & Teams & SR & W & DR-agree & SR & W & DR-agree \\
\midrule
$5\cdot10^{-6}$ & 5 & $+$0.003 & $+$0.021 & \textbf{$+$0.046} & 0/10 & 0/10 & 0/10 \\
$10^{-5}$$^\dagger$ & 5 & $+$0.028 & \textbf{$+$0.065} & $+$0.049 & 0/10 & 0/10 & 0/10 \\
$2\cdot10^{-5}$ & 5 & $-$0.117 & $+$0.024 & \textbf{$+$0.083} & 4/10 & 1/10 & 0/10 \\
\midrule
$4\cdot10^{-5}$ & 3 & $-$0.452 & $-$0.109 & \textbf{$-$0.075} & 6/6 & 3/6 & 3/6 \\
$8\cdot10^{-5}$ & 3 & $-$0.584 & $-$0.560 & \textbf{$-$0.497} & 6/6 & 6/6 & 5/6 \\
\bottomrule
\end{tabular}

%% file: sections/slates.tex
\section{Signals that Use All Completions of a Prompt}
\label{app:slates}

Section~\ref{sec:signals} treats one rollout. The experiment samples $m$ completions per agent for each prompt, and some signals read all of them (the agreement outcome model of Section~\ref{sec:experiment}). This appendix extends Lemma~\ref{lem:baseline} and Proposition~\ref{prop:variance-monotone}(iii) to that setting, gives the mean and the noise term of the centred update used in the measurement, and proves Corollary~\ref{cor:hybrid}.

\paragraph{Setting.}
Let $m\ge2$, fix a routed layer with agents $1,\dots,K$, and let $H$ be the (random) prompt. For $s=1,\dots,m$, slate $s$ holds one completion $A_s^{(j)}\sim\pi_{\theta_j}(\cdot\mid H)$ of every agent; all $mK$ completions are independent given $H$. Write $\mathcal A:=(A_s^{(j)})_{s,j}$ and $\mathcal A^{(-i)}:=(A_s^{(j)})_{s,\,j\ne i}$. Given $(H,\mathcal A)$ the pairs $(I_s,G_s)$, $s=1,\dots,m$, are independent: $I_s$ has the probabilities $p_{s,j}$ that \eqref{eq:softmax-router} assigns on slate $s$, the reward is $G_s=R_s^{(I_s)}$, and $\E[G_s\mid H,\mathcal A,I_s=j]=g_j(H,A_s^{(j)})$ as in \eqref{eq:candidate-outcome}, so $\bar g_s:=\E[G_s\mid H,\mathcal A]=\sum_jp_{s,j}\,g_j(H,A_s^{(j)})$. An outcome model may also use randomness that is independent of everything else (a random sub-pool of completions, say); every statement below then holds conditionally on it. For agent $i$ let $\Psi_s:=\nabla_{\theta_i}\log\pi_{\theta_i}(A_s^{(i)}\mid H)$, and let $X_s$ be its signal on slate $s$. The centred update is
\begin{equation}
\label{eq:centred-update}
u=\frac1m\sum_{s=1}^{m}\Psi_s\Bigl(X_s-\frac1{m-1}\sum_{t\ne s}X_t\Bigr)
=\frac1m\sum_{s=1}^{m}\widetilde\Psi_sX_s,
\qquad
\widetilde\Psi_s:=\Psi_s-\frac1{m-1}\sum_{t\ne s}\Psi_t ,
\end{equation}
where the second form follows by exchanging the order of the double sum. Write $\widetilde\omega_s:=\norm{\widetilde\Psi_s}^2$.

\begin{proposition}[Centred update with $m$ slates]
\label{prop:slates}
Suppose that for every slate $s$, $\E[X_s^2]<\infty$, $\E[\norm{\Psi_t}^2X_s^2]<\infty$ for all $t$, and
\begin{equation}
\label{eq:slate-form}
\E\left[X_s\mid H,\mathcal A\right]=\bar g_s-\beta_s,
\qquad
\beta_s \text{ a function of } (H,\mathcal A^{(-i)}).
\end{equation}
Let $\zeta_s:=X_s-(\bar g_s-\beta_s)$.
\begin{enumerate}[label=(\roman*),leftmargin=*,itemsep=1pt,topsep=1pt]
\item $\E[u]=\nabla_{\theta_i}J_{\mathrm{sys}}$.
\item If $\E[\zeta_s\zeta_t\mid H,\mathcal A]=0$ for $s\ne t$, then
\begin{equation}
\label{eq:slate-second-moment}
\E\norm{u}^2
=
\E\,\Bigl\lVert\frac1m\sum_{s}\widetilde\Psi_s(\bar g_s-\beta_s)\Bigr\rVert^2
+
\frac1{m^2}\sum_{s}\E\left[\widetilde\omega_s\,\E\left[\zeta_s^2\mid H,\mathcal A\right]\right].
\end{equation}
\item For each single completion, Proposition~\ref{prop:variance-monotone}(iii) holds with $(H,A^{(N)})$ replaced by $(H,\mathcal A)$: $\operatorname{tr}\operatorname{Cov}(\Psi_sX_s)$ equals the right-hand side of \eqref{eq:three-term} with $\beta=\beta_s$, $\zeta=\zeta_s$, and the same $B^\star$.
\end{enumerate}
\end{proposition}

The condition in (ii) holds whenever $\zeta_s$ depends on the router draws and the reward noise of slate $s$ only, since these are independent across slates given $(H,\mathcal A)$; this covers every signal in Table~\ref{tab:signals}. Two signals with the same $\beta_s$ then differ in $\E\norm u^2$, and hence in $\operatorname{tr}\operatorname{Cov}(u)$, only through the last term of \eqref{eq:slate-second-moment}, which is the quantity compared in Section~\ref{sec:results}.

\begin{proof}
(i) Fix $s$ and $t\ne s$. Neither $\bar g_t$ nor $\beta_t$ depends on $A_s^{(i)}$, and $A_s^{(i)}$ is independent of all other completions given $H$, so
$\E[\Psi_sX_t]=\E[\Psi_s(\bar g_t-\beta_t)]=\E\bigl[\E[\Psi_s\mid H,\mathcal A\setminus A_s^{(i)}]\,(\bar g_t-\beta_t)\bigr]=0$,
because the score has conditional mean zero. For $t=s$, $\E[\Psi_sX_s]=\E[\Psi_s\bar g_s]-\E[\Psi_s\beta_s]$; the second term vanishes for the same reason and the first equals $\E[\Psi_sG_s]=\nabla_{\theta_i}J_{\mathrm{sys}}$. Averaging over $s$ gives the claim.

(ii) By \eqref{eq:centred-update}, $u=\frac1m\sum_s\widetilde\Psi_s(\bar g_s-\beta_s)+\frac1m\sum_s\widetilde\Psi_s\zeta_s$. Given $(H,\mathcal A)$ the first sum and every $\widetilde\Psi_s$ are fixed, $\E[\zeta_s\mid H,\mathcal A]=0$, and the $\zeta_s$ are uncorrelated, so the cross terms vanish and $\E[\norm u^2\mid H,\mathcal A]$ equals the squared norm of the first sum plus $m^{-2}\sum_s\widetilde\omega_s\E[\zeta_s^2\mid H,\mathcal A]$.

(iii) The proof of Proposition~\ref{prop:variance-monotone}(iii) uses only that $\Psi_i$, $\bar g$ and $\beta$ are measurable with respect to the conditioning variables, that $\E[\zeta\mid\cdot]=0$, and that $\beta$ is a function of variables independent of $A^{(i)}$ given $H$ with $\E[\omega_i\mid H,\cdot]=\bar\omega_i$. All three hold for $(H,\mathcal A)$ and $(H,\mathcal A^{(-i)})$; and since slate $s$ is independent of the other slates given $H$, $\E[\omega_s\bar g_s\mid H,\mathcal A^{(-i)}]=\E[\omega_s\bar g_s\mid H,A_s^{(-i)}]$, so $B^\star$ is unchanged.
\end{proof}

\begin{remark}[Baselines that read agent $i$'s completions]
\label{rem:cross-slate}
If $\E[X_s\mid H,\mathcal A]=\bar g_s-\beta_s$ with $\beta_s$ any function of $(H,\mathcal A)$, the proof of (i) gives
$\E[u]-\nabla_{\theta_i}J_{\mathrm{sys}}=\frac1{m(m-1)}\sum_{s\ne t}\E[\Psi_s\beta_t]-\frac1m\sum_s\E[\Psi_s\beta_s]$.
Condition \eqref{eq:slate-form} makes every term vanish; it is sufficient, not necessary. It excludes agent $i$'s completions on \emph{all} slates from $\beta_s$, not only $A_s^{(i)}$: if $\beta_t$ reads $A_s^{(i)}$ for $s\ne t$ but not $A_t^{(i)}$, each $\Psi_sX_s$ alone is unbiased and the centred update has bias $\frac1{m(m-1)}\sum_{s\ne t}\E[\Psi_s\beta_t]$, which is nonzero in general. With an agreement baseline whose pool drops only $A_t^{(i)}$, exact enumeration of $40$ small routed systems gives a bias with median $17\%$ of the norm of the system gradient (all agents' parameters; zero in three systems, at most $2.1$ times the norm). A term with conditional mean zero given $(H,\mathcal A)$, such as the control variate in \eqref{eq:hybrid}, may read every completion.
\end{remark}

\paragraph{Proof of Corollary~\ref{cor:hybrid}.}
We prove the $m$-slate statement for $m\ge2$; for a single rollout the same computation of $\E[X\mid H,A^{(1:K)}]$ verifies the hypothesis of Proposition~\ref{prop:variance-monotone}(iii). On slate $s$ let $\widehat\mu_{s,j}$ be a function of $(H,\mathcal A)$ and $\widehat\mu^{-i}_{s,j}$ a function of $(H,\mathcal A^{(-i)})$, and let
$X_s=G_s-\widehat\mu_{s,I_s}+\sum_jp_{s,j}\widehat\mu_{s,j}-\sum_{j\ne i}p^{-i}_{s,j}\widehat\mu^{-i}_{s,j}$.
Given $(H,\mathcal A)$ only $I_s$ and $G_s$ are random, and $\E[\widehat\mu_{s,I_s}\mid H,\mathcal A]=\sum_jp_{s,j}\widehat\mu_{s,j}$, so
$\E[X_s\mid H,\mathcal A]=\bar g_s-\beta_s$ with $\beta_s=\sum_{j\ne i}p^{-i}_{s,j}\widehat\mu^{-i}_{s,j}$. The removed probabilities $p^{-i}_{s,j}$ are computed from the candidates $j\ne i$ of slate $s$, so $\beta_s$ is a function of $(H,\mathcal A^{(-i)})$ and \eqref{eq:slate-form} holds. The noise is
$\zeta_s=(G_s-\widehat\mu_{s,I_s})-\E[G_s-\widehat\mu_{s,I_s}\mid H,\mathcal A]$, a function of $(I_s,G_s)$ given $(H,\mathcal A)$, with
$\E[\zeta_s^2\mid H,\mathcal A]=\Var(G_s-\widehat\mu_{s,I_s}\mid H,\mathcal A)$; when each candidate's reward is a function of prompt and candidate this is $\Var_{I\sim p_s}(R_s^{(I)}-\widehat\mu_{s,I})$. Proposition~\ref{prop:slates} gives gradient-correctness and the decomposition.

\emph{The family between $X^{\mathrm H}$ and $\widehat\Delta^{\mathrm{rem}}$.} Let $\epsilon>0$, so that $p_{s,j}\ge\epsilon/K$. For $X_{s,\lambda}=X_s-\lambda C_s$ with $C_s=\sum_{j\ne i}p^{-i}_{s,j}\frac{\1_{\{I_s=j\}}}{p_{s,j}}(G_s-\widehat\mu^{-i}_{s,j})$, we have $\E[C_s\mid H,\mathcal A]=\sum_{j\ne i}p^{-i}_{s,j}(g_j(H,A_s^{(j)})-\widehat\mu^{-i}_{s,j})$, again a function of $(H,\mathcal A^{(-i)})$; hence \eqref{eq:slate-form} holds, and $X_{s,\lambda}$ is gradient-correct, for every constant $\lambda$, with $\beta_{s,\lambda}=\beta_s+\lambda\,\E[C_s\mid H,\mathcal A]$; at $\lambda=1$ its conditional mean is $\Delta_i^{\mathrm{rem}}$ for any pair of models. If $\widehat\mu_j=\widehat\mu_j^{-i}$ for all $j\ne i$, then $\widehat g_j$ of \eqref{eq:dr-candidate} satisfies $\sum_jp_j\widehat g_j=G-\widehat\mu_I+\sum_jp_j\widehat\mu_j$ and $\sum_{j\ne i}p_j^{-i}\widehat g_j=\sum_{j\ne i}p_j^{-i}\widehat\mu_j+C$, so $X_{1}=\widehat\Delta_i^{\mathrm{rem}}$. \qed

\paragraph{The noise of DR against the noise of the hybrid signal.}
With $w_{ij}=p_j-\1_{\{j\ne i\}}p_j^{-i}$ as in Appendix~\ref{app:dr-details}, the noise of $\widehat\Delta_i^{\mathrm{rem}}$ around its conditional mean is
\[
\zeta^{\mathrm{DR}}=\frac{w_{iI}}{p_I}\,(G-\widehat\mu_I)-\sum_jw_{ij}\,(g_j-\widehat\mu_j),
\qquad
\frac{w_{iI}}{p_I}=1-\1_{\{I\ne i\}}\frac{p_I^{-i}}{p_I},
\]
while the hybrid signal has $\zeta^{\mathrm{H}}=(G-\widehat\mu_I)-\sum_jp_j(g_j-\widehat\mu_j)$: the same residual with weight one. When the router concentrates on candidate $i$, $p_I^{-i}/p_I$ is large on the rare draws $I\ne i$. When rewards are functions of prompt and candidate, with $e:=\max_j\lvert R^{(j)}-\widehat\mu_j\rvert$, $\E[(\zeta^{\mathrm{DR}})^2\mid H,A^{(1:K)}]$ equals $\sum_{j\ne i}(p_j^{-i})^2(R^{(j)}-\widehat\mu_j)^2/p_j$ plus a remainder between $-6e^2$ and $e^2$, while $\E[(\zeta^{\mathrm H})^2\mid H,A^{(1:K)}]=\Var_{I\sim p}(R^{(I)}-\widehat\mu_I)\le e^2$ has no such factor. Neither noise is smaller in general; as the router concentrates on candidate $i$, $p_j\to\epsilon/K$ for $j\ne i$ and DR's noise approaches $\frac K\epsilon\sum_{j\ne i}(p_j^{-i})^2(R^{(j)}-\widehat\mu_j)^2$ up to the same remainder, while the hybrid's stays below $e^2$. In exchange the hybrid baseline is $\sum_{j\ne i}p_j^{-i}\widehat\mu_j^{-i}$ instead of $\sum_{j\ne i}p_j^{-i}g_j$, so its mismatch term in \eqref{eq:three-term} depends on the outcome model, and $X^{\mathrm H}_i$ is not an unbiased estimate of $\Delta_i^{\mathrm{rem}}$: $\E[X^{\mathrm H}_i\mid H,A^{(1:K)}]=\Delta_i^{\mathrm{rem}}+\sum_{j\ne i}p_j^{-i}(g_j-\widehat\mu_j^{-i})$.

%% file: sections/implementation.tex
\section{Implementation Details}
\label{app:implementation}

\paragraph{Agents and prompts.}
Agent roles are direct arithmetic, equation-first reasoning, and quantity-first reasoning (the key quantities on one line, then one computation line), with sampling temperatures $(0.35,0.50,0.72)$, top-$p$ $(0.80,0.85,0.95)$ and no top-$k$ truncation. Prompts use the model's chat template: a system message states the role and asks for two to four short lines ending with \texttt{\#\#\#\# <number>}, followed by three few-shot exemplars as earlier user and assistant turns (GSM8K training problems 0, 2 and 4, outside the $512$ training problems and the test set, with terse solutions in each agent's style). We assign money/rate, geometry/measurement, and counting/logic specialties; a keyword heuristic labels prompts and tells each agent whether a prompt is in its specialty. One extractor serves the reward, the agreement $\widehat\mu$ and evaluation: the number after \texttt{\#\#\#\#} if present, otherwise the last number, compared numerically. A completion that reaches the cap without an end-of-sequence token has no answer and reward $0$. Without \texttt{\#\#\#\#}, a hyphen before a number is read as a minus sign (``10-3'' gives $-3$); this fallback applies to $0.2$--$0.5\%$ of finished completions and is the same in every arm. LoRA adapters use rank $16$, scaling $32$, zero dropout, on all attention and MLP projections; the backbone is loaded in bf16 without quantization in every run, and all runs use one GPU type.

\paragraph{Router and outcome model.}
One network mean-pools hashed token embeddings (32{,}768 buckets, dimension 64) of the prompt and of the candidate, concatenates an agent embedding (dimension 16), and applies a one-hidden-layer predictor (width 128). Its logit $s_j$ defines $p_j=(1-\epsilon)\exp(s_j/\tau)/\sum_k\exp(s_k/\tau)+\epsilon/3$, and its sigmoid is $\widehat\mu_j$. Over the run, $\tau$ decreases linearly from $1.00$ to $0.70$ and $\epsilon$ from $0.05$ to $0.03$; removal renormalizes the same rule over the two remaining candidates. The network learns only from selected feedback with a $50{,}000$-entry replay buffer, batches of $64$, inverse-propensity-weighted binary cross-entropy, and eight replay steps per update during the $25$-update warmup and four thereafter. A propensity floor of $0.05$ and a weight clip of $3.0$ enter only this loss. Actor signals use the exact factual propensities, bounded below by $\epsilon/K$ ($0.01$ at the end of the anneal), as Corollary~\ref{cor:dr-gradient} requires, and are computed before the current rollout enters replay.

\paragraph{Actor update.}
Each run uses a fixed set of $512$ training problems, $150$ updates of one prompt each, $m=4$ completions per agent, and completions of at most $L=128$ tokens, chosen by the length calibration below. GRPO standardizes each agent's four signals as $\mathrm{Adv}_{i,g}=(X_{i,g}-\overline X_i)/\max\{\operatorname{sd}(X_{i,1:m}),10^{-6}\}$, with the sample standard deviation. Token log-probabilities are those of the agent's sampling distribution: temperature-scaled logits restricted to the top-$p$ nucleus, so the score function is that of the policy that generated the data. The output layer is evaluated in fp32, for sampling and for these log-probabilities alike. The loss mask includes the end-of-sequence token, so stopping is trained. The loss sums the token log-probabilities of the group's $m$ completions weighted by their advantages and divides by the group's total token count. One gradient step per fresh rollout makes the importance ratio identically one, so the clip ($0.2$) is inactive. The loss also contains, with coefficient $0.02$, the mean sampled-token log-probability under the sampling distribution; its expected gradient, $0.02\,\nabla_\theta\E[1/T]$ with $T$ the group's token count, is a weak length term identical across arms. LoRA uses AdamW at $10^{-5}$ without weight decay and with gradient-norm clipping at $1.0$ per agent; the router uses AdamW at $10^{-3}$.

\paragraph{Evaluation.}
All $1{,}319$ GSM8K test problems, every $25$ updates, in one batched greedy pass that grades every candidate and logs $\widehat\mu$; the router deploys the highest-scoring candidate. The extra grades are diagnostics only. The frozen reference is the base model with agent~1's prompt and greedy decoding, run once.

\paragraph{Seeds and pairing.}
A seed controls the LoRA initialization, completion sampling, router draws, replay minibatches and the order of training prompts; the problem sets are fixed by a separate data seed. For a given seed every arm starts from the same initialization, so arms are paired. Every result file records the code version, GPU type, precision, batch sizes and wall-clock time.

\paragraph{Gradient-noise measurement.}
Checkpoints at updates $25$, $88$ and $150$ of three shared-reward seeds; the same fixed list of $256$ training prompts at every checkpoint; $m=4$ slates per prompt; router draws from a fixed seed. For each completion $s$ of agent $i$, one backward pass gives $\Psi_s$ with respect to agent $i$'s adapter (about $2.4\times10^{7}$ parameters). We store the exact $4\times4$ Gram matrix of each prompt--agent group, exact sums of the per-prompt updates of every signal and form over each half of the prompt list (only their inner products are kept), and a CountSketch ($2^{20}$ buckets) of every $\Psi_s$, whose Gram matrix gives the cross-prompt inner products for the bootstrap. $\E\lVert u\rVert^2$ is exact; $\lVert\E u\rVert^2$ is the inner product of the two half-sums divided by the product of their sizes, which is unbiased. Agreement model of DR-agree: for agent $i$'s signal on slate $g$, $\widehat\mu_j$ is the fraction of the prompt's twelve completions other than candidate $j$ and agent $i$'s slate-$g$ completion whose final answer equals candidate $j$'s. This model reads agent $i$'s completions on other slates; it enters only $\widehat\Delta^{\mathrm{rem}}$, whose conditional mean given all completions does not depend on the outcome model (the model reads completions only, not $I$ or $G$), and never a baseline (Remark~\ref{rem:cross-slate}). H uses two agreement models: $\widehat\mu_j$, the share of the prompt's other eleven completions with candidate $j$'s final answer, and $\widehat\mu_j^{-i}$, the same share computed without agent $i$'s completions. In all three, $\widehat\mu_j=0$ if candidate $j$ has no parseable answer. DRo and RLo use the logged rewards in place of $\widehat\mu$ and are oracle diagnostics, never trained. As a check on the score, we report $\E_H\lVert\E[\Psi\mid H]\rVert^2/\E\lVert\Psi\rVert^2$ per checkpoint and agent (Table~\ref{tab:score-mean}), estimated by the off-diagonal entries of each group's Gram matrix; it is zero for the exact score of the sampling distribution. It is near $0$ for agents 2 and 3 and, for agent~1 (temperature $0.35$, top-$p$ $0.8$), $0.04$ at $u=25$ and $88$ and $0.06$--$0.14$ at $u=150$. This is consistent with top-$p$ boundaries that differ between generation (cached, padded batch) and rescoring (full forward) in bf16, amplified by the low temperature. The prompt-level mean $\E[\Psi\mid H]$ cancels in every centred update, whose coefficients sum to zero within each group; the completion-level score error it indicates does not.

\paragraph{Length calibration.}
Before any data at the final configuration, we generated with a $512$-token cap greedily on the first $256$ GSM8K test problems and with the training temperatures and top-$p$ ($m=4$) on $128$ of the training problems, for three prompt sets (Table~\ref{tab:calibration}). Truncation and capped accuracy at every smaller cap $L$ follow exactly, because capped generation is the uncapped prefix. The fixed rule took the smallest $L$ at which, for every agent, sampled truncation is at most $5\%$, the answer parses in at least $95\%$ of finished completions, and capped greedy accuracy is at least $0.30$ and within $0.05$ of uncapped. It selected prompt set v3 with $L=128$. A further sanity threshold (uncapped accuracy at least $0.55$ for agents 1 and 2) failed for agent~1 ($0.531$); with no pipeline error found, we kept the rule's choice and recorded the override before any data. The calibration problems are part of the evaluation set; they were used only to choose the prompt set and $L$. Calibration decoded with the bf16 output head; with the fp32 head used in all runs, the frozen model scores $0.516$ on the same $256$ problems ($0.513$ on all $1{,}319$).

\paragraph{Confirmation cells.}
The three confirmation policies (seeds 4--6) use the training recipe above with two changes, both fixed before their data. First, all three adapters are trained in fp32. In the design runs the adapter library had left the adapters of agents 2 and 3 in bf16, where an AdamW step of $10^{-5}$ does not change a weight of magnitude above $2^{-8}$: in the one design run whose checkpoints were kept (seed 3), $81\%$ and $83\%$ of the entries of their \texttt{lora\_A} matrices are identical at updates 25 and 150 (agent 1, in fp32: none), while \texttt{lora\_B}, initialised at zero, moved as far as agent 1's. Second, the score is evaluated on the sampling distribution that generated the completion, as described next. Prompt lists: the design list is $256$ of the $512$ training problems; list B is the other $256$; lists C and D are $256$ problems each drawn from the remaining GSM8K training problems, which no policy, prompt or design step has used. Each cell measures $256$ prompts $\times$ $3$ agents $\times$ $4$ completions. Each prompt's $m\times m$ Gram matrix of completion scores, per agent, gives the centred update \eqref{eq:centred-update} of every signal, and $\kappa(X)$ sums $\E_I\lVert u_X\rVert^2$ over prompts, agents and cells, with the exact expectation over the router's $3^4$ draws per prompt. Pooled, the squared mean update is below $0.1\%$ of each signal's second moment, so $\kappa$ equals the ratio of $\operatorname{tr}\operatorname{Cov}(u)$ to within $0.1\%$.

\paragraph{Score of the sampling distribution.}
Agents sample with temperature and top-$p$, so the score is the gradient of the log-probability under the truncated distribution. Top-$p$ makes $\pi_\theta$ discontinuous in $\theta$ where a nucleus changes; all identities are applied to the sampling policy with every nucleus held fixed, which gives its gradient wherever no token lies on a nucleus boundary, and $J_{\mathrm{sys}}$ is the expected reward of the sampling policies. The design cells recomputed the nucleus when the completion was rescored. Generation runs batched with a key--value cache and rescoring runs without, the two passes differ in bf16, and near the top-$p$ boundary they can select different nuclei: on the first confirmation cell $0.4\%$, $0.8\%$ and $2.5\%$ of the tokens of agents 1--3 have a different nucleus. The confirmation cells therefore record, for every generated token, the set of tokens kept by top-$p$ and their scaled logits, and score on the recorded set. For a sample of prompts we also compute the exact conditional mean of each token's score under the recorded generation-time distribution, which is zero for an exact score. We predicted that with the recorded set agent 1's mean would be below $1\%$ of the score's second moment. It was not: $6.5\%$ $[0.3\%,15.1\%]$ on the first confirmation cell ($10.2\%$ with the recomputed nucleus; $0.1\%$ and $0.03\%$ for agents 2 and 3), because the rescoring pass assigns slightly different probabilities on the same set. After this failure, and therefore after seeing data, we changed the measured score to the \emph{centred} score: from each token's score we subtract its exact expectation under the recorded generation-time distribution, so the conditional mean is zero by construction. That first cell was measured with both scores on the same completions: $\kappa(\mathrm{H})$ is $0.781$ centred and $0.779$ uncentred, $\kappa(\mathrm{DR})$ is $20.2$ and $19.9$, $\kappa(\mathrm{DRo})$ $0.504$ and $0.505$; the largest relative change over the signals of Table~\ref{tab:hybrid-full} is $1.8\%$ and no ordering changes. Without the policy of that cell (seeds 5 and 6 only, \nCellsNewSeeds{} cells), $\kappa(\mathrm H)=\rHNewSeeds$ \rHNewSeedsCI. All confirmation cells use the centred score; training uses the uncentred score on the recorded set.

\paragraph{Training in the plain update.}
The training comparison of Section~\ref{sec:train3} uses the update that Proposition~\ref{prop:slates} covers: for agent $i$ and a prompt with $m=4$ slates, the loss is $-\frac1m\sum_s a_s\log\pi_{\theta_i}(A_s^{(i)}\mid H)$ with $a_s=X_s-\frac1{m-1}\sum_{t\ne s}X_t$ and the sequence-sum log-probability; there is no standardization, no token normalisation, no gradient-norm clipping and no length term. AdamW, learning rate $10^{-5}$, $500$ updates of one prompt each after the $25$-update router warm-up, with $\tau$ and $\epsilon$ annealed over the $500$ updates. The learning rate was chosen before the comparison by a rule fixed in advance, on one shared-reward seed and $300$ updates: the larger of $\{10^{-5},5\cdot10^{-5}\}$ whose agent-mean accuracy does not fall more than $0.02$ below the frozen agents; at $5\cdot10^{-5}$ agent 1 collapsed (accuracy $0.295$), so $10^{-5}$ was used; the comparison was to run only if agents 2 and 3 gained at least $0.02$ at the chosen rate (they gained $0.069$; agent 1 gained $0.002$). Runs of different signals with the same seed share the adapter initialisation, the prompt order and the sampling seeds. The hybrid signal's agreement models are computed from the twelve completions of the training rollout itself; no signal uses additional generation or the reward of an unselected candidate. The outcome is the accuracy gain of the greedy agents over the frozen agents on the $1{,}319$ test problems, averaged over the three agents. \firstBatchSentence{} Seed 17 and three winner-take-all runs in the plain update had been planned and were cancelled when GPU work was stopped early; by then the failure of the primary prediction no longer depended on them (with the six differences of seeds 11--16 the test cannot reach significance for any outcome of two further seeds). \incompleteRunsSentence{} In seed 16 the H run lost agent 1: its accuracy was $0.541$ at update $250$ and $0.222$ at update $500$, while agents 2 and 3 gained $0.027$ and $0.063$. Its largest update came at update $268$ (gradient norm $124$, on a prompt for which all four of its completions were wrong), and from about update $270$ its sampled completions increasingly lack the final-answer line (about $2\%$ before, $50$--$60\%$ at the end). Other runs had larger single updates without collapsing (largest update norm of agent 1 in seeds 11--15: $105$--$170$ for H, $121$--$224$ for SR, $212$--$1{,}509$ for DR), and in the learning-rate pilot shared reward collapsed the same agent at $5\cdot10^{-5}$. The run is kept in every number. The stop rule fixed in advance, an arm more than $0.05$ below the frozen agents at update $250$ in two seeds, was not met.

\paragraph{Second dataset.}
SVAMP \citep{patel2021svamp} was measured with the same backbone at the end of the router warm-up (update $25$ of a $150$-update schedule: trained router, untouched adapters), two seeds, with the centred score, on $256$ of a $512$-problem training subset. Prompts, temperatures, top-$p$, the completion cap and the adapter shape are unchanged. Before measurement, each configuration runs the pipeline's checks on its own model and data ($16$ prompts); one check requires the exact conditional mean of the uncentred score to stay below $3\%$ of its second moment for every agent, and by a rule fixed in advance a configuration that fails a check counts for no prediction. SVAMP passed (\smokeRatioSvamp{} for agents 1--3): $\kappa(\mathrm H)=\rHSvamp$ \rHSvampCI, $\kappa(\mathrm{DR})=\rDRSvamp$ and $\kappa(\mathrm{DRo})=\rDRoSvamp$. The same statistic on the same $16$ prompts in its two cells is \cellDiagSvamp, above the limit for agent 1 in both, and over the confirmation cells agent 1's ratio is \exactRatioAgentOne{} (prediction C8); this mismatch of a few percent is why every $\kappa$ is measured with the centred score. A second planned configuration, GSM8K with \texttt{Qwen2.5-3B-Instruct} \citep{qwen25}, failed this check, for a reason we have not identified, so prediction C10 is decided on SVAMP alone and we do not report that configuration.

\paragraph{Theorem~\ref{thm:selective} on logged slates.}
The confirmation cells store the grade of every candidate, so the allocation rule of Section~\ref{sec:selective} can be evaluated offline, with slates as entries. Evaluating a slate grades its two unselected candidates, at the same cost for every slate, and returns the exact removal contribution of all three agents, so $\sigma^2=0$ and the weight of slate $s$ is $w_s=\sum_i a_{i,s}^2\eta_{i,s}^2$, with $\eta_{i,s}$ the error of a learned contribution signal against the exact one and $a_{i,s}$ the norm of agent $i$'s centred score on that slate. An allocation $p$ has risk $\sum_s w_s(1/p_s-1)$ under the budget $\sum_s p_s=BN$ in each cell of $N=1{,}024$ slates. We compare uniform evaluation, the oracle rule \eqref{eq:optimal-query}, and a feasible two-stage rule: a uniform pilot spends a fifth of the budget; for each agent, $\eta^2$ on the pilot is regressed on the squared propensity correction, which is observed without grading and, for DR, contains the observable part of its error; the rest of the budget follows \eqref{eq:optimal-query} with the fitted proxy and a floor of $0.2B$. The corrected estimate stays unbiased given the pilot. The pilot's cost is part of the feasible risk: with an uninformative proxy the two-stage rule is worse than uniform, by $19\%$ at $B=0.25$. Two reference lines separate what the leverage $a^2$ alone gives, which needs no evaluation, from what the pilot adds: allocation by leverage without a pilot, and the two-stage rule without the correction feature. Table~\ref{tab:allocation} gives the risk relative to uniform over the \eFiveNCells{} cells (\eFiveNSlates{} slates). For the submitted DR signal at $B=0.25$ the oracle rule attains \eFiveOracle{} \eFiveOracleCI{} of uniform's risk, leverage alone \eFiveLeverage{} \eFiveLeverageCI, the two-stage rule without the feature \eFiveNoFeature, and the feasible rule \eFiveFeasible{} \eFiveFeasibleCI{} (single cells: \eFiveCellMin--\eFiveCellMax; update 25: \eFiveFeasibleEarly; update 150: \eFiveFeasibleLate, whose cells carry \eFiveLateShare{} of the pooled risk), which is \eFiveRecovered{} \eFiveRecoveredCI{} of the oracle's reduction. For DR the gain therefore comes from the correction feature, that is, from the part of DR's own error that is visible without grading; for the plug-in signal the feasible rule (\eFiveFeasiblePL) adds little to leverage alone (\eFiveLeveragePL) at this budget. \eFiveSummary: that the feasible rule is below uniform, that it recovers at least half of the oracle's reduction, and that the oracle's ratio is smaller for DR than for the plug-in signal (\eFiveOraclePL), which the heavier tail of DR's error had made likely. Design, proxy and budgets were fixed before the computation, which was run once; intervals resample prompts, one resample per prompt list. The exercise says nothing about training, about noisy or unequally priced evaluations, or about the masked critic.

\begin{table}[H]
\centering
\scriptsize
\setlength{\tabcolsep}{4pt}
\input{tables/allocation}
\caption{Allocating exact evaluations on logged slates: risk relative to uniform allocation at the same budget $B$ (expected share of slates evaluated), pooled over the confirmation cells. Oracle: \eqref{eq:optimal-query} with the true errors. Leverage: allocation by $\sum_ia_{i,s}^2$, no evaluation used. Feasible: two-stage rule with a uniform pilot and a proxy fitted on it; ``pilot, no feature'': the same without the correction feature.}
\label{tab:allocation}
\end{table}

\paragraph{Compute.}
Every job runs on one NVIDIA A40 (48\,GB). The three shared-reward runs took $5.8$ GPU-hours in total and the nine gradient-noise checkpoints $5.4$ (about $36$ minutes each); evaluation on the $1{,}319$ test problems takes about $14$ minutes per pass. In the confirmation study the \nCellJobs{} cell measurements and their three policies took $\gpuHoursCells$ GPU-hours (about $\minPerCell$ minutes per cell), the \nTrainingRuns{} training runs in the plain update $\gpuHoursTraining$ (about $\hoursPerTrainingRun$ hours per run with its two evaluations), and the cells on other data and with the other model $\gpuHoursGen$; with pilot runs, smoke tests and start-up the study was billed about $\gpuHoursBilled$ A40 GPU-hours.

\begin{table}[H]
\centering
\scriptsize
\setlength{\tabcolsep}{3pt}
\input{tables/calibration}
\caption{Length calibration, agent 1/2/3. Greedy: first $256$ test problems; sampled: $128$ training problems, $m=4$. $^\dagger$Selected configuration.}
\label{tab:calibration}
\end{table}

\section{Toy Routed System}
\label{app:toy}

$K=3$ agents, each with a tabular softmax policy over three actions per history; $4$--$6$ histories per instance; eight instances. Every (history, agent, action) has a fixed reward in $\{0,1\}$ and a fixed router score; the router is \eqref{eq:softmax-router} with $\tau=0.7$, $\epsilon=0.03$, and each prompt has $m=4$ slates. Form (a) is computed exactly (pairwise moments for signals that depend on one slate; enumeration over all slate multisets for RL); form (b) by Monte Carlo with $10^6$ prompts per instance ($10^5$ in the sweeps). The noisy outcome model adds fixed Gaussian noise (sd $0.3$, clipped to $[0,1]$) to the true reward, giving mean squared error $0.048$ on the deployed candidate; PL and RL use it. With deterministic rewards, DR with the true outcome model coincides with DRo. Varying $\tau\in\{0.3,0.5,0.7,1\}$ at $\epsilon=0.03$ and $\epsilon\in\{0.01,0.03,0.1,0.3\}$ at $\tau=0.7$, DRo's form-(a) variance stays between $0.22$ and $0.49$ of SR's (instance means); with the noisy outcome model, DR is worse than SR at $\tau=0.3$ ($1.80$), where RL ($0.83$) is the best non-oracle signal.

\begin{table}[H]
\centering
\scriptsize
\setlength{\tabcolsep}{3pt}
\input{tables/toy}
\caption{Toy system: mean and [min, max] over eight instances. Relative bias is $\lVert\E u-\nabla J\rVert/\lVert\nabla J\rVert$.}
\label{tab:toy}
\end{table}

\begin{figure}[H]
\centering
\includegraphics[width=\linewidth]{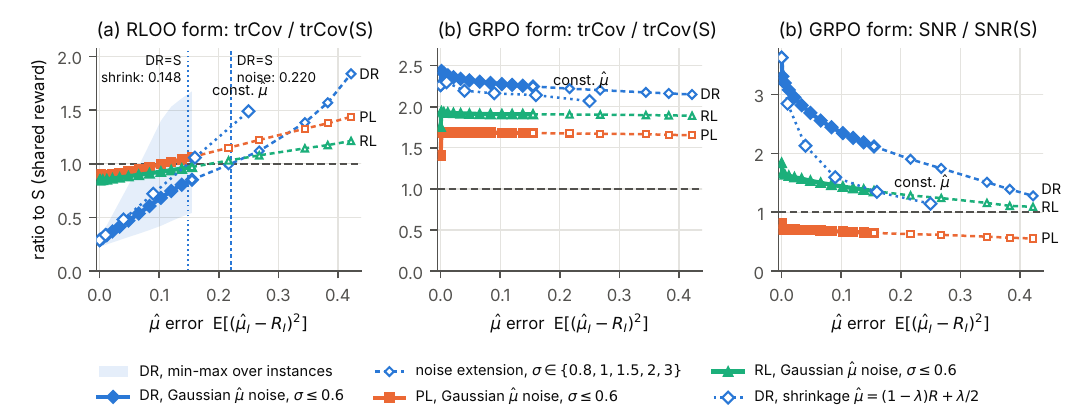}\\[2mm]
\includegraphics[width=\linewidth]{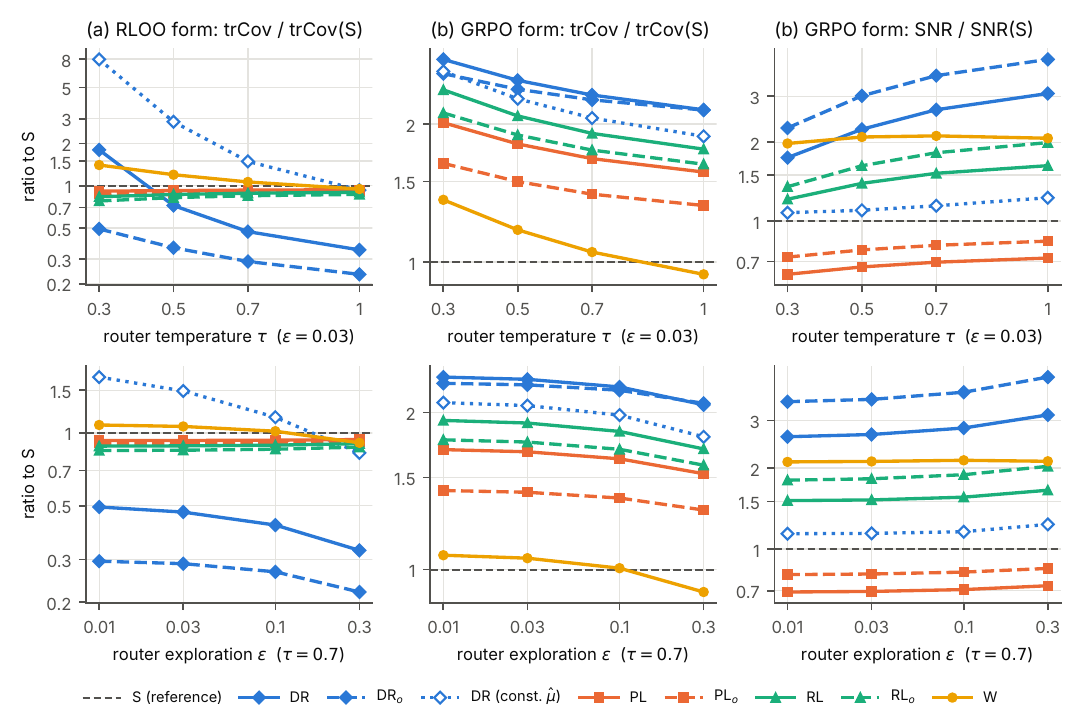}
\caption{Toy system. Top: variance relative to SR against outcome-model error. Bottom: sweeps over $\tau$ and $\epsilon$. In the figure, S is SR.}
\label{fig:toy}
\end{figure}

\paragraph{Learning curves.}
Here returns are exact, so the consequence of Theorem~\ref{thm:wta} and of the noise ordering can be followed in training. In each instance the three agents are trained together from their initial policies with one signal: one history per step, the centred update \eqref{eq:centred-update} with $m=4$, router and outcome model fixed, \eFourSteps{} steps, \eFourRuns{} runs per instance and signal, with Adam (step size $\eFourLrAdam$) and with plain SGD (step size $\eFourLrSgd$); each step size is the value of a four-point grid with the highest final return of shared reward on eight other instances. The hybrid signal uses the noisy outcome model in both of its terms. We report the normalised return $(J-J_0)/(J^\star-J_0)$, with $J$ computed by enumeration, $J_0$ its initial value and $J^\star$ the optimum, at the end of training and averaged over training (Table~\ref{tab:toy-learn}). Two predictions were fixed before these runs. (a) Winner-take-all ends below shared reward, in at least $6$ of the $8$ instances and on average. The difference of final returns, W minus SR, is negative in $\eFourNegAdam$ of $8$ instances with Adam (from \eFourDiffAdamMin{} to \eFourDiffAdamMax; mean \eFourDiffAdam) and in $\eFourNegSgd$ of $8$ with SGD (mean \eFourDiffSgd). The gap is one between limit points: in exploratory runs made afterwards, a step size too small for shared reward to converge within the horizon (Adam, $\eFourSensLr$) leaves the two at $\eFourSensS$ and $\eFourSensW$, and every larger step size of the grid gives the gap in all $8$ instances. (b) With Adam, of a gradient-correct signal and shared reward, the one with the lower second moment of the update at the initial policies has the larger average return over training, in at least $30$ of the $\eFourPairs$ signal--instance pairs. It does in $\eFourAgreeAdam$: in all $\eFourPairsLow$ pairs of DRo, H and DR with the noisy model, whose second moments are at most $\eFourLowMax$ of that of SR, and in $\eFourAgreeRest$ of the $\eFourPairsRest$ pairs of PL and of DR with the constant model, whose second moments are closer to that of SR or above it (with SGD, for which we made no prediction, $\eFourAgreeSgd$ of $\eFourPairs$).

\begin{table}[H]
\centering
\scriptsize
\setlength{\tabcolsep}{3pt}
\input{tables/toy_learn}
\caption{Toy system, training. Relative bias and second moment of the update at the initial policies (form (a), exact); normalised return $(J-J_0)/(J^\star-J_0)$ after \eFourSteps{} steps and its average over training (area). Mean and [min, max] over the eight instances of the mean over \eFourRuns{} runs. Table~\ref{tab:toy} reports the variance of the update, this table its second moment.}
\label{tab:toy-learn}
\end{table}

\section{Additional Experimental Results}
\label{app:gradnoise-extra}

Tables~\ref{tab:design}--\ref{tab:grpo-training} report the design cells with the statistic of Section~\ref{sec:experiment}, the confirmation cells, the predictions with their outcomes, and the training runs. The tables and figures after them are from the first study on the design cells (three shared-reward policies at updates 25, 88 and 150, one list of $256$ training prompts). They use the score with the nucleus recomputed at rescoring and the router draw that was logged, where Table~\ref{tab:design} uses the exact expectation over the router's draws; the five predictions of that study (P1--P5) were fixed before its data, and Table~\ref{tab:predictions-evidence} gives their outcomes.

\begin{table}[H]
\centering
\scriptsize
\setlength{\tabcolsep}{4pt}
\input{tables/hybrid_design}
\caption{Design cells: $\kappa(X)$ as in Table~\ref{tab:hybrid-main}, on the three shared-reward policies (seeds 1--3) and the one prompt list on which the hybrid signal was designed, pooled over seeds and agents. These cells use the score with the nucleus recomputed at rescoring (Appendix~\ref{app:implementation}). $^\ast$Oracle diagnostics.}
\label{tab:design}
\end{table}

\begin{table}[H]
\centering
\scriptsize
\setlength{\tabcolsep}{4pt}
\input{tables/hybrid_full}
\caption{Confirmation cells, all signals: $\kappa(X)$ as in Table~\ref{tab:hybrid-main}. $^\ast$Oracle diagnostics.}
\label{tab:hybrid-full}
\end{table}

\begin{table}[H]
\centering
\scriptsize
\setlength{\tabcolsep}{4pt}
\input{tables/hybrid_agents}
\caption{Confirmation cells per agent: $\kappa(X)$ restricted to one agent's adapter, and H per seed and checkpoint (three prompt lists pooled). Agent~2 receives most of the router's traffic. $^\ast$Oracle diagnostic.}
\label{tab:hybrid-agents}
\end{table}

\begin{table}[H]
\centering
\scriptsize
\setlength{\tabcolsep}{3pt}
\input{tables/hybrid_training}
\caption{Training in the plain update: $500$ updates, same seeds for every signal, greedy accuracy on the $1{,}319$ GSM8K test problems. Gains are relative to the frozen agents ($0.513$, $0.613$, $0.557$). Paired differences use a $t$ interval over seeds. \agentPatternSentence}
\label{tab:hybrid-training}
\end{table}

\begin{table}[H]
\centering
\scriptsize
\setlength{\tabcolsep}{4pt}
\input{tables/grpo_training}
\caption{Training in the GRPO form (first study): $150$ updates, standardized advantages, norm clip $1$; finished runs only. In these runs the adapters of agents 2 and 3 were trained in bf16 (Appendix~\ref{app:implementation}). No signal is separated from shared reward, and none exceeds the frozen agents with the router trained for $25$ updates by more than the seed-to-seed spread.}
\label{tab:grpo-training}
\end{table}

\begin{table}[H]
\centering
\scriptsize
\setlength{\tabcolsep}{3pt}
\input{tables/predictions2}
\caption{Predictions for the confirmation cells (C1--C11) and for training in the plain update (T1--T4), committed before the data (PREDICTIONS\_2.md in the code), with outcomes. Ratios are to shared reward unless stated; brackets are $95\%$ intervals. C8 was evaluated on the first confirmation cell with the uncentred score; after its failure the measured score was centred (Appendix~\ref{app:implementation}), a change made after seeing data. C11: standardization and clipping both fix the scale of every update, so a signal with a smaller second moment is scaled up again. C10 was decided on SVAMP alone, because the Qwen configuration failed one of the pipeline's checks (Appendix~\ref{app:implementation}). T1--T4 are evaluated on all seeds for which the three signals finished; the first \fbNSeeds{} seeds were analysed before the others were run.}
\label{tab:predictions2}
\end{table}

\begin{table}[H]
\centering
\scriptsize
\setlength{\tabcolsep}{3pt}
\input{tables/predictions_evidence}
\caption{First study: predictions fixed before its data (PREDICTIONS.md in the code) with the evidence for each outcome. P5 uses the training runs that finished (Table~\ref{tab:grpo-training}).}
\label{tab:predictions-evidence}
\end{table}

\begin{table}[H]
\centering
\scriptsize
\setlength{\tabcolsep}{3pt}
\input{tables/theory_check}
\caption{Measured and predicted variance differences, form (a), pooled over seeds and agents. T1: outcome-model noise, predicted by $\E[\frac{1}{m^2}\sum_s\lVert\Psi_s-\frac{1}{m-1}\sum_{t\ne s}\Psi_t\rVert^2\Var(\zeta_s\mid A_s)]$. T2: router noise, predicted by the same expression with $\Var(\nu_s\mid A_s)$; $\overline{\mathrm{SR}}=\sum_jp_jR^{(j)}$ is SR without router noise.}
\label{tab:theory-check}
\end{table}

\begin{figure}[H]
\centering
\begin{minipage}[c]{0.715\linewidth}
\centering
\includegraphics[trim={0 0 276.5bp 0},clip,scale=1.40]{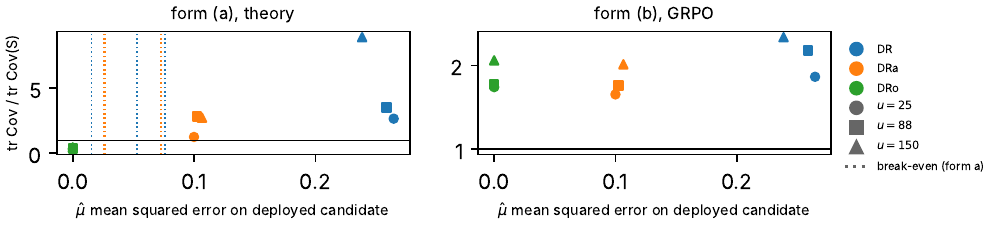}\\[1mm]
\includegraphics[trim={210bp 0 75.5bp 0},clip,scale=1.40]{figures/dose_response}
\end{minipage}%
\begin{minipage}[c]{0.27\linewidth}
\centering
\includegraphics[trim={402bp 0 0 0},clip,scale=1.40]{figures/dose_response}
\end{minipage}
\vspace{-3mm}
\caption{Variance relative to SR against the outcome model's mean squared error on the deployed candidate (DR, DR-agree, DRo). Top: form (a), the update the theory covers. Bottom: form (b), GRPO. The legend applies to both. Dotted: the error below which DR (blue) or DR-agree (orange) would match SR, extrapolated from the measured excess $\operatorname{tr}\operatorname{Cov}(\cdot)-\operatorname{tr}\operatorname{Cov}(\mathrm{DRo})$ along $\widehat\mu_\lambda=R+\lambda(\widehat\mu-R)$, on which excess and error both scale as $\lambda^2$; one line per checkpoint. In the figure, S is SR and DRa is DR-agree.}
\label{fig:dose}
\end{figure}

\begin{table}[H]
\centering
\scriptsize
\setlength{\tabcolsep}{3pt}
\input{tables/score_mean}
\caption{Score-mean check $\E_H\lVert\E[\Psi\mid H]\rVert^2/\E\lVert\Psi\rVert^2$ with $95\%$ bootstrap intervals over prompts (256 prompts, $m=4$); zero for the exact score.}
\label{tab:score-mean}
\end{table}

\begin{figure}[H]
\centering
\includegraphics[width=0.6\linewidth]{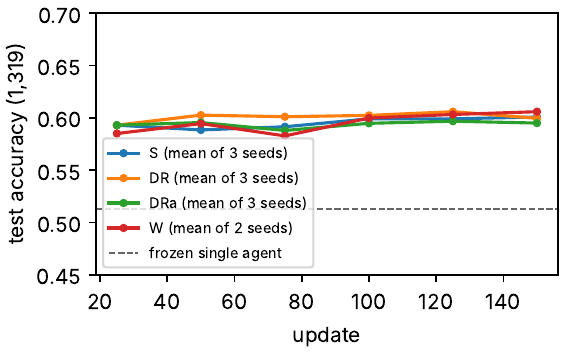}
\caption{First study, GRPO form: test accuracy (1,319 problems) every 25 updates, mean over the finished seeds of each signal (Table~\ref{tab:grpo-training}). In the legend, S is SR and DRa is DR-agree.}
\label{fig:curves}
\end{figure}

%% file: tables/allocation.tex
\begin{tabular}{lcccccccc}
\toprule
 & \multicolumn{2}{c}{$B=0.10$} & \multicolumn{4}{c}{$B=0.25$} & \multicolumn{2}{c}{$B=0.50$} \\
Estimate & oracle & feasible & oracle & leverage & pilot, no feature & feasible & oracle & feasible \\
\midrule
DR (learned model) & 0.07 & 0.50 & 0.05 & 0.90 & 1.00 & 0.24 & 0.02 & 0.12 \\
DR (agreement model) & 0.03 & 0.32 & 0.01 & 0.74 & 0.84 & 0.17 & 0.00 & 0.07 \\
Plug-in (learned model) & 0.48 & 0.75 & 0.38 & 0.67 & 0.82 & 0.65 & 0.18 & 0.43 \\
Plug-in (agreement model) & 0.38 & 0.66 & 0.27 & 0.71 & 0.86 & 0.55 & 0.09 & 0.36 \\
\bottomrule
\end{tabular}

%% file: tables/calibration.tex
\begin{tabular}{@{}llccc@{}}
\toprule
Prompts & $L$ & Greedy accuracy at cap $L$ & Sampled truncation & Answer parsed (finished) \\
\midrule
v1 & 96 & 0.438/0.164/0.273 & 0.037/0.393/0.127 & 0.996/0.633/1.000 \\
v1 & 128 & 0.465/0.285/0.324 & 0.012/0.211/0.064 & 0.996/0.726/1.000 \\
v1 & 160 & 0.465/0.355/0.344 & 0.002/0.105/0.039 & 0.996/0.761/1.000 \\
v1 & 256 & 0.465/0.422/0.359 & 0.000/0.008/0.010 & 0.996/0.796/1.000 \\
v1 & 512 & 0.465/0.426/0.363 & 0.000/0.000/0.004 & 0.996/0.801/1.000 \\
\addlinespace[2pt]
v2 & 96 & 0.410/0.211/0.129 & 0.146/0.572/0.020 & 0.935/0.966/1.000 \\
v2 & 128 & 0.469/0.379/0.129 & 0.082/0.357/0.008 & 0.941/0.979/1.000 \\
v2 & 160 & 0.496/0.504/0.129 & 0.035/0.186/0.000 & 0.944/0.985/1.000 \\
v2 & 256 & 0.504/0.609/0.129 & 0.002/0.025/0.000 & 0.945/0.988/1.000 \\
v2 & 512 & 0.504/0.629/0.129 & 0.000/0.006/0.000 & 0.945/0.988/1.000 \\
\addlinespace[2pt]
v3 & 96 & 0.531/0.605/0.555 & 0.016/0.055/0.045 & 1.000/0.987/1.000 \\
v3 & 128$^\dagger$ & 0.531/0.633/0.562 & 0.006/0.020/0.018 & 1.000/0.988/1.000 \\
v3 & 160 & 0.531/0.641/0.562 & 0.004/0.004/0.012 & 1.000/0.988/1.000 \\
v3 & 256 & 0.531/0.645/0.562 & 0.002/0.002/0.004 & 1.000/0.988/1.000 \\
v3 & 512 & 0.531/0.645/0.562 & 0.002/0.002/0.002 & 1.000/0.988/1.000 \\
\bottomrule
\end{tabular}

%% file: tables/toy.tex
\begin{tabular}{lccccc}
\toprule
Signal & rel.\ bias (a) & var/SR (a) & SNR (a) & var/SR (b) & SNR (b) \\
\midrule
SR & 0 {\scriptsize[0, 0]} & 1.00 {\scriptsize[1.00, 1.00]} & 0.026 {\scriptsize[0.012, 0.046]} & 1.00 {\scriptsize[1.00, 1.00]} & 0.026 {\scriptsize[0.012, 0.046]} \\
W & 0.79 {\scriptsize[0.32, 1.54]} & 1.06 {\scriptsize[0.62, 1.90]} & 0.054 {\scriptsize[0.031, 0.103]} & 1.05 {\scriptsize[0.63, 1.83]} & 0.054 {\scriptsize[0.031, 0.104]} \\
PL & 0 {\scriptsize[0, 0]} & 0.93 {\scriptsize[0.83, 1.01]} & 0.028 {\scriptsize[0.012, 0.055]} & 1.69 {\scriptsize[1.13, 2.76]} & 0.019 {\scriptsize[0.005, 0.047]} \\
PLo & 0 {\scriptsize[0, 0]} & 0.91 {\scriptsize[0.85, 0.96]} & 0.029 {\scriptsize[0.012, 0.054]} & 1.41 {\scriptsize[1.09, 1.74]} & 0.022 {\scriptsize[0.008, 0.047]} \\
DR, noisy $\widehat\mu$ & 0 {\scriptsize[0, 0]} & 0.47 {\scriptsize[0.32, 0.80]} & 0.061 {\scriptsize[0.015, 0.120]} & 2.32 {\scriptsize[1.42, 3.97]} & 0.067 {\scriptsize[0.029, 0.122]} \\
DR, constant $\widehat\mu$ & 0 {\scriptsize[0, 0]} & 1.49 {\scriptsize[0.66, 4.28]} & 0.027 {\scriptsize[0.003, 0.069]} & 2.07 {\scriptsize[1.33, 3.45]} & 0.030 {\scriptsize[0.014, 0.062]} \\
DRo & 0 {\scriptsize[0, 0]} & 0.29 {\scriptsize[0.23, 0.35]} & 0.090 {\scriptsize[0.039, 0.156]} & 2.26 {\scriptsize[1.52, 3.14]} & 0.087 {\scriptsize[0.057, 0.144]} \\
RL & 0 {\scriptsize[0, 0]} & 0.88 {\scriptsize[0.83, 0.92]} & 0.030 {\scriptsize[0.013, 0.055]} & 1.91 {\scriptsize[1.19, 3.30]} & 0.038 {\scriptsize[0.024, 0.066]} \\
RLo & 0 {\scriptsize[0, 0]} & 0.85 {\scriptsize[0.80, 0.88]} & 0.031 {\scriptsize[0.014, 0.057]} & 1.76 {\scriptsize[1.20, 2.46]} & 0.045 {\scriptsize[0.028, 0.073]} \\
\bottomrule
\end{tabular}

%% file: tables/toy_learn.tex
\begin{tabular}{lcccccc}
\toprule
 & & & \multicolumn{2}{c}{Adam} & \multicolumn{2}{c}{SGD} \\
Signal & rel.\ bias & second moment$/$SR & final return & area & final return & area \\
\midrule
SR & 0.00 & 1.00 {\scriptsize[1.00, 1.00]} & 0.989 {\scriptsize[0.983, 0.995]} & 0.915 & 0.978 {\scriptsize[0.962, 0.993]} & 0.921 \\
W & 0.79 & 1.09 {\scriptsize[0.65, 1.94]} & 0.848 {\scriptsize[0.676, 0.950]} & 0.815 & 0.846 {\scriptsize[0.665, 0.949]} & 0.823 \\
DRo & 0.00 & 0.31 {\scriptsize[0.24, 0.37]} & 0.998 {\scriptsize[0.998, 0.999]} & 0.965 & 0.988 {\scriptsize[0.979, 0.995]} & 0.935 \\
H, noisy $\widehat\mu$ & 0.00 & 0.47 {\scriptsize[0.31, 0.59]} & 0.994 {\scriptsize[0.987, 0.998]} & 0.949 & 0.986 {\scriptsize[0.976, 0.992]} & 0.933 \\
PL, noisy $\widehat\mu$ & 0.00 & 0.93 {\scriptsize[0.84, 1.01]} & 0.988 {\scriptsize[0.981, 0.996]} & 0.915 & 0.979 {\scriptsize[0.968, 0.991]} & 0.921 \\
DR, noisy $\widehat\mu$ & 0.00 & 0.48 {\scriptsize[0.33, 0.80]} & 0.994 {\scriptsize[0.986, 0.999]} & 0.953 & 0.986 {\scriptsize[0.975, 0.994]} & 0.933 \\
DR, constant $\widehat\mu$ & 0.00 & 1.48 {\scriptsize[0.68, 4.25]} & 0.987 {\scriptsize[0.969, 0.996]} & 0.923 & 0.976 {\scriptsize[0.955, 0.992]} & 0.921 \\
\bottomrule
\end{tabular}

%% file: tables/hybrid_design.tex
\begin{tabular}{@{}llccc@{}}
\toprule
Signal & $\widehat\mu$ & update 25 & update 88 & update 150 \\
\midrule
Winner-take-all (W), biased & -- & 0.62\,{\scriptsize[0.57, 0.68]} & 0.68\,{\scriptsize[0.60, 0.78]} & 0.68\,{\scriptsize[0.61, 0.76]} \\
Plug-in removal & learned & 1.02\,{\scriptsize[1.00, 1.03]} & 1.09\,{\scriptsize[1.05, 1.12]} & 1.08\,{\scriptsize[1.05, 1.11]} \\
Plug-in removal (baseline only) & agreement & 0.91\,{\scriptsize[0.88, 0.94]} & 0.90\,{\scriptsize[0.86, 0.94]} & 0.91\,{\scriptsize[0.87, 0.95]} \\
Control variate only (CV) & agreement & 0.76\,{\scriptsize[0.73, 0.80]} & 0.79\,{\scriptsize[0.76, 0.82]} & 0.84\,{\scriptsize[0.81, 0.87]} \\
Hybrid & learned & 1.09\,{\scriptsize[1.06, 1.12]} & 1.11\,{\scriptsize[1.07, 1.15]} & 1.09\,{\scriptsize[1.06, 1.12]} \\
Hybrid (H) & agreement & 0.68\,{\scriptsize[0.63, 0.73]} & 0.69\,{\scriptsize[0.63, 0.75]} & 0.76\,{\scriptsize[0.70, 0.81]} \\
DR removal (DR, submitted) & learned & 2.00\,{\scriptsize[1.77, 2.25]} & 3.44\,{\scriptsize[2.90, 4.05]} & 8.09\,{\scriptsize[6.77, 9.55]} \\
DR removal & agreement & 0.92\,{\scriptsize[0.83, 1.02]} & 1.62\,{\scriptsize[1.32, 1.98]} & 3.40\,{\scriptsize[2.81, 4.07]} \\
\midrule
Expected reward$^\ast$ & oracle & 0.50\,{\scriptsize[0.46, 0.54]} & 0.56\,{\scriptsize[0.53, 0.59]} & 0.64\,{\scriptsize[0.61, 0.67]} \\
Plug-in removal$^\ast$ & oracle & 0.84\,{\scriptsize[0.78, 0.89]} & 0.79\,{\scriptsize[0.74, 0.84]} & 0.79\,{\scriptsize[0.74, 0.85]} \\
Exact removal contribution (DRo)$^\ast$ & oracle & 0.34\,{\scriptsize[0.30, 0.38]} & 0.35\,{\scriptsize[0.31, 0.39]} & 0.44\,{\scriptsize[0.39, 0.48]} \\
\bottomrule
\end{tabular}

%% file: tables/hybrid_full.tex
\begin{tabular}{@{}llccc@{}}
\toprule
Signal & $\widehat\mu$ & update 25 & update 150 & both \\
\midrule
Shared reward (SR) & -- & 1 & 1 & 1 \\
Winner-take-all (W), biased & -- & 0.59\,{\scriptsize[0.56, 0.63]} & 0.60\,{\scriptsize[0.56, 0.64]} & 0.60\,{\scriptsize[0.57, 0.63]} \\
Plug-in removal & learned & 1.01\,{\scriptsize[1.00, 1.01]} & 1.10\,{\scriptsize[1.07, 1.12]} & 1.05\,{\scriptsize[1.04, 1.06]} \\
Plug-in removal (baseline only) & agreement & 0.90\,{\scriptsize[0.88, 0.92]} & 0.92\,{\scriptsize[0.89, 0.95]} & 0.91\,{\scriptsize[0.89, 0.93]} \\
Control variate only (CV) & agreement & 0.77\,{\scriptsize[0.75, 0.79]} & 0.88\,{\scriptsize[0.86, 0.89]} & 0.82\,{\scriptsize[0.80, 0.84]} \\
Hybrid & learned & 1.08\,{\scriptsize[1.06, 1.09]} & 1.13\,{\scriptsize[1.10, 1.16]} & 1.10\,{\scriptsize[1.09, 1.12]} \\
\textbf{Hybrid (H)} & agreement & 0.66\,{\scriptsize[0.63, 0.70]} & 0.80\,{\scriptsize[0.76, 0.83]} & 0.73\,{\scriptsize[0.70, 0.76]} \\
DR removal (DR, submitted) & learned & 1.90\,{\scriptsize[1.75, 2.07]} & 15.4\,{\scriptsize[13.7, 17.4]} & 8.44\,{\scriptsize[7.60, 9.38]} \\
DR removal & agreement & 0.72\,{\scriptsize[0.68, 0.76]} & 5.45\,{\scriptsize[4.84, 6.10]} & 3.01\,{\scriptsize[2.73, 3.31]} \\
\midrule
Expected reward $\sum_jp_jR^{(j)}$$^\ast$ & oracle & 0.47\,{\scriptsize[0.45, 0.48]} & 0.71\,{\scriptsize[0.69, 0.73]} & 0.58\,{\scriptsize[0.57, 0.60]} \\
Plug-in removal$^\ast$ & oracle & 0.80\,{\scriptsize[0.77, 0.82]} & 0.78\,{\scriptsize[0.74, 0.82]} & 0.79\,{\scriptsize[0.76, 0.81]} \\
Exact removal contribution (DRo)$^\ast$ & oracle & 0.26\,{\scriptsize[0.25, 0.27]} & 0.48\,{\scriptsize[0.45, 0.51]} & 0.37\,{\scriptsize[0.35, 0.39]} \\
\bottomrule
\end{tabular}

%% file: tables/hybrid_agents.tex
\begin{tabular}{@{}llcccc@{}}
\toprule
Ratio to SR & cells & agent 1 & agent 2 & agent 3 & all \\
\midrule
Hybrid (H) & update 25 & 0.57\,{\scriptsize[0.52, 0.62]} & 0.70\,{\scriptsize[0.67, 0.74]} & 0.70\,{\scriptsize[0.67, 0.73]} & 0.66\,{\scriptsize[0.63, 0.70]} \\
 & update 150 & 0.73\,{\scriptsize[0.67, 0.79]} & 0.93\,{\scriptsize[0.90, 0.96]} & 0.73\,{\scriptsize[0.68, 0.79]} & 0.80\,{\scriptsize[0.76, 0.83]} \\
 & both & 0.65\,{\scriptsize[0.60, 0.70]} & 0.81\,{\scriptsize[0.78, 0.84]} & 0.71\,{\scriptsize[0.68, 0.75]} & 0.73\,{\scriptsize[0.70, 0.76]} \\
\addlinespace
Control variate only & update 25 & 0.77\,{\scriptsize[0.74, 0.79]} & 0.77\,{\scriptsize[0.75, 0.80]} & 0.76\,{\scriptsize[0.74, 0.78]} & 0.77\,{\scriptsize[0.75, 0.79]} \\
 & update 150 & 0.87\,{\scriptsize[0.85, 0.90]} & 0.89\,{\scriptsize[0.87, 0.90]} & 0.87\,{\scriptsize[0.86, 0.89]} & 0.88\,{\scriptsize[0.86, 0.89]} \\
 & both & 0.82\,{\scriptsize[0.80, 0.84]} & 0.83\,{\scriptsize[0.81, 0.85]} & 0.81\,{\scriptsize[0.79, 0.83]} & 0.82\,{\scriptsize[0.80, 0.84]} \\
\addlinespace
Baseline only & update 25 & 0.80\,{\scriptsize[0.76, 0.85]} & 0.93\,{\scriptsize[0.91, 0.95]} & 0.94\,{\scriptsize[0.92, 0.96]} & 0.90\,{\scriptsize[0.88, 0.92]} \\
 & update 150 & 0.85\,{\scriptsize[0.80, 0.91]} & 1.04\,{\scriptsize[1.01, 1.07]} & 0.86\,{\scriptsize[0.82, 0.90]} & 0.92\,{\scriptsize[0.89, 0.95]} \\
 & both & 0.83\,{\scriptsize[0.79, 0.87]} & 0.98\,{\scriptsize[0.97, 1.00]} & 0.90\,{\scriptsize[0.88, 0.93]} & 0.91\,{\scriptsize[0.89, 0.93]} \\
\addlinespace
Exact removal$^\ast$ & update 25 & 0.18\,{\scriptsize[0.15, 0.20]} & 0.35\,{\scriptsize[0.33, 0.38]} & 0.25\,{\scriptsize[0.23, 0.27]} & 0.26\,{\scriptsize[0.25, 0.27]} \\
 & update 150 & 0.29\,{\scriptsize[0.23, 0.36]} & 0.98\,{\scriptsize[0.91, 1.04]} & 0.17\,{\scriptsize[0.14, 0.19]} & 0.48\,{\scriptsize[0.45, 0.51]} \\
 & both & 0.24\,{\scriptsize[0.20, 0.27]} & 0.66\,{\scriptsize[0.62, 0.70]} & 0.21\,{\scriptsize[0.19, 0.23]} & 0.37\,{\scriptsize[0.35, 0.39]} \\
\addlinespace
DR removal & update 25 & 1.13\,{\scriptsize[0.97, 1.33]} & 2.78\,{\scriptsize[2.53, 3.07]} & 1.70\,{\scriptsize[1.59, 1.83]} & 1.90\,{\scriptsize[1.75, 2.07]} \\
 & update 150 & 4.05\,{\scriptsize[2.99, 5.52]} & 38.8\,{\scriptsize[34.1, 44.2]} & 2.43\,{\scriptsize[1.90, 3.05]} & 15.4\,{\scriptsize[13.7, 17.4]} \\
 & both & 2.62\,{\scriptsize[2.07, 3.36]} & 20.5\,{\scriptsize[18.3, 23.0]} & 2.03\,{\scriptsize[1.78, 2.33]} & 8.44\,{\scriptsize[7.60, 9.38]} \\
\addlinespace
\midrule
H, seed 4 & update 25 & 0.59\,{\scriptsize[0.53, 0.65]} & 0.79\,{\scriptsize[0.76, 0.83]} & 0.65\,{\scriptsize[0.60, 0.71]} & 0.69\,{\scriptsize[0.65, 0.73]} \\
H, seed 4 & update 150 & 0.65\,{\scriptsize[0.56, 0.77]} & 1.00\,{\scriptsize[0.95, 1.06]} & 0.71\,{\scriptsize[0.64, 0.79]} & 0.80\,{\scriptsize[0.75, 0.86]} \\
H, seed 5 & update 25 & 0.56\,{\scriptsize[0.49, 0.64]} & 0.57\,{\scriptsize[0.51, 0.62]} & 0.79\,{\scriptsize[0.77, 0.83]} & 0.66\,{\scriptsize[0.62, 0.69]} \\
H, seed 5 & update 150 & 0.68\,{\scriptsize[0.58, 0.79]} & 0.88\,{\scriptsize[0.82, 0.94]} & 0.74\,{\scriptsize[0.69, 0.80]} & 0.77\,{\scriptsize[0.72, 0.82]} \\
H, seed 6 & update 25 & 0.57\,{\scriptsize[0.48, 0.65]} & 0.73\,{\scriptsize[0.68, 0.78]} & 0.63\,{\scriptsize[0.58, 0.68]} & 0.64\,{\scriptsize[0.60, 0.69]} \\
H, seed 6 & update 150 & 0.82\,{\scriptsize[0.75, 0.90]} & 0.90\,{\scriptsize[0.84, 0.95]} & 0.74\,{\scriptsize[0.66, 0.85]} & 0.82\,{\scriptsize[0.78, 0.87]} \\
\bottomrule
\end{tabular}

%% file: tables/hybrid_training.tex
\begin{tabular}{@{}lcccccccc@{}}
\toprule
& & \multicolumn{4}{c}{Accuracy gain over the frozen agent} & & \multicolumn{2}{c}{Paired difference to SR} \\
\cmidrule(lr){3-6}\cmidrule(lr){8-9}
Signal & seeds & agent mean (sd) & agent 1 & agent 2 & agent 3 & deployed & mean & 95\% interval \\
\midrule
Shared reward (SR) & 7 & $+$0.046 (0.012) & $+$0.010 & $+$0.059 & $+$0.070 & 0.657 & -- & -- \\
Hybrid (H) & 7 & $+$0.033 (0.047) & $-$0.038 & $+$0.067 & $+$0.070 & 0.666 & $-$0.013 & [$-$0.062, $+$0.035] \\
DR removal (DR) & 7 & $+$0.034 (0.025) & $+$0.033 & $+$0.027 & $+$0.042 & 0.629 & $-$0.012 & [$-$0.030, $+$0.005] \\
\bottomrule
\end{tabular}

%% file: tables/grpo_training.tex
\begin{tabular}{@{}lccccc@{}}
\toprule
System & seeds & Deployed & Correct of 1{,}319 & $\Delta$ to SR (paired) [95\%] & Agent gain \\
\midrule
Frozen agent 1 (greedy) & -- & 0.513 & 676 & -- & -- \\
Best frozen agent (agent 2) & -- & 0.613 & 809 & -- & -- \\
Frozen agents, router after warm-up & 3 & 0.593 & 801, 742, 803 & -- & 0 \\
\midrule
Shared reward (SR) & 3 & 0.601 & 772, 782, 823 & -- & $+$0.012 \\
DR removal, learned $\widehat\mu$ & 3 & 0.599 & 763, 806, 803 & $-$0.001 [$-$0.044, $+$0.042] & $+$0.010 \\
DR removal, agreement $\widehat\mu$ & 3 & 0.595 & 780, 792, 782 & $-$0.006 [$-$0.060, $+$0.049] & $+$0.008 \\
Winner-take-all (W) & 2 & 0.606 & 786, 812 & $+$0.017 [$-$0.060, $+$0.094] & $+$0.016 \\
\bottomrule
\end{tabular}

%% file: tables/predictions2.tex
\begin{tabular}{@{}lp{0.50\linewidth}lp{0.30\linewidth}@{}}
\toprule
 & Prediction & Outcome & Measured \\
\midrule
C1 & H$/$SR $<1$ at updates 25 and 150 (interval excludes 1) and in every seed $\times$ checkpoint cell (primary) & held & 0.66 [0.63, 0.70]; 0.80 [0.76, 0.83]; cells 0.64--0.82 \\
C2 & DRo $<$ H $<$ SR $<$ DR at update 150; DR$/$SR $>2$ & held & DRo$/$H 0.61 [0.56, 0.65]; DR$/$SR 15.4 \\
C3 & Router-noise term with $\widehat\mu$ over without: $0.45$--$0.65$ (agreement), $\ge0.95$ (learned) & held & 0.57; 1.12 \\
C4 & CV in $[0.70,0.90]$; baseline only in $[0.85,1.00]$; H below both; H with learned $\widehat\mu$ in $[1.00,1.20]$ & held & 0.820; 0.908; H 0.728; 1.101 \\
C5 & H$/$SR falls with the agreement pool: $2>5>11$ completions; pool of 2 in $[0.90,1.15]$ & held & 1.09 / 0.81 / 0.73 \\
C6 & $X_{0.1}<$ H at both checkpoints; $X_{1}>$ H at update 150 & held & $X_{0.1}/$H 0.85, 0.93; $X_1/$H 8.0 \\
C7 & Re-routed at $\tau=3$: DR with full-pool agreement $<$ H; at the trained $\tau$, update 150: $>$ H & held & 0.75 [0.72, 0.78]; 6.7 \\
C8 & Agent 1's score mean on the recorded nucleus: interval contains 0, exact ratio $<0.01$; more nucleus flips than agent 3 & \textbf{failed} & pilot cell: exact ratio 0.065 [0.003, 0.151]; flips 0.39\% against 2.5\% \\
C9 & $\lVert\bar u_{\mathrm W}-\bar u_{\mathrm S}\rVert^2>0$ at update 150 (weak) & \textbf{failed} & 1.23 [-0.92, 5.30] \\
C10 & H$/$SR $<1$ on SVAMP and with Qwen2.5-3B-Instruct (direction only) & held & SVAMP 0.70 [0.59, 0.84] (SVAMP only: the other configuration failed its smoke test) \\
C11 & GRPO form: H$/$SR in $[1.4,2.2]$; plain form with norm clip 1: H$/$SR $>1$ & held & 1.92 [1.83, 2.02]; 1.61 [1.57, 1.67] \\
\midrule
T1 & Training gain: H $-$ DR $>0$ (paired $t$, one-sided $5\%$; primary) & \textbf{failed} & $-$0.001 [$-$0.063, $+$0.060], $n=7$, $p=0.517$ \\
T2 & Training gain: SR $-$ DR $>0$ & \textbf{failed} & $+$0.012 [$-$0.005, $+$0.030], $n=7$, $p=0.070$ \\
T3 & Training gain: H $-$ SR $\ge0$ as a point estimate (not powered) & \textbf{failed} & $-$0.013 [$-$0.062, $+$0.035], $n=7$ \\
T4 & Order of gains is the inverse order of the measured second moments & \textbf{failed} & gains: SR $>$ DR $>$ H; second moments: H $<$ SR $<$ DR \\
\bottomrule
\end{tabular}

%% file: tables/predictions_evidence.tex
\begin{tabular}{@{}lp{0.42\linewidth}lp{0.36\linewidth}@{}}
\toprule
 & Prediction & Outcome & Evidence \\
\midrule
P1 & Unbiased signals share SR's mean update; W's differs & \textbf{failed} & {\scriptsize unbiased signals off SR: none; W differs at $u\ge88$: no} \\
P2 & Form (a): DRo $<$ RL $\le$ SR $<$ W; learned-$\widehat\mu$ DR $\ge$ SR, with the sign predicted by Prop.~\ref{prop:variance-monotone}(iii) & \textbf{failed} & {\scriptsize $u{=}25$: DRo 0.33, RL 1.02, W 0.62, DR 2.65; $u{=}88$: DRo 0.36, RL 1.04, W 0.68, DR 3.50; $u{=}150$: DRo 0.45, RL 1.04, W 0.69, DR 8.91} \\
P3 & Form (a): DRo $\le$ DR-agree $\le$ DR & held & {\scriptsize DRo, DR-agree, DR: $u{=}25$: 0.33, 1.25, 2.65; $u{=}88$: 0.36, 2.82, 3.50; $u{=}150$: 0.45, 2.73, 8.91} \\
P4 & Measured variance differences match the predicted noise and baseline terms & held & {\scriptsize 15 of 15 checks within CI} \\
P5 & Higher form-(b) SNR trains faster (weak); W is the weakest arm & \textbf{failed} & {\scriptsize SNR(b) order DR-agree, SR, DR; speed order DR, DR-agree, SR; W lowest: no} \\
\bottomrule
\end{tabular}

%% file: tables/theory_check.tex
\begin{tabular}{@{}llccc@{}}
\toprule
Check & Difference & Measured & Predicted & Measured $-$ predicted \\
\midrule
\multicolumn{5}{@{}l}{\emph{checkpoint $u=25$}} \\
T1 & tr\,Cov(DR) $-$ tr\,Cov(DRo) & 14627\,{\scriptsize[8643, 24486]} & 10204\,{\scriptsize[8631, 11858]} & 4424\,{\scriptsize[-851, 13111]} \\
T1 & tr\,Cov(DR-agree) $-$ tr\,Cov(DRo) & 5773\,{\scriptsize[3060, 9973]} & 3577\,{\scriptsize[2978, 4223]} & 2195\,{\scriptsize[-271, 6052]} \\
T2 & tr\,Cov(SR) $-$ tr\,Cov($\overline{\mathrm{SR}}$) & 3206\,{\scriptsize[2682, 3780]} & 3044\,{\scriptsize[2637, 3514]} & 162\,{\scriptsize[-283, 587]} \\
T3 & $\E[(\overline{\mathrm{SR}}{-}\cdot)^2]$ $-$ DRo, per completion & 40354\,{\scriptsize[35906, 45042]} & 40694 & -340\,{\scriptsize[-990, 350]} \\
T3 & $\E[(\overline{\mathrm{PL}}{-}\cdot)^2]$ $-$ DRo, per completion & 16911\,{\scriptsize[13197, 20637]} & 17364 & -453\,{\scriptsize[-1165, 244]} \\
\multicolumn{5}{@{}l}{\emph{checkpoint $u=88$}} \\
T1 & tr\,Cov(DR) $-$ tr\,Cov(DRo) & 17514\,{\scriptsize[11164, 25376]} & 17570\,{\scriptsize[14720, 20556]} & -56\,{\scriptsize[-6035, 7945]} \\
T1 & tr\,Cov(DR-agree) $-$ tr\,Cov(DRo) & 13740\,{\scriptsize[2642, 34632]} & 7209\,{\scriptsize[5506, 9186]} & 6531\,{\scriptsize[-4479, 26820]} \\
T2 & tr\,Cov(SR) $-$ tr\,Cov($\overline{\mathrm{SR}}$) & 2396\,{\scriptsize[1907, 2948]} & 2510\,{\scriptsize[2156, 2904]} & -113\,{\scriptsize[-465, 237]} \\
T3 & $\E[(\overline{\mathrm{SR}}{-}\cdot)^2]$ $-$ DRo, per completion & 47431\,{\scriptsize[41718, 54175]} & 47551 & -120\,{\scriptsize[-1006, 751]} \\
T3 & $\E[(\overline{\mathrm{PL}}{-}\cdot)^2]$ $-$ DRo, per completion & 19874\,{\scriptsize[16237, 23800]} & 19828 & 46\,{\scriptsize[-655, 796]} \\
\multicolumn{5}{@{}l}{\emph{checkpoint $u=150$}} \\
T1 & tr\,Cov(DR) $-$ tr\,Cov(DRo) & 44765\,{\scriptsize[15180, 95344]} & 41778\,{\scriptsize[35175, 49288]} & 2987\,{\scriptsize[-27347, 52894]} \\
T1 & tr\,Cov(DR-agree) $-$ tr\,Cov(DRo) & 12073\,{\scriptsize[3446, 25227]} & 16207\,{\scriptsize[12728, 19851]} & -4134\,{\scriptsize[-13212, 8221]} \\
T2 & tr\,Cov(SR) $-$ tr\,Cov($\overline{\mathrm{SR}}$) & 1781\,{\scriptsize[1417, 2183]} & 1947\,{\scriptsize[1677, 2238]} & -166\,{\scriptsize[-436, 110]} \\
T3 & $\E[(\overline{\mathrm{SR}}{-}\cdot)^2]$ $-$ DRo, per completion & 52259\,{\scriptsize[46715, 58568]} & 52347 & -88\,{\scriptsize[-857, 655]} \\
T3 & $\E[(\overline{\mathrm{PL}}{-}\cdot)^2]$ $-$ DRo, per completion & 18950\,{\scriptsize[15221, 22806]} & 19291 & -342\,{\scriptsize[-1115, 474]} \\
\bottomrule
\end{tabular}

%% file: tables/score_mean.tex
\begin{tabular}{@{}lcccc@{}}
\toprule
Cell (seed, checkpoint) & Agent 1 & Agent 2 & Agent 3 & All agents \\
\midrule
seed 1, $u=25$ & 0.039\,{\scriptsize[0.001, 0.080]} & 0.010\,{\scriptsize[-0.020, 0.040]} & -0.002\,{\scriptsize[-0.009, 0.005]} & 0.014\,{\scriptsize[-0.002, 0.030]} \\
seed 1, $u=88$ & 0.035\,{\scriptsize[-0.031, 0.110]} & 0.050\,{\scriptsize[-0.033, 0.163]} & -0.002\,{\scriptsize[-0.008, 0.005]} & 0.027\,{\scriptsize[-0.011, 0.078]} \\
seed 1, $u=150$ & 0.075\,{\scriptsize[-0.002, 0.157]} & 0.017\,{\scriptsize[-0.026, 0.065]} & 0.003\,{\scriptsize[-0.003, 0.011]} & 0.028\,{\scriptsize[0.001, 0.056]} \\
seed 2, $u=25$ & 0.041\,{\scriptsize[-0.004, 0.083]} & -0.001\,{\scriptsize[-0.022, 0.019]} & -0.002\,{\scriptsize[-0.009, 0.006]} & 0.011\,{\scriptsize[-0.004, 0.026]} \\
seed 2, $u=88$ & 0.045\,{\scriptsize[-0.017, 0.115]} & 0.010\,{\scriptsize[-0.025, 0.050]} & -0.001\,{\scriptsize[-0.007, 0.005]} & 0.016\,{\scriptsize[-0.006, 0.041]} \\
seed 2, $u=150$ & 0.060\,{\scriptsize[-0.006, 0.132]} & 0.000\,{\scriptsize[-0.020, 0.022]} & -0.008\,{\scriptsize[-0.013, -0.001]} & 0.016\,{\scriptsize[-0.006, 0.041]} \\
seed 3, $u=25$ & 0.036\,{\scriptsize[-0.020, 0.096]} & 0.017\,{\scriptsize[-0.016, 0.051]} & 0.001\,{\scriptsize[-0.005, 0.008]} & 0.017\,{\scriptsize[-0.000, 0.036]} \\
seed 3, $u=88$ & 0.042\,{\scriptsize[-0.013, 0.099]} & -0.008\,{\scriptsize[-0.024, 0.009]} & -0.004\,{\scriptsize[-0.011, 0.003]} & 0.008\,{\scriptsize[-0.009, 0.028]} \\
seed 3, $u=150$ & 0.140\,{\scriptsize[-0.015, 0.350]} & 0.015\,{\scriptsize[-0.009, 0.041]} & -0.001\,{\scriptsize[-0.010, 0.009]} & 0.051\,{\scriptsize[-0.002, 0.135]} \\
\midrule
pooled, $u=25$ & & & & 0.014\,{\scriptsize[0.004, 0.025]} \\
pooled, $u=88$ & & & & 0.017\,{\scriptsize[0.002, 0.038]} \\
pooled, $u=150$ & & & & 0.031\,{\scriptsize[0.008, 0.063]} \\
\bottomrule
\end{tabular}

%% file: sections/e8.tex
\section{Nine-Agent Teams}
\label{app:e8}

\paragraph{Registration.}
The design, the hypotheses, the test, the estimates and every threshold were committed to the code repository (\texttt{PREDICTIONS\_3.md}) before any data of this experiment existed. The measurements before training (Table~\ref{tab:e8-step0}) were added in a second commit before any training run. No team, copy or outcome is excluded.

\paragraph{Team.}
Model, prompts, sampling settings and answer extraction are those of Section~\ref{sec:train3} and Appendix~\ref{app:implementation}; a copy of role $r$ is prompted and sampled exactly as agent $r$ there. Seats 1--3 hold role~1 and are not trained: at each update they contribute answers drawn without replacement from a fixed pool of $12$ answers per training problem, generated once by the untrained role-1 agent. Seats 4--6 are three copies of role~2 and seats 7--9 three copies of role~3. Each copy has its own LoRA adapter, its own AdamW optimizer and its own sampling seed; the three copies of a role start bit-identical. Within each role the three signals are assigned to the three copies by a random permutation drawn from a committed seed.

\paragraph{Router and signals.}
For each of the $m=4$ slates the router picks one of the nine answers uniformly at random, independently of the answers and across slates, so $p_j=1/9$, and after removing agent $i$ each remaining answer has probability $1/8$. Only the picked answer is graded. For the copy in seat $i$: SR is $G$; W is $\1_{\{I=i\}}G$; DR-agree is the removal estimate \eqref{eq:dr-candidate} with these probabilities and the agreement model,
\[
\text{DR-agree}_i=\frac{\widehat\mu_i}{9}-\sum_{j\ne i}\frac{\widehat\mu_j}{72}+\1_{\{I=i\}}\bigl(G-\widehat\mu_i\bigr)-\frac18\,\1_{\{I\ne i\}}\bigl(G-\widehat\mu_I\bigr),
\]
where $\widehat\mu_j$ is the share of the prompt's other answers ($34$ or $35$) with answer $j$'s final number, and $0$ if answer $j$ has none. Every signal uses only the picked answer's reward and the answers' text. The three signals give the same expected update given the prompt, $\frac19\operatorname{Cov}(\Psi,R^{(i)}\mid H)$, which equals $\frac19\nabla_{\theta_i}\E[R^{(i)}\mid H]$ for an exact score; the score used in training has a small conditional mean (prediction C8), whose extra term is the same for all three signals.

\paragraph{Update and outcome.}
The update is \eqref{eq:centred-update} with the sequence-sum log-probability: no standardization, no length normalization, no gradient-norm clipping and no auxiliary loss. AdamW (betas $0.9$ and $0.999$, no weight decay) takes one step per update, including updates whose four advantages are all zero; one prompt per update, $500$ updates. The outcome is the copy's greedy accuracy on the $1{,}319$ GSM8K test problems after update $500$ minus the untrained agent's accuracy for its role ($809/1{,}319$ for role~2 and $735/1{,}319$ for role~3); no checkpoint is selected. A copy collapses if its gain is below $-0.05$.

\paragraph{Learning rates and the test.}
The $21$ team seeds were assigned to rates by a committed permutation: five teams each at $5\cdot10^{-6}$, $10^{-5}$ (the rate of Section~\ref{sec:train3}) and $2\cdot10^{-5}$, which form the test, and three teams each at $4\cdot10^{-5}$ and $8\cdot10^{-5}$, which map where training breaks and never enter the test. A block is the three copies of one role in one team, giving \eEightBlocks{} blocks at the test rates. With $d_b$ the gain of the DR-agree copy minus that of the SR copy and $s_b$ the standard deviation of all gains in the block's rate and role, the statistic is $T=\sum_bd_b/s_b$, and the one-sided $p$-value counts, exactly, the share of the $2^{30}$ sign patterns $e$ with $\sum_be_bd_b/s_b\ge T$. The test is exact under the null that exchanging the DR-agree and SR labels within a block changes no gain, whether or not the copies of a team interfere. The hypotheses were tested in a fixed order, each only if the previous one held: H1, DR-agree $>$ SR; H2, DR-agree $>$ W; H3, W $>$ SR. A binding futility check after the first six teams (continue if $T>0$) gave $T=12.2$.

\paragraph{Results of the test.}
H1 held: \eEightHOnePos{} of \eEightBlocks{} differences are positive, mean \eEightHOneMeanD, $p=\eEightHOneP$. H2 failed: \eEightHTwoPos{} of \eEightBlocks{} positive, mean \eEightHTwoMeanD, $p=\eEightHTwoP$. H3 was therefore not tested. The same statistic over all five rates, a declared sensitivity analysis, gives $p=\eEightSensP$ for DR-agree $>$ SR. Table~\ref{tab:e8-full} gives the paired differences at each rate.

\begin{table}[H]
\centering
\scriptsize
\setlength{\tabcolsep}{4pt}
\input{tables/e8_full}
\caption{Nine-agent teams: paired differences of accuracy gain within teams, mean over the team's two roles, with $95\%$ $t$ intervals over teams.}
\label{tab:e8-full}
\end{table}

\paragraph{Best rate for each signal.}
For each signal and each team, the best rate is chosen on the other $20$ teams (the rate with the highest mean gain; ties go to the lower rate), and the signal's tuned gain is its mean gain over the teams run at their chosen rate: \eEightTunedS{} for SR (best rate \eEightBestS), \eEightTunedW{} for W (\eEightBestW) and \eEightTunedD{} for DR-agree (\eEightBestD). The difference DR-agree $-$ SR is \eEightDelta, with a $95\%$ team-bootstrap interval \eEightDeltaCI{} ($10{,}000$ resamples). By the reading fixed in advance, an interval above zero means that tuned DR-agree gains more than tuned SR. No best rate lies at an edge of the rates run, so the rule for adding rates was not triggered. Rate tolerance, an estimate also fixed in advance: DR-agree's mean gain reaches tuned SR's at \eEightTolDN{} rates (\eEightTolD), SR's only at \eEightTolS{} and W's only at \eEightTolW.

\paragraph{Noise before and during training.}
Before any training run, the noise of each signal was measured on $256$ fresh prompts with a shared-reward policy at update $25$, for teams of three, six and nine agents built from merged sampling replicates under the uniform router (Table~\ref{tab:e8-step0}). Every prediction of this measurement held, and the gate it set for training (DR-agree $\le0.15$ and DR-agree $<$ W $<1$ for each role) passed. During training, the logged update norms give mean second-moment ratios to SR of \eEightRealWTwo{} (W) and \eEightRealDTwo{} (DR-agree) for role~2 and \eEightRealWThree{} and \eEightRealDThree{} for role~3, close to the values measured before training.

\begin{table}[H]
\centering
\scriptsize
\setlength{\tabcolsep}{4pt}
\input{tables/e8_step0}
\caption{Gradient noise relative to shared reward, $\kappa(X)$, under the uniform router, by number of agents, before training ($95\%$ intervals over prompts). The removal estimate DR-agree gains most from more agents. $^\ast$Exact contribution, not a usable signal.}
\label{tab:e8-step0}
\end{table}

\paragraph{Other routers and parameter groups.}
The same measurement registered further predictions, and all held. Under a router that reads the answers (an $\epsilon$-mixed softmax, $\tau=1$ and $\epsilon=0.2$, over the number of other answers in the slate with the same final number), which deploys a correct answer more often than the uniform router (\eEightVoteAcc{} against \eEightUnifAcc), DR-agree's noise relative to SR is \eEightVoteDTwo{} \eEightVoteDTwoCI{} for role~2 and \eEightVoteDThree{} \eEightVoteDThreeCI{} for role~3 (predicted $\le0.30$). Under the learned router of Section~\ref{sec:noise}, not retrained, going from three to six agents lowers DR-agree's noise from \eEightTeamDThreeEarly{} to \eEightTeamDSixEarly{} at update $25$ and from \eEightTeamDThreeLate{} to \eEightTeamDSixLate{} at update $150$, and H's from \eEightTeamHThreeEarly{} to \eEightTeamHSixEarly{} and from \eEightTeamHThreeLate{} to \eEightTeamHSixLate{} (agents pooled; with six agents the agreement model also sees more answers). In every LoRA parameter group DR-agree $<$ W $<1$, and DR-agree's ratio differs between groups by a factor of at most \eEightGroupMax, so the reduction is nearly uniform across the coordinates that AdamW rescales (Section~\ref{sec:signals}).

\paragraph{Checks and cost.}
After each team, an independent implementation recomputed every trained advantage of all $500$ updates and six copies from the logged picks, rewards and answers (largest error $6.1\cdot10^{-16}$). Before the teams, the signal code reproduced the logged signals of all $21$ runs of Section~\ref{sec:train3}, up to the rounding of their logged router scores, and a ten-update team passed the registered checks. Every team ran on one NVIDIA A40 (about four to five hours per team); the experiment cost about $\$55$.

%% file: tables/e8_full.tex
\begin{tabular}{@{}lc ccc@{}}
\toprule
Learning rate & Teams & DR-agree $-$ SR & DR-agree $-$ W & W $-$ SR \\
\midrule
$5\cdot10^{-6}$ & 5 & $+$0.043\,{\scriptsize[$+$0.024, $+$0.062]} & $+$0.025\,{\scriptsize[$+$0.007, $+$0.044]} & $+$0.017\,{\scriptsize[$-$0.005, $+$0.039]} \\
$10^{-5}$ & 5 & $+$0.020\,{\scriptsize[$-$0.054, $+$0.095]} & $-$0.016\,{\scriptsize[$-$0.068, $+$0.036]} & $+$0.037\,{\scriptsize[$-$0.005, $+$0.078]} \\
$2\cdot10^{-5}$ & 5 & $+$0.200\,{\scriptsize[$-$0.029, $+$0.430]} & $+$0.059\,{\scriptsize[$+$0.005, $+$0.113]} & $+$0.141\,{\scriptsize[$-$0.105, $+$0.387]} \\
\midrule
$4\cdot10^{-5}$ & 3 & $+$0.377\,{\scriptsize[$+$0.121, $+$0.634]} & $+$0.034\,{\scriptsize[$-$0.170, $+$0.238]} & $+$0.343\,{\scriptsize[$+$0.266, $+$0.420]} \\
$8\cdot10^{-5}$ & 3 & $+$0.087\,{\scriptsize[$-$0.286, $+$0.460]} & $+$0.063\,{\scriptsize[$-$0.228, $+$0.353]} & $+$0.024\,{\scriptsize[$-$0.059, $+$0.107]} \\
\bottomrule
\end{tabular}

%% file: tables/e8_step0.tex
\begin{tabular}{@{}lc cccc@{}}
\toprule
Role & Agents & W & DR-agree & H & DRo$^\ast$ \\
\midrule
2 & 3 & 0.77\,{\scriptsize[0.71, 0.85]} & 0.41\,{\scriptsize[0.38, 0.44]} & 0.64\,{\scriptsize[0.56, 0.72]} & 0.16\,{\scriptsize[0.15, 0.18]} \\
 & 6 & 0.45\,{\scriptsize[0.40, 0.49]} & 0.15\,{\scriptsize[0.14, 0.17]} & 0.55\,{\scriptsize[0.48, 0.62]} & 0.04\,{\scriptsize[0.03, 0.05]} \\
 & 9 & 0.31\,{\scriptsize[0.28, 0.35]} & 0.09\,{\scriptsize[0.08, 0.10]} & 0.54\,{\scriptsize[0.47, 0.61]} & 0.02\,{\scriptsize[0.01, 0.02]} \\
\midrule
3 & 3 & 0.72\,{\scriptsize[0.66, 0.80]} & 0.37\,{\scriptsize[0.35, 0.40]} & 0.58\,{\scriptsize[0.52, 0.64]} & 0.16\,{\scriptsize[0.15, 0.18]} \\
 & 6 & 0.42\,{\scriptsize[0.38, 0.47]} & 0.13\,{\scriptsize[0.12, 0.15]} & 0.50\,{\scriptsize[0.44, 0.56]} & 0.04\,{\scriptsize[0.04, 0.04]} \\
 & 9 & 0.29\,{\scriptsize[0.26, 0.33]} & 0.08\,{\scriptsize[0.07, 0.09]} & 0.48\,{\scriptsize[0.43, 0.55]} & 0.02\,{\scriptsize[0.02, 0.02]} \\
\bottomrule
\end{tabular}